%% file: tuning.tex
\documentclass[11pt,a4paper]{article}
\usepackage{times}
\usepackage{configuration/ato_report}

\usepackage[hyperfootnotes=false]{hyperref}
\usepackage{url}
\usepackage{xcolor}
\usepackage{booktabs}       
\definecolor{color1}{HTML}{406058}
\definecolor{color2}{HTML}{78A0C8}
\definecolor{color3}{HTML}{3850A0}
\definecolor{color4}{HTML}{6860A8}
\definecolor{color5}{HTML}{A03850}
\definecolor{color6}{HTML}{AE7694}
\definecolor{color7}{HTML}{8488B4}
\hypersetup{
    pdftitle={When Can LLM Reasoning Reliably Reuse Intermediate Results?},
    pdfauthor={Deyuan Liu, Zecheng Wang, Xuanyi Liu, Jungang Li, Yihao Hu, Mingjie Tang, Bryan Dai, Jie Fu},
    colorlinks=true,
    linkcolor=black,
    citecolor=color1,
    urlcolor=color2,
    pdfborder={0 0 0}
}
\usepackage[T1]{fontenc}

\input{resources/preamble.tex}

\title{When Can LLM Reasoning Reliably\\Reuse Intermediate Results?}

\author{
  Deyuan Liu$^{1}$ \quad Zecheng Wang$^{2}$ \quad
  Xuanyi Liu$^{3}$ \quad Jungang Li$^{4}$ \\
  Yihao Hu$^{1,5}$ \quad Mingjie Tang$^{6}$ \quad
  Bryan Dai$^{6}$ \quad Jie Fu$^{6}$\thanks{Corresponding author.} \\[5pt]
  \footnotesize $^{1}$Westlake University \quad $^{2}$Wechat AI \quad $^{3}$Peking University \\
  \footnotesize $^{4}$HKUST \quad $^{5}$Ant Group \quad $^{6}$IQuest Research
}

\begin{document}

\maketitle

\input{resources/main.tex}

\section*{AI Use Statement}
\addcontentsline{toc}{section}{AI Use Statement}
\label{sec:ai_use}
We used generative AI tools to assist manuscript writing, revision, and the development and checking of mathematical derivations. The language models used to generate candidate solutions and continuations, together with the answer-selection procedures, are part of the experimental protocols described in \secref{sec:experiments}. We take responsibility for the final manuscript and its scientific claims.

\clearpage

\phantomsection
\addcontentsline{toc}{section}{References}
\bibliography{resources/reference}
\bibliographystyle{plainnat}

\clearpage
\appendix
\addtocontents{toc}{\protect\setcounter{tocdepth}{1}}

\input{resources/appendix.tex}

\end{document}

%% file: resources/preamble.tex
\usepackage{amsmath,amssymb,amsthm}

\providecommand{\mtheta}{\boldsymbol{\theta}}

\providecommand{\cG}{\mathcal{G}}

\providecommand{\cV}{\mathcal{V}}

\newenvironment{talign*}
{\csname align*\endcsname}
{\endalign}

\usepackage{algorithm}
\usepackage{algorithmic}
\usepackage{graphicx}
\usepackage{subcaption}
\usepackage[flushleft]{threeparttable}
\usepackage{float}
\usepackage{afterpage} \usepackage{multirow}
\usepackage{xspace}
\usepackage{natbib}
\usepackage{enumitem}
\usepackage[font=small]{caption}
\usepackage{autobreak}
\usepackage{sidecap}
\usepackage{wrapfig}
\usepackage{placeins} \usepackage{needspace} \usepackage{bbding}
\usepackage[toc, page, header]{appendix}
\usepackage{tikz}
\usetikzlibrary{arrows.meta,positioning,calc}
\usepackage{pifont}
\usepackage{mdframed}
\usepackage{tcolorbox}
\tcbuselibrary{skins,breakable}
\usepackage{colortbl}
\usepackage{mathrsfs}
\usepackage{footmisc}

\definecolor{coral}{RGB}{255,127,80}
\definecolor{darkgreen}{RGB}{0,100,0}
\definecolor{darkyellow}{RGB}{204,153,0}
\definecolor{salmon}{RGB}{250,128,114}
\definecolor{champagne}{RGB}{247,231,206}
\definecolor{kleinblue}{RGB}{0,47,167}
\definecolor{bauhiniapurple}{RGB}{193,84,193}
\definecolor{darkred}{RGB}{150,0,0}

\newcommand{\darkredtext}[1]{{\color{darkred}#1}}

\newcommand{\thmref}[1]{\hyperref[#1]{\darkredtext{Thm.~\ref*{#1}}}}
\newcommand{\defref}[1]{\hyperref[#1]{\darkredtext{Def.~\ref*{#1}}}}
\newcommand{\propref}[1]{\hyperref[#1]{\darkredtext{Prop.~\ref*{#1}}}}
\newcommand{\assumpref}[1]{\hyperref[#1]{\darkredtext{Assump.~\ref*{#1}}}}
\newcommand{\remarkref}[1]{\hyperref[#1]{\darkredtext{Rem.~\ref*{#1}}}}
\newcommand{\hypref}[1]{\hyperref[#1]{\darkredtext{Hyp.~\ref*{#1}}}}
\newcommand{\conjref}[1]{\hyperref[#1]{\darkredtext{Conj.~\ref*{#1}}}}
\newcommand{\lemref}[1]{\hyperref[#1]{\darkredtext{Lem.~\ref*{#1}}}}
\newcommand{\corref}[1]{\hyperref[#1]{\darkredtext{Cor.~\ref*{#1}}}}
\newcommand{\noteref}[1]{\hyperref[#1]{\darkredtext{Nota.~\ref*{#1}}}}
\newcommand{\claimref}[1]{\hyperref[#1]{\darkredtext{Clm.~\ref*{#1}}}}
\newcommand{\obsref}[1]{\hyperref[#1]{\darkredtext{Obs.~\ref*{#1}}}}

\newcommand{\figref}[1]{\hyperref[#1]{\darkredtext{Figure~\ref*{#1}}}}
\newcommand{\tabref}[1]{\hyperref[#1]{\darkredtext{Table~\ref*{#1}}}}
\newcommand{\secref}[1]{\hyperref[#1]{\darkredtext{Section~\ref*{#1}}}}
\newcommand{\appref}[1]{\hyperref[#1]{\darkredtext{Appendix~\ref*{#1}}}}
\newcommand{\algref}[1]{\hyperref[#1]{\darkredtext{Algorithm~\ref*{#1}}}}
\newcommand{\equref}[1]{\hyperref[#1]{\darkredtext{Equation~\ref*{#1}}}}

\newtheoremstyle{custom}
{1pt} {1pt} {\itshape} {} {\bfseries} {} { } {\thmname{#1} \thmnumber{#2} \thmnote{(#3)} . } 
\theoremstyle{custom}

\newtheorem{innerdefinition}{Definition}
\newtheorem{innerproposition}{Proposition}
\newtheorem{innerassumption}{Assumption}
\newtheorem{innerremark}{Remark}
\newtheorem{innertheorem}{Theorem}
\newtheorem{innerhypothesis}{Hypothesis}
\newtheorem{innerconjecture}{Conjecture}
\newtheorem{innerlemma}{Lemma}
\newtheorem{innercorollary}{Corollary}
\newtheorem{innerexample}{Example}
\newtheorem{innernotation}{Notation}
\newtheorem{innerclaim}{Claim}
\newtheorem{innerproblem}{Problem}

\newtheorem{innerobservation}{Observation}

\tcbset{
  theoremframebase/.style={
    enhanced,
    breakable,
    sharp corners,
    boxrule=0pt,
    leftrule=2pt,
    top=6pt,
    bottom=5pt,
    left=5pt,
    right=5pt,
    before skip=4pt,
    after skip=4pt,
  }
}

\newtcolorbox{definitionframe}{theoremframebase, colback=black!3, colframe=black!55}
\newtcolorbox{propositionframe}{theoremframebase, colback=black!3, colframe=black!55}
\newtcolorbox{assumptionframe}{theoremframebase, colback=black!3, colframe=black!55}
\newtcolorbox{remarkframe}{theoremframebase, colback=black!3, colframe=black!55}
\newtcolorbox{theoremframe}{theoremframebase, colback=black!3, colframe=black!55}
\newtcolorbox{hypothesisframe}{theoremframebase, colback=black!3, colframe=black!55}
\newtcolorbox{conjectureframe}{theoremframebase, colback=black!3, colframe=black!55}
\newtcolorbox{lemmaframe}{theoremframebase, colback=black!3, colframe=black!55}
\newtcolorbox{corollaryframe}{theoremframebase, colback=black!3, colframe=black!55}
\newtcolorbox{exampleframe}{theoremframebase, colback=black!3, colframe=black!55}
\newtcolorbox{notationframe}{theoremframebase, colback=black!3, colframe=black!55}
\newtcolorbox{claimframe}{theoremframebase, colback=black!3, colframe=black!55}
\newtcolorbox{problemframe}{theoremframebase, colback=black!3, colframe=black!55}
\newtcolorbox{observationframe}{theoremframebase, colback=black!3, colframe=black!55}

\newenvironment{definition}
{\begin{definitionframe}\begin{innerdefinition}}
{\end{innerdefinition}\end{definitionframe}}

\newenvironment{proposition}
{\begin{propositionframe}\begin{innerproposition}}
{\end{innerproposition}\end{propositionframe}}

\newenvironment{remark}
{\begin{remarkframe}\begin{innerremark}}
{\end{innerremark}\end{remarkframe}}

\newenvironment{theorem}
{\begin{theoremframe}\begin{innertheorem}}
{\end{innertheorem}\end{theoremframe}}

\newenvironment{corollary}
{\begin{corollaryframe}\begin{innercorollary}}
{\end{innercorollary}\end{corollaryframe}}

\usepackage[textwidth=3.5cm]{todonotes}

\usepackage{makecell}

\usepackage{bm}

\definecolor{ForestGreen}{RGB}{34,139,34}
\definecolor{OrangeRed}{RGB}{255,69,0}

%% file: resources/main.tex
\definecolor{trenddn}{RGB}{0,128,128}
\definecolor{trendup}{RGB}{180,50,50}
\newcommand{\tmark}[3]{\raisebox{0.5em}{\scalebox{0.6}{\textcolor{#1}{#2}}}\hspace{0.05em}#3}

\begin{abstract}
Sampling complete solutions gives large language models more chances to solve a problem, but each attempt starts again.
For tasks with several subproblems, all required answers must be correct in the same attempt.
Retaining intermediate results can avoid repeating correct work, but errors in those results may affect later attempts.
We study when this reuse is reliable and how much computation it requires.
For local retries that stop when no candidate passes within the retry budget, we give conditions for correct completion and bound the average number of attempts.
We also show how checking cost and the chance of passing a check affect total computation.
Experiments across eight reasoning models examine composite mathematics and arithmetic.
Across four DeepSeek models, average success on the second subproblem is 36.2 percent without first-subproblem context, 57.2 percent with its answer, and 58.3 percent with its reasoning and answer.
These scores measure second-subproblem coverage and include continuations made when selection fails.
Arithmetic experiments show that similar first-attempt accuracies can lead to very different coverage after repeated attempts.
Code is available at \mbox{\url{https://github.com/SempraETY/Intermediate-State-Reuse}}.
\end{abstract}

\begin{figure}[H]
    \centering
    \includegraphics[width=\linewidth]{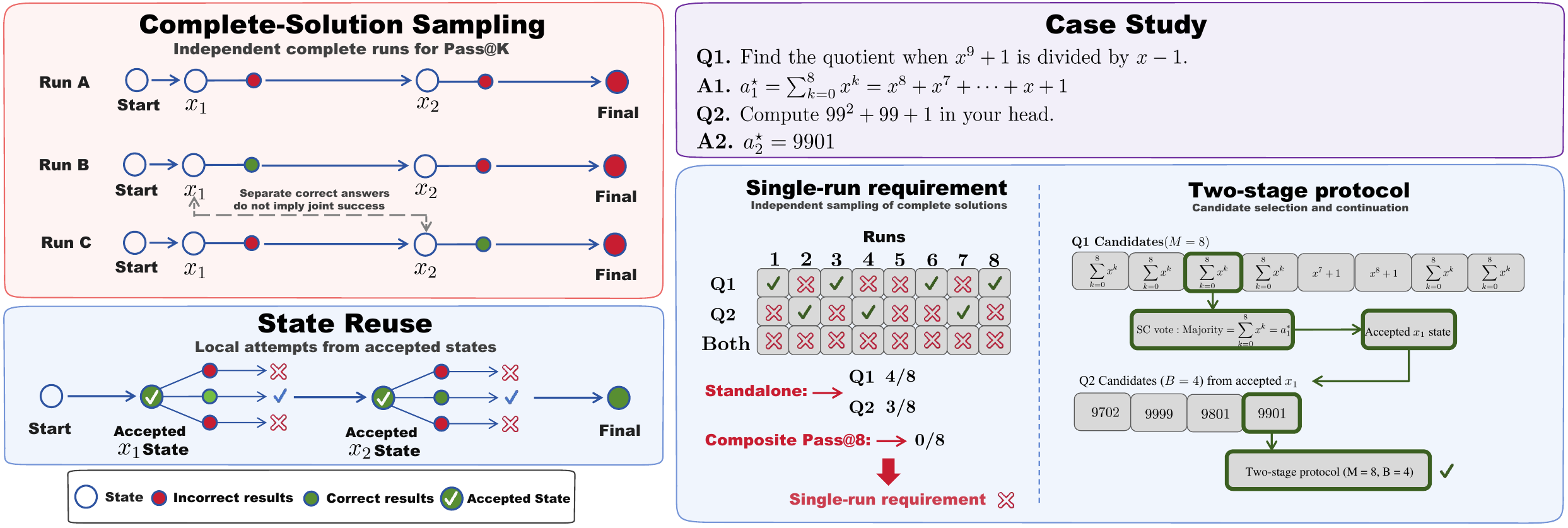}
    \caption{\textbf{Complete-solution sampling and intermediate-state reuse.}
    \textbf{Left:} Independent sampling requires a successful complete run. The analyzed reuse loop preserves earlier accepted work, retries rejected candidates, and stops when a step's budget is exhausted. Passing a check does not guarantee correctness.
    \textbf{Right:} An illustrative two-subproblem candidate-selection example, with SC denoting self-consistency, $M$ the first-subproblem candidate budget, and $B$ the continuation cap. Fractions count correct candidates among those shown; the example is not an experimental measurement or the single-state loop used in the proof.}
    \label{fig:framework_comparison}
\end{figure}
\clearpage
\begingroup
\small
\setlength{\parskip}{0pt}
\tableofcontents
\endgroup
\clearpage

\section{Introduction}
\label{sec:intro}

Many practical reasoning tasks require several distinct subproblems to be solved together. Solving each subproblem in isolation does not ensure that a model can solve the complete task~\citep{zhang2025reasoning}.
With model weights fixed, a system can spend more computation at inference time to improve its chance of finding a solution.
A system can generate more complete answers, compare candidates, or check intermediate work~\citep{brown2024large,wang2022self,lightman2023let}.
How this computation is allocated matters alongside how much is available~\citep{snell2025scaling}.

Under complete-solution sampling, all required subproblem answers must be correct in one run.
One run may solve the first subproblem and another the second, while neither solves both.
The sample pool contains correct answers to both subproblems, but no run has solved the complete task.
As tasks get longer, more intermediate work can be lost when an attempt fails near the end.
A restart retries the failed subproblem but repeats earlier work that may already be correct.
Existing methods already use intermediate work: Least-to-Most prompting solves subproblems in sequence using earlier answers~\citep{zhou2022least}, while Tree of Thoughts searches over partial reasoning states~\citep{yao2023tree}.

A subproblem may contain several reasoning steps; our analysis uses steps whose intermediate states can be checked and retained.
Consider a three-step calculation that derives a formula, substitutes the given values, and computes a final quantity.
Suppose the first two steps are correct but the final calculation fails.
A complete restart generates all three steps again.
If the formula and substitution can be checked and retained, a new attempt can repeat only the final calculation.
The retained results become the retry's input.

The difficulty is deciding what can safely be kept.
A check may reject a correct candidate or accept an incorrect one; benchmarks document the difficulty of identifying errors in intermediate reasoning~\citep{zheng2025processbench,li2025socratic}.
An error passed forward may then affect later attempts, as studies of self-conditioning illustrate~\citep{sinha2025illusion}.
Even a correct intermediate answer may leave out assumptions or other information needed to continue.
Atom of Thoughts constructs self-contained reasoning states~\citep{teng2025atom}, while MEM1 learns to retain compact information across interaction steps~\citep{zhou2025mem1}.
Keeping a full reasoning trace instead supplies more context, but may also retain unchecked mistakes and increases the context to process.
Checking itself consumes computation, so avoiding a restart does not by itself establish a saving.
This report studies the conditions under which intermediate results can be reused reliably and the cost of doing so.

Recent search methods also avoid unnecessary branching: Chain-in-Tree continues routine steps sequentially and reuses generated children when branching is needed~\citep{li2026chain}.
Process supervision evaluates intermediate reasoning~\citep{lightman2023let}, and MAKER combines fine-grained decomposition with voting to execute long deterministic tasks~\citep{meyerson2025solving}.
MAKER derives near-linear expected computation under its voting model, including the cost of discarded responses.
We consider finite local retries, where the first candidate to pass a check is retained and may carry an error into later steps.

The analyzed procedure checks each new result and retries from the last retained state after a rejection.
It stops when every candidate within the local budget is rejected.
For fixed generation and checking, retries reduce exhaustion without improving the reliability of accepted states.
We give an attainable reliability bound and matching worst-case bounds on expected attempts.
The work bound includes paths after an incorrect acceptance, where correct-prefix acceptance guarantees no longer apply.
Exhaustion limits the work spent at frequently rejected states.
Sufficient precision and acceptance probability yield linear expected attempts; total cost also includes checking and state processing.

We also report sampling experiments across eight reasoning models.
Arithmetic tasks let us ask how much repeated attempts add beyond first-attempt accuracy.
Pairs of mathematics problems let us compare candidate selection and the context supplied to the second subproblem: no first-subproblem context, its answer alone, or its reasoning and answer.
The paired subproblems are mathematically separate. Their continuation prompts differ in both content and instructions, so the comparison concerns entire prompt configurations.
Budget sweeps examine how success and generated tokens change when more computation is spent on \mbox{first-subproblem candidates or subsequent continuations}.

Across four DeepSeek models, second-subproblem success is 36.2\% without first-subproblem context, 57.2\% with its answer, and 58.3\% with its reasoning and answer.
Paired uncertainty is unavailable for these prompt comparisons.
These scores include fallback continuations and measure success on the second subproblem, rather than joint success of both answers.
Arithmetic results show that similar first-attempt accuracies can lead to very different coverage after repeated attempts.
A separate EdgeBench study compares two curves for predicting later agent scores from early observations and observations of the same agents on other tasks~\citep{edgebench2026}.\looseness=-1 

\Needspace{5\baselineskip}
\textbf{Report overview.}
\begin{enumerate}[label=(\alph*), nosep, leftmargin=12pt]
    \item \textbf{Reliability of reuse.} We give an attainable trajectory bound and the exact accepted-state error threshold for a finite cap to guarantee target reliability.
    \item \textbf{Expected work and checking cost.} We bound worst-case expected work, including erroneous paths, and connect reliable completion to acceptance probability and attempt cost.
    \item \textbf{Empirical findings.} We compare retries, answer selection, retained context, and budget allocation across eight reasoning models, and evaluate two curves for predicting later agent scores.
\end{enumerate}

\section{Preliminaries}
\label{sec:prelim}

We first define complete-task success, the state passed between steps, and the sampling metrics. These definitions distinguish finding a complete answer from continuing successfully after an intermediate result.

\subsection{Tasks, States, and Sampling Metrics}
\label{sec:ideal_laws}

A composite task $x_1\oplus x_2$ presents two subproblems in one prompt and is correct only when both returned subproblem answers are correct.
Let $a^\star(x)$ be the target answer to $x$, with $a^\star(x_1\oplus x_2)=[a^\star(x_1),a^\star(x_2)]$.
Let $\mtheta$ denote the model parameters and $p_{\mtheta}(\cdot\mid s)$ its distribution over generated text under a fixed decoding configuration and input context $s$.
For a problem $x$, fix a decomposition into $N\geq 1$ steps, where $N$ counts the steps in this chosen decomposition.
Write $z_t$ for the text generated at step $t\in\{1,\ldots,N\}$ and $s_t=[x,z_1,\ldots,z_t]$ for the resulting state, with $s_0=x$.
Each step is sampled from $p_{\mtheta}(\cdot\mid s_{t-1})$.
For reuse, $s_t$ denotes the state actually passed forward: task information plus retained generated text or extracted intermediate answers. The prompt specifies how this state is presented to the model. Correctness concerns the information carried forward; a correct extracted answer need not imply that discarded reasoning was correct.
We call a task \emph{stepwise decomposable} if it can be split into steps whose intermediate states admit a task-specific accept-or-reject check.
A state is \emph{accepted} if it passes the check; acceptance need not imply correctness. A \emph{stage} groups steps, such as those used to solve one subproblem.
Let $y$ be the final answer extracted from a completed run and $L$ its total number of generated tokens.
The single-run metrics are:
\begin{equation}
A_{\mtheta,1}(x)=P\!\left(y=a^\star(x)\right),\qquad C_{\mtheta,1}(x)=\mathbb{E}[L].
\label{eq:prelim_single_metrics}
\end{equation}

For an integer $K\geq 1$, independent sampling draws $K$ independent complete runs from the same distribution and counts success if any run returns the correct answer.
\begin{equation}
A_{\mtheta,K}(x)=1-\left(1-A_{\mtheta,1}(x)\right)^K,\qquad C_{\mtheta,K}(x)=K\,C_{\mtheta,1}(x).
\label{eq:prelim_passk}
\end{equation}
Let $\delta\in(0,1)$ denote the target failure probability, so the desired reliability is $1-\delta$.

\subsection{Complete-Solution Sampling}
In a common pool of composite runs, joint success implies that each subproblem is covered, while separately correct subproblem answers may occur in different runs (\propref{prop:breakdown}).
If standalone and composite subproblem marginals agree at $a_1,a_2\in(0,1)$, with independence within and across runs, then for $K\geq2$,
\begin{equation}
\begin{aligned}
1-(1-a_1a_2)^K
&<[1-(1-a_1)^K]\\
&\quad\cdot[1-(1-a_2)^K].
\end{aligned}
\label{eq:main_sampling_comparison}
\end{equation}
One run must solve both subproblems, whereas separate coverage can use different runs.
This sampling constraint also applies to other stochastic solvers; it does not equate token costs across prompt formats.
Under the additional assumptions of exponentially declining complete-run accuracy and linear generation cost~\citep{zhang2025reasoning}, complete restarts require exponentially growing generation cost at fixed reliability.
\appref{app:sampling_background} gives the product comparison, numerical example, and scaling calculation.
We next examine the reliability and cost of retaining intermediate work.

\section{Reliability and Cost of State Reuse}
\label{sec:connected_solution}

The reuse loop maintains one accepted state at a time. We analyze the probability that every retained state is correct, the number of executed attempts, and the resulting computational cost.

\subsection{Local Retries and Reliable Completion}
\label{sec:vsr_procedure}
\label{sec:failure_decomposition}

At step $t$, generate a candidate from the retained state $s_{t-1}^{\checkmark}$ and check the state to be passed forward.
Let $\cV_t$ indicate acceptance and $\cG_t$ its correctness.
Retain the first accepted state; if all $R$ attempts are rejected, terminate the entire task.
Conditional on the complete history at entry to each reached step, attempts are independent and identically distributed, including at incorrect inputs.
Rejected candidates do not change the state or retry law; different steps may depend on earlier accepted results.
Checking an extracted answer does not establish that a retained reasoning trace is also correct.

At a correct entry history $s$, let $\rho_t(s)$ be the rejection probability and $w_t(s)$ the probability of incorrect acceptance in one attempt.
Assume $\rho_t(s)\leq\rho\in[0,1)$ and accepted-state error $\alpha_t(s)\leq\alpha\in[0,1)$ at every correct entry history.
Conditioning on the first acceptance gives
\begin{equation}
\alpha_t(R;s)=\frac{w_t(s)}{1-\rho_t(s)}=\alpha_t(s).
\label{eq:fixed_check_contamination}
\end{equation}
Accepted-state precision is $1-\alpha_t(s)$; it does not improve with $R$.
The local trajectory fails if the retry budget is exhausted or an incorrect candidate is accepted:
\begin{equation}
\begin{aligned}
\beta_t(R;s)&=\rho_t(s)^R\\
&\quad+\bigl(1-\rho_t(s)^R\bigr)\alpha_t(s).
\end{aligned}
\label{eq:vsr_local_failure}
\end{equation}
More retries reduce exhaustion; improving accepted-state precision requires a different generation or checking rule.
Let $H_t$ mean that the first $t$ steps complete with every retained state correct, with $H_0$ certain.
Assume $H_N$ yields the target answer; later repair can also succeed outside $H_N$.

\begin{theorem}[Trajectory reliability]
\label{thm:reliability}
Under the preceding assumptions,
\begin{equation}
\Pr(H_N)\geq[(1-\rho^R)(1-\alpha)]^N.
\label{eq:sharp_reliability}
\end{equation}
An admissible process attains equality for every $N,R$.
For $0<\rho<1$, a finite common cap guarantees $1-\delta$ over this class if and only if
\begin{equation}
\alpha<1-(1-\delta)^{1/N}.
\label{eq:precision_boundary}
\end{equation}
\end{theorem}

Each correct prefix advances correctly with conditional probability at least $(1-\rho^R)(1-\alpha)$; the chain rule gives the product without independence between steps.
A homogeneous process with irreversible errors after acceptance attains it.
At equality in Eq.~\eqref{eq:precision_boundary}, the target is approached only with unbounded retries; for $\rho=0$, equality is allowed and one attempt suffices.
This is a worst-case guarantee, not a necessary precision threshold for every system that can repair errors.
\appref{app:new_reliability} gives the proof, exact minimum cap, and the simpler sufficient allocation $N\alpha<\delta$ used in \propref{prop:logarithmic_retry}.

\subsection{Expected Work under the Stopping Rule}
\label{sec:compositionality_restore}
\label{sec:work}

We measure work by $W$, the number of executed candidate attempts.
The pathwise cap $W\leq NR$ reserves the full budget at every step. Actual work can be smaller because a step ends as soon as a candidate passes.
Let $\mathcal P_{N,R}$ contain all processes following the preceding retry and stopping rules for at most $N$ steps, with arbitrary history-dependent rejection probabilities and no correctness requirement.
If acceptance were at least $h>0$ at every reached state, the geometric waiting-time bound would give $\mathbb E[W]\leq N/h$ directly.
After an incorrect acceptance, the process may reach states where the correct-prefix acceptance bound no longer holds.
The next theorem includes work at these states.

\begin{theorem}[Worst-case expected attempts]
\label{thm:work}
For every pair of integers $N,R\geq1$,
\begin{equation}
\sup_{P\in\mathcal P_{N,R}}\mathbb E_P[W]
=\Theta\!\left(N+\frac{NR}{\log(N+1)}\right),
\label{eq:matching_work}
\end{equation}
with positive implicit constants independent of both the process length $N$ and the retry cap $R$.
\end{theorem}

Frequently rejected states are expensive to leave but likely to exhaust the cap, ending the run.
The proof uses this tradeoff to bound their expected visits; a homogeneous process gives the matching lower bound (\appref{app:work_proof}).
That lower bound concerns the unrestricted class, not every reliable process.
The stopping pattern is related to classical start-up demonstration tests~\citep{gera2013startup}.
\appref{app:stopping_context} gives the exact dynamic recursion and separates this uniform bound from fixed-law waiting-time calculations.
Together with \thmref{thm:reliability}, this yields reliable completion with an $O(\log N)$ cap and linear expected work.
Here the target failure probability is fixed, acceptance at correct prefixes has a positive constant lower bound, and accepted-state error decreases with chain length as specified in \corref{cor:coverage_cost}.

\paragraph{Example: reserved budget and executed work.}
For 100 steps, a cap of 12 attempts reserves at most 1,200 attempts.
Suppose every reached state has rejection probability $1/2$, and every correct prefix has accepted-state error at most $0.0002$.
The reliability bound exceeds $0.956$, while each reached step uses fewer than two attempts in expectation, giving fewer than 200 over the task.
The probabilities in this example are hypothetical.
The example distinguishes the reserved budget from expected work; \thmref{thm:work} additionally permits arbitrary rejection probabilities after an incorrect acceptance.

\subsection{Checking Precision and Computation}
\label{sec:checking_cost}

As a chain grows, the reliability bound requires a smaller error rate among accepted states. Reaching this precision may make checking more expensive or cause it to reject more candidates.
Let $h_N$ lower-bound acceptance probability at every correct prefix, and let $g_N$ uniformly bound per-attempt cost, including generation, checking, and state processing at every reached state.

\begin{corollary}[Precision, acceptance, and cost]
\label{cor:coverage_cost}
Fix $\delta\in(0,1)$ and $c\in[0,1)$.
If correct prefixes have acceptance probability at least $h_N\in(0,1]$ and accepted-state error at most $c\delta/N$, choose
\begin{equation}
R_N=\left\lceil\frac{1}{h_N}\log\frac{N}{(1-c)\delta}\right\rceil.
\label{eq:coverage_cap}
\end{equation}
Then $\Pr(H_N)\geq1-\delta$ and
\begin{equation}
\mathbb E[W]=O(N/h_N),\qquad
C_{\mathrm{vsr}}=O(Ng_N/h_N).
\label{eq:coverage_cost}
\end{equation}
The attempt order $N/h_N$ is attained within this reliable class for this cap.
\end{corollary}

The cap controls exhaustion; \thmref{thm:work} bounds all executed attempts, including those after an incorrect acceptance (\appref{app:reliable_cost}).
For constant $h_N$, expected attempts are linear, whereas the deterministic cost cap remains $O(N\log N\,g_N)$.
Since $W\geq N\mathbf1_{H_N}$ and $\mathbb E[W\mid H_N]\leq\mathbb E[W]/(1-\delta)$, this conclusion is compatible with reliable completion, rather than relying only on cheap failures.

At a common precision target, the relevant cost scale is $g_N/h_N$.
Under independent binary checks, requiring unanimity can make acceptance polynomially rare, yielding $\Theta(N^{1+\gamma}\log N)$ expected cost for some $\gamma>0$ when every check is executed.
With an additional separation condition, majority checking keeps acceptance bounded away from zero and gives $O(N\log N)$ expected cost~\citep{hoeffding1963probability}.
\appref{app:checking_proofs} specifies the noise model and derives both cases; these repeatedly check one candidate rather than compare different answers by self-consistency.

The work bound relies on the execution rule: feedback-dependent retries or continuation after exhaustion can each use $NR$ attempts.
The experiments below use batch selection, multiple continuations, and fallback with different success criteria. They do not directly test these guarantees.

\section{Experiments}
\label{sec:experiments}

We compare coverage across budgets, selection rules, and continuation prompts.
The tasks cover composite mathematics, complete-answer arithmetic retries, and agent-score prediction.

\subsection{Experimental Setup}
\label{sec:exp_setup}

\paragraph{Models and tasks.}
We evaluate DeepSeek-R1-Distill (1.5B, 7B, 8B, 14B)~\citep{DeepSeekR125}, Qwen3-Thinking (4B, 30B-A3B)~\citep{qwen3report}, and GPT-OSS (20B, 120B)~\citep{openai2025gptoss}.
LoRe-Compo contains 250 pairs of distinct-subject MATH500 problems, using each of the 500 subproblems once~\citep{zhang2025reasoning,hendrycksmath2021}.
The second subproblem is a separate mathematical problem that does not require the first answer.
Arithmetic uses 50 problems for each base (7 or 10), operation (addition or multiplication), and digit count $n\in\{1,\ldots,10\}$, with 32 complete-answer attempts per problem.

\paragraph{Procedures and success criteria.}
Composite Pass@$K$ requires both correct answers in one of $K$ complete runs.
The two-stage protocol instead generates $M$ answers to $x_1$, retains at most $B$ candidates, and continues on $x_2$.
Oracle selects reference-matching answers; self-consistency (SC) uses the modal answer when its vote share reaches 0.5.
If selection returns none, fallback candidates are retained for continuation.
Both procedures generate whole subproblem answers without checking individual reasoning steps.\looseness=-1 

The reported second-subproblem success rate $A_2^{\mathrm{ret}}$ counts any correct answer to $x_2$ among the retained continuations, including fallbacks and regardless of $x_1$ correctness.
Joint retained-path success would require both answers to be correct along one shared path, with all 250 pairs still in the denominator.
Thus $A_2^{\mathrm{ret}}$ upper-bounds joint coverage but does not determine it.
It measures coverage, not final-answer selection accuracy.
The theoretical loop returns one first-accepted state and stops at exhaustion; the experimental protocol uses multiple continuations and fallback.
\appref{app:metric_alignment} defines these events.

\paragraph{Sampling, matching, and cost.}
All mathematical generations have a 10,240-token cap.
A composite run shares this cap across two answers, whereas staged solving grants one cap per subproblem.
LoRe-Compo uses family-specific decoding settings and direct prompt strings without model-specific chat templates.
Decoding differs between Pass@1 and larger budgets, some Pass@1 values come from prior work, and cross-budget nesting is unspecified for the composite results.
Answer matching uses numeric tolerance or normalized strings, with separate extraction rules for single and paired answers (\appref{app:lore_compo_impl}).
We report recorded generated-token counts and success percentage points per thousand tokens.
Equal attempt counts allow different generation budgets. The token counts exclude some checking and state processing costs and are not matched for whole-task reliability.

\begin{table}[t]
    \centering
    \begin{minipage}[t]{\dimexpr0.5\textwidth-5pt\relax}
        \vspace{0pt}
        \captionsetup{font=footnotesize,justification=raggedright,singlelinecheck=false,skip=3pt}
\caption{\textbf{Second-subproblem success after selection.} $A_2^{\mathrm{ret}}$ counts a correct $x_2$ answer on any retained path, including fallbacks, regardless of $x_1$ correctness. Oracle and SC share $M=8$, $B=4$. This continuation metric differs from the complete-solution joint-success criterion in Table~\ref{table:compositionality_breakdown}.}
        \label{table:connected_restoration}
    \end{minipage}\hfill
    \begin{minipage}[t]{\dimexpr0.5\textwidth-5pt\relax}
        \vspace{0pt}
        \captionsetup{font=footnotesize,justification=raggedright,singlelinecheck=false,skip=3pt}
\caption{\textbf{Composite Pass@$K$ under Independent Sampling.} Success requires both correct answers in one of $K$ complete runs. DeepSeek $K=1$ values are from~\citet{zhang2025reasoning}; other cells are reported aggregates from this study. Superscripts give percentage-point changes from $K=1$; dashes indicate unreported cells. Decoding varies with $K$, and cross-budget sample nesting is unspecified.}
        \label{table:compositionality_breakdown}
    \end{minipage}\par
    \nointerlineskip
    \begin{minipage}[t]{\dimexpr0.5\textwidth-5pt\relax}
        \vspace{0pt}
        \centering
        \fontsize{8}{10}\selectfont
        \renewcommand{\arraystretch}{1.2}
        \setlength{\tabcolsep}{0.8pt}
\begin{tabular*}{\linewidth}{@{\extracolsep{\fill}}lcc@{}}
    \toprule[0.95pt]
    \multirow{2}{*}{\textbf{Model}} & \multicolumn{2}{c}{\textbf{$A_2^{\mathrm{ret}}$ (\%)}} \\
    \cmidrule(lr){2-3}
    & \textbf{Oracle} & \textbf{SC} \\
    \midrule[0.6pt]
    DeepSeek-R1-1.5B & 55.6 & 55.2 \\
    DeepSeek-R1-7B & 61.6 & 61.6 \\
    DeepSeek-R1-8B & 51.2 & 52.8 \\
    DeepSeek-R1-14B & 60.4 & 62.0 \\
    \midrule[0.4pt]
    Qwen3-4B & 53.6 & 54.4 \\
    Qwen3-30B-A3B & 58.4 & 58.0 \\
    \midrule[0.4pt]
    GPT-OSS-20B & 47.2 & 48.4 \\
    GPT-OSS-120B & 62.4 & 62.4 \\
    \bottomrule[0.95pt]
\end{tabular*}
    \end{minipage}\hfill
    \begin{minipage}[t]{\dimexpr0.5\textwidth-5pt\relax}
        \vspace{0pt}
        \centering
        \fontsize{8}{10}\selectfont
        \renewcommand{\arraystretch}{1.2}
        \setlength{\tabcolsep}{2.0pt}
        \renewcommand{\tmark}[3]{#3\hspace{0.03em}\raisebox{0.5em}{\fontsize{6}{6}\selectfont\textcolor{#1}{#2}}}
\begin{tabular*}{\linewidth}{@{\extracolsep{\fill}}lcccc@{}}
    \toprule[0.95pt]
    \multirow{2}{*}[-0.7ex]{\textbf{Model}} & \multicolumn{4}{c}{\vphantom{\shortstack{\textbf{$x_2$ success}\\[-1pt]\textbf{(\%)}}}\textbf{Pass@$K$ Coverage (\%)}} \\
    \cmidrule(lr){2-5}
    & \textbf{1} & \textbf{8} & \textbf{32} & \textbf{128} \\
    \midrule[0.6pt]
    DeepSeek-R1-1.5B & 4.0 & \tmark{trenddn}{$\uparrow$13.2}{17.2} & \tmark{trenddn}{$\uparrow$24.4}{28.4} & -- \\
    DeepSeek-R1-7B & 2.4 & \tmark{trenddn}{$\uparrow$12.8}{15.2} & \tmark{trenddn}{$\uparrow$23.6}{26.0} & \tmark{trenddn}{$\uparrow$30.0}{32.4} \\
    DeepSeek-R1-8B & 13.2 & \tmark{trenddn}{$\uparrow$17.2}{30.4} & \tmark{trenddn}{$\uparrow$20.8}{34.0} & \tmark{trenddn}{$\uparrow$24.0}{37.2} \\
    DeepSeek-R1-14B & 2.0 & \tmark{trenddn}{$\uparrow$12.4}{14.4} & \tmark{trenddn}{$\uparrow$19.2}{21.2} & \tmark{trenddn}{$\uparrow$23.2}{25.2} \\
    \midrule[0.4pt]
    Qwen3-4B & 4.8 & \tmark{trenddn}{$\uparrow$22.4}{27.2} & -- & \tmark{trenddn}{$\uparrow$33.2}{38.0} \\
    Qwen3-30B-A3B & 2.8 & \tmark{trenddn}{$\uparrow$12.0}{14.8} & \tmark{trenddn}{$\uparrow$24.4}{27.2} & \tmark{trenddn}{$\uparrow$28.8}{31.6} \\
    \midrule[0.4pt]
    GPT-OSS-20B & 16.8 & \tmark{trenddn}{$\uparrow$18.0}{34.8} & \tmark{trenddn}{$\uparrow$16.4}{33.2} & \tmark{trenddn}{$\uparrow$18.4}{35.2} \\
    GPT-OSS-120B & 19.6 & \tmark{trenddn}{$\uparrow$17.2}{36.8} & \tmark{trenddn}{$\uparrow$18.4}{38.0} & \tmark{trenddn}{$\uparrow$18.8}{38.4} \\
    \bottomrule[0.95pt]
    \end{tabular*}
    \end{minipage}
\end{table}

\subsection{Complete-Solution Sampling on Composite Tasks}
\label{sec:topological_failure}

\tabref{table:compositionality_breakdown} reports joint composite Pass@$K$.
The eight-model mean is 8.2\% at $K=1$ and 33.3\% at the largest recorded budget ($K=32$ for R1-1.5B and $K=128$ otherwise).
Coverage improves across these endpoint configurations, but the gains and model ordering vary.
Qwen3-4B starts at 4.8\%, below R1-8B at 13.2\%, and reaches 38.0\% compared with 37.2\% at $K=128$.
GPT-OSS-120B has the highest Pass@1 and Pass@128, at 19.6\% and 38.4\%.
Model rankings depend on the sampling budget.

GPT-OSS-20B changes from 34.8\% at $K=8$ to 33.2\% at $K=32$.
Because nesting and decoding are not held fixed across all configurations, this decrease does not show that adding samples to an unchanged pool reduces coverage.
The table reports how often a complete run answers both subproblems correctly.
Measuring unused partial success requires the statistic in \appref{app:metric_alignment}.

\subsection{Candidate Selection and Retained Context}
\label{sec:restoration}

\paragraph{Intermediate-answer selection.}
\tabref{table:connected_restoration} compares Oracle and SC at $M=8$, $B=4$ using the same second-subproblem criterion.
The four-DeepSeek SC mean is 57.9\%; across all eight models, SC and Oracle differ by at most 1.6 percentage points.
Here, Oracle and SC yield similar second-subproblem success despite using different selection rules.
These scores include fallback continuations, so they do not show that the selectors retain equally accurate first answers.
The subproblems are mathematically separate, so a wrong first answer need not prevent success on the second.

The reported SC-to-Pass@8 generated-token ratios are approximately 1.01--1.12 for R1-1.5B, R1-7B, R1-14B, and GPT-OSS-20B, and 1.4--1.5 for the other models.
These procedures use different success criteria, so the ratios cannot establish savings at matched whole-task success rates.

\Needspace{18\baselineskip}
\begin{wraptable}{r}{0.4\textwidth}
    
    \centering
    \captionsetup{font=footnotesize,justification=raggedright,singlelinecheck=false,skip=4pt}
\caption{\textbf{Context for the Second Subproblem.} Four-DeepSeek means at the same continuation budget. $A_2^{\mathrm{ret}}$ counts second-subproblem success, including fallbacks. Change is the relative percentage increase over \texttt{no\_context}. Strategies supply no first-subproblem context, its answer only, or its reasoning and answer.}
    \label{table:prompt_ablation}
    \renewcommand{\arraystretch}{1.15}
    \setlength{\tabcolsep}{2.5mm}
    \centering\small
\begin{tabular*}{\linewidth}{@{\extracolsep{\fill}}lcc@{}}
    \toprule[0.95pt]
    
    \textbf{Strategy} & \textbf{$A_2^{\mathrm{ret}}$} & \textbf{Change} \\
    \midrule[0.6pt]

    \textit{no\_context} & 36.2\% & 0\% \\
    \textit{answer\_only} & 57.2\% & +58\% \\
    \textit{full\_context} & 58.3\% & +61\% \\
    \bottomrule[0.95pt]
    \end{tabular*}
\end{wraptable}

\paragraph{Continuation-prompt configurations.}
\label{sec:synergy_deep_dive}
At the same continuation budget, the four-DeepSeek means are 36.2\% for \texttt{no\_context}, 57.2\% for \texttt{answer\_only}, and 58.3\% for \texttt{full\_context} (\tabref{table:prompt_ablation}).
The answer-only prompt scores 21.0 percentage points above no context; full context scores another 1.1 points higher.
This small gap does not establish that the two prompts are equivalent.
The answer-only mean matches the Oracle mean in Table~\ref{table:connected_restoration}; the SC mean reported above is a separate result. The available prompt-comparison summary does not identify the selector configuration.

The prompts differ in instructions as well as content; these scores compare complete configurations and do not isolate the contribution of the first answer.
The answer-only protocol retains the same extracted answer assessed during selection, whereas full context also retains an unchecked reasoning trace truncated to fit the context budget.

\FloatBarrier
\begin{table*}[t]
    
    \caption{\textbf{Budget Allocation Sweep.} $M$ is the first-subproblem candidate budget and $B$ the cap retained for continuation. Succ. is second-subproblem success (\%), including fallbacks, regardless of first-answer correctness. Tokens reports recorded counts; Eff. divides Succ. by Tokens in thousands. Average-row efficiencies divide mean success by mean tokens; they are not means of model-level efficiencies. Superscripts give percentage-point differences from $(4,2)$. Orange marks the reference configuration $(8,4)$, not an efficiency optimum; blue marks configurations with higher success and a larger measured token cost.}
    \label{table:budget_allocation}
    
    \centering
    \renewcommand{\arraystretch}{1.2}
    \setlength{\tabcolsep}{1.6mm}
    \resizebox{0.95\textwidth}{!}{
    \begin{tabular}{lcccccccccccc}
    \toprule[0.95pt]
    \multirow{2}{*}[-0.7ex]{\textbf{Model}} & \multicolumn{3}{c}{$M{=}4, B{=}2$} & \multicolumn{3}{c}{$M{=}8, B{=}2$} & \multicolumn{3}{c}{$M{=}8, B{=}4$} & \multicolumn{3}{c}{$M{=}16, B{=}8$} \\
    \cmidrule(lr){2-4}\cmidrule(lr){5-7}\cmidrule(lr){8-10}\cmidrule(lr){11-13}
    & \textbf{Succ.} & \textbf{Tokens} & \textbf{Eff.} & \textbf{Succ.} & \textbf{Tokens} & \textbf{Eff.} & \textbf{Succ.} & \textbf{Tokens} & \textbf{Eff.} & \textbf{Succ.} & \textbf{Tokens} & \textbf{Eff.} \\
    \midrule[0.6pt]

    DeepSeek-R1-1.5B & 53.6\% & 25.1K & 2.13 & 54.0\% & 46.8K & 1.15 & \cellcolor[RGB]{250,235,215}\tmark{trenddn}{$\uparrow$2.0}{55.6\%} & \cellcolor[RGB]{250,235,215}50.0K & \cellcolor[RGB]{250,235,215}1.11 & \cellcolor[RGB]{220,235,250}\tmark{trenddn}{$\uparrow$3.2}{56.8\%} & 100.6K & 0.56 \\
    DeepSeek-R1-7B & 61.2\% & 29.2K & 2.09 & 61.2\% & 52.8K & 1.16 & \cellcolor[RGB]{250,235,215}\tmark{gray}{$\uparrow$.4}{61.6\%} & \cellcolor[RGB]{250,235,215}57.4K & \cellcolor[RGB]{250,235,215}1.07 & 61.2\% & 115.0K & 0.53 \\
    DeepSeek-R1-8B & 41.2\% & 59.6K & 0.69 & 44.8\% & 100.7K & 0.44 & \cellcolor[RGB]{250,235,215}\tmark{trenddn}{$\uparrow$10.0}{51.2\%} & \cellcolor[RGB]{250,235,215}116.7K & \cellcolor[RGB]{250,235,215}0.44 & \cellcolor[RGB]{220,235,250}\tmark{trenddn}{$\uparrow$12.4}{53.6\%} & 233.0K & 0.23 \\
    DeepSeek-R1-14B & 59.6\% & 32.6K & 1.83 & 60.0\% & 57.7K & 1.04 & \cellcolor[RGB]{250,235,215}\tmark{trenddn}{$\uparrow$.8}{60.4\%} & \cellcolor[RGB]{250,235,215}63.7K & \cellcolor[RGB]{250,235,215}0.95 & \cellcolor[RGB]{220,235,250}\tmark{trenddn}{$\uparrow$1.6}{61.2\%} & 128.1K & 0.48 \\
    Qwen3-4B & 47.2\% & 58.7K & 0.80 & 48.0\% & 100.3K & 0.48 & \cellcolor[RGB]{250,235,215}\tmark{trenddn}{$\uparrow$6.4}{53.6\%} & \cellcolor[RGB]{250,235,215}117.6K & \cellcolor[RGB]{250,235,215}0.46 & 53.2\% & 234.7K & 0.23 \\
    Qwen3-30B-A3B & 54.4\% & 55.4K & 0.98 & 57.2\% & 95.2K & 0.60 & \cellcolor[RGB]{250,235,215}\tmark{trenddn}{$\uparrow$4.0}{58.4\%} & \cellcolor[RGB]{250,235,215}110.4K & \cellcolor[RGB]{250,235,215}0.53 & \cellcolor[RGB]{220,235,250}\tmark{trenddn}{$\uparrow$6.0}{60.4\%} & 219.1K & 0.28 \\
    GPT-OSS-20B & 42.8\% & 15.9K & 2.69 & 48.4\% & 28.2K & 1.72 & \cellcolor[RGB]{250,235,215}\tmark{trenddn}{$\uparrow$4.4}{47.2\%} & \cellcolor[RGB]{250,235,215}33.2K & \cellcolor[RGB]{250,235,215}1.42 & \cellcolor[RGB]{220,235,250}\tmark{trenddn}{$\uparrow$7.6}{50.4\%} & 68.3K & 0.74 \\
    GPT-OSS-120B & 62.0\% & 14.4K & 4.31 & 62.4\% & 25.4K & 2.46 & \cellcolor[RGB]{250,235,215}\tmark{gray}{$\uparrow$.4}{62.4\%} & \cellcolor[RGB]{250,235,215}29.0K & \cellcolor[RGB]{250,235,215}2.15 & 62.4\% & 57.3K & 1.09 \\
    \midrule[0.4pt]

    \textit{DeepSeek Avg} & 53.9\% & 36.6K & 1.47 & 55.0\% & 64.5K & 0.85 & \cellcolor[RGB]{250,235,215}\tmark{trenddn}{$\uparrow$3.3}{57.2\%} & \cellcolor[RGB]{250,235,215}71.9K & \cellcolor[RGB]{250,235,215}0.80 & \cellcolor[RGB]{220,235,250}\tmark{trenddn}{$\uparrow$4.3}{58.2\%} & 144.2K & 0.40 \\
    \textit{All Avg} & 52.8\% & 36.4K & 1.45 & 54.5\% & 63.4K & 0.86 & \cellcolor[RGB]{250,235,215}\tmark{trenddn}{$\uparrow$3.5}{56.3\%} & \cellcolor[RGB]{250,235,215}72.3K & \cellcolor[RGB]{250,235,215}0.78 & \cellcolor[RGB]{220,235,250}\tmark{trenddn}{$\uparrow$4.6}{57.4\%} & 144.5K & 0.40 \\
    \bottomrule[0.95pt]
    \end{tabular}
    }
\end{table*}

\begin{figure}[t]
    \centering
	    \includegraphics[width=\linewidth]{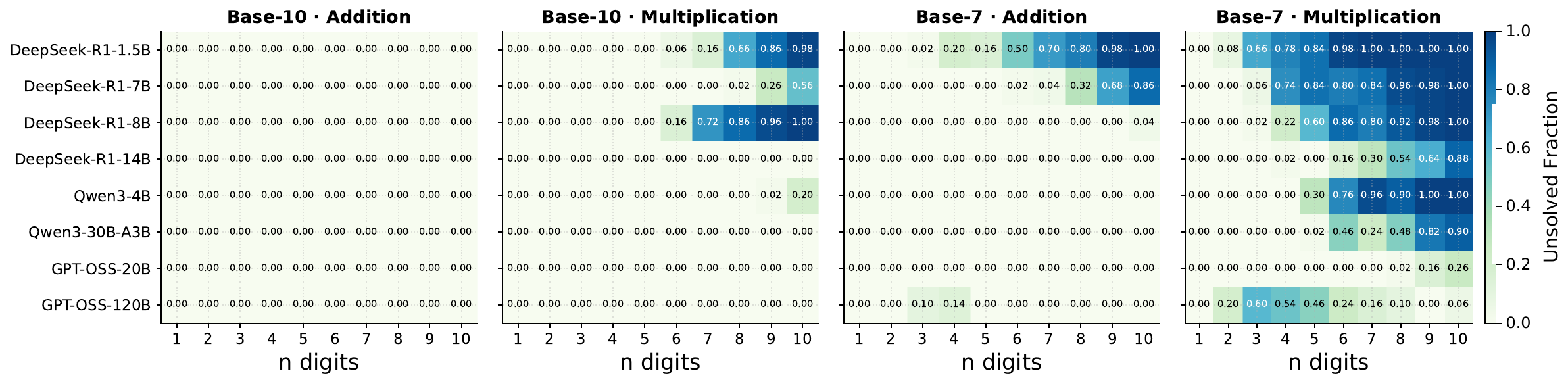}

\caption{\textbf{Unsolved Fraction after 32 Attempts.} Each cell reports $\hat E(32)$, the fraction of 50 problems unsolved after 32 attempts; $n$ is operand digit count. Panels show Base 10 addition/multiplication and Base 7 addition/multiplication. All Base 10 addition problems are solved; Base 7 multiplication has the largest mean unsolved fraction. For R1-7B, the latter rises from 0.06 at $n=3$ to 0.74 at $n=4$. GPT-OSS-120B shows nonmonotonic behavior at low digit counts under the recorded sampling protocol.}
    \label{fig:floor_heatmap}
    
\end{figure}

\paragraph{How candidate and continuation budgets trade off.}
\label{sec:pareto}
\tabref{table:budget_allocation} compares four tested allocations. Its $(8,4)$ success column reproduces the Oracle column of Table~\ref{table:connected_restoration}; the aggregate table alone does not identify the selector used for every other allocation.
The eight-model mean rises from 52.8\% at 36.4K tokens for $(4,2)$ to 56.3\% at 72.3K for $(8,4)$, a 3.5-point gain for approximately twice the generation.
Moving to $(16,8)$ gives 57.4\% at 144.5K, another 1.1 points for approximately twice the tokens again.
The second budget expansion yields a smaller gain in success.
We use $(8,4)$ as an intermediate reference; $(4,2)$ has higher success per thousand generated tokens.

The intermediate allocations separate changes in candidate count from changes in continuation count.
At $B=2$, increasing $M$ from 4 to 8 raises mean success by 1.7 points, from 52.8\% to 54.5\%, for 27.0K additional tokens.
At $M=8$, increasing $B$ from 2 to 4 raises mean success by 1.8 points, from 54.5\% to 56.3\%, for 8.9K additional tokens.
Increasing the continuation budget therefore gives a similar mean gain with fewer additional generated tokens in this comparison.

The model-level responses also vary.
At $M=8$, raising $B$ from 2 to 4 changes R1-8B from 44.8\% to 51.2\%, while GPT-OSS-120B stays at 62.4\%.
GPT-OSS-20B changes from 48.4\% to 47.2\%; without nested pools this cannot be attributed to keeping more candidates.
Across $(4,2)$ to $(8,4)$, R1-8B gains 10.0 points, compared with 0.4 points each for R1-7B and GPT-OSS-120B.
The gains therefore depend on both the model and the allocation of the sampling budget.

\paragraph{Resolution of the comparisons.}
With 250 pairs per model, one changed pair corresponds to 0.4 percentage points.
The largest SC--Oracle difference is a net difference of four successful pairs.
Paired uncertainty is unavailable for these comparisons and for the 1.1-point full-context increment, so we interpret the figures as descriptive comparisons of the evaluated configurations.

\clearpage
\Needspace{23\baselineskip}
\subsection{Finite-Budget Arithmetic Retries}
\label{sec:two_component}

\begin{wrapfigure}{r}{0.4\textwidth}
    
    \centering
    \captionsetup{font=footnotesize,justification=raggedright,singlelinecheck=false,skip=4pt}
    \includegraphics[width=\linewidth]{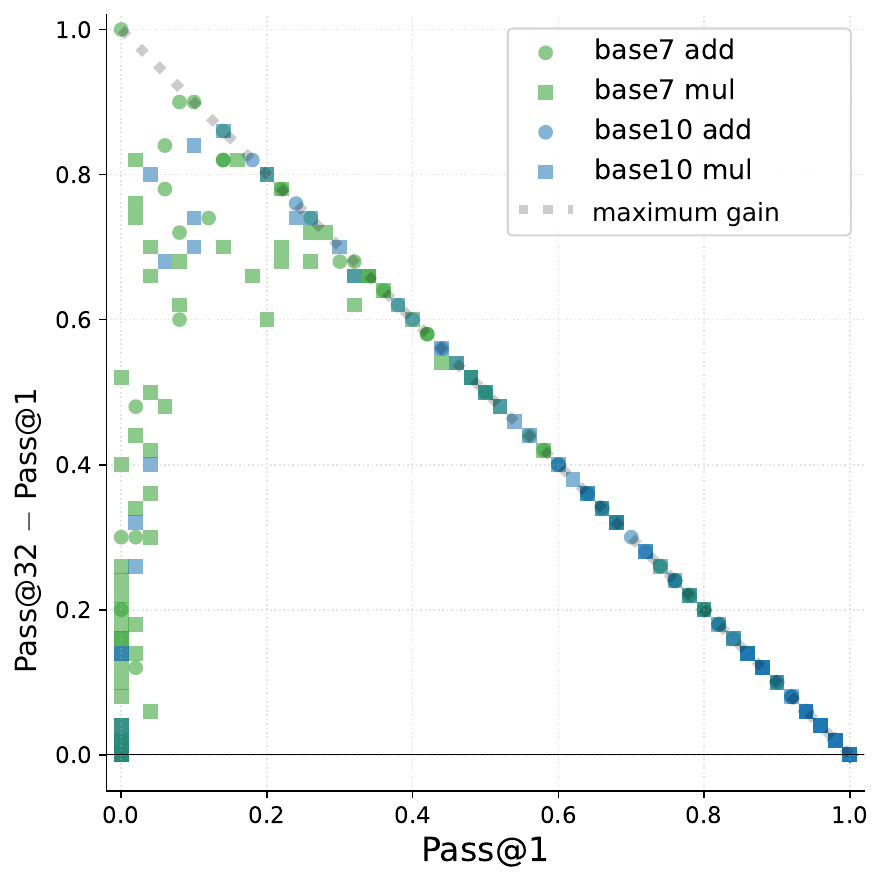}
\caption{\textbf{Gain from Additional Attempts.} Each point is a combination of model, base, operation, and digit count. The gain is Pass@32 minus Pass@1; the dashed line is its maximum $1-\mathrm{Pass}@1$. Large observed gains occur around Pass@1 values of 0.05--0.35 in these evaluated arithmetic conditions.}
    \label{fig:retry_sweet_spot}
    
\end{wrapfigure}

The arithmetic experiments ask how many initially missed problems can be solved with more attempts.
Here the candidate is an entire arithmetic answer, not a checked digit-level state.
We write $\hat E(R)$ for the fraction of problems with no correct answer among their first $R$ attempts; empirical Pass@$R$ is $1-\hat E(R)$.
The checkpoints use prefixes of the same pool of 32 complete-answer attempts.

Across 320 model--base--operation--digit conditions, 235 have no unsolved problems after 32 attempts and 85 have a nonzero unsolved fraction (\figref{fig:floor_heatmap}).
All 80 Base 10 addition conditions have zero unsolved fraction; Base 7 multiplication has a mean of 0.40, with 33 of 80 conditions above 0.5.
The dependence on digit count varies across models, including nonmonotonic behavior at low digit counts for GPT-OSS-120B.
Large retry gains appear around Pass@1 values of 0.05--0.35 (\figref{fig:retry_sweet_spot}), where repeated attempts can cover \mbox{many initially missed problems}.

First-attempt accuracy alone does not determine the gain from retries.
On Base 7 addition at $n=3$, R1-1.5B rises from Pass@1 of 0.08 to Pass@32 of 0.98, leaving $\hat E(32)=0.02$.
On Base 7 multiplication at $n\geq7$, the same model has Pass@1 of zero and $\hat E(32)=1$.
R1-14B at $n=6$ on Base 7 multiplication starts at Pass@1 of 0.14 but solves 84\% of the evaluated problems with the same total sampling budget of 32 complete-answer attempts (\figref{fig:app_pass1_heatmap}).\looseness=-1 

A low Pass@1 can therefore accompany either substantial retry gains or no observed success within the available budget.
The observations concern finite pools of complete answers: $\hat E(32)>0$ does not identify a permanent failure floor, and digit count is not the number of checked intermediate states.
\appref{app:base_comparison} and \appref{app:floor_analysis} give the curves, mixture fits, and token counts.

\subsection{Agent Score Prediction}
\label{sec:agent_prediction}

We also examine whether a repeated-attempt curve can predict later agent scores from early observations.
The retry curve approaches a fitted score ceiling with diminishing gains; a sigmoid in log time provides an alternative model of bounded score growth (\appref{app:agent_derivation}).

EdgeBench~\citep{edgebench2026} supplies complete observations at 2, 4, 6, 8, 10, and 12 hours for 252 task--model pairs across five models and 51 tasks.
For each target pair, we fit its 2-hour and 4-hour scores using category-level priors from other tasks. Priors exclude all model trajectories of the target task.

\begin{figure}[t]
    \centering
    \includegraphics[width=0.72\linewidth]{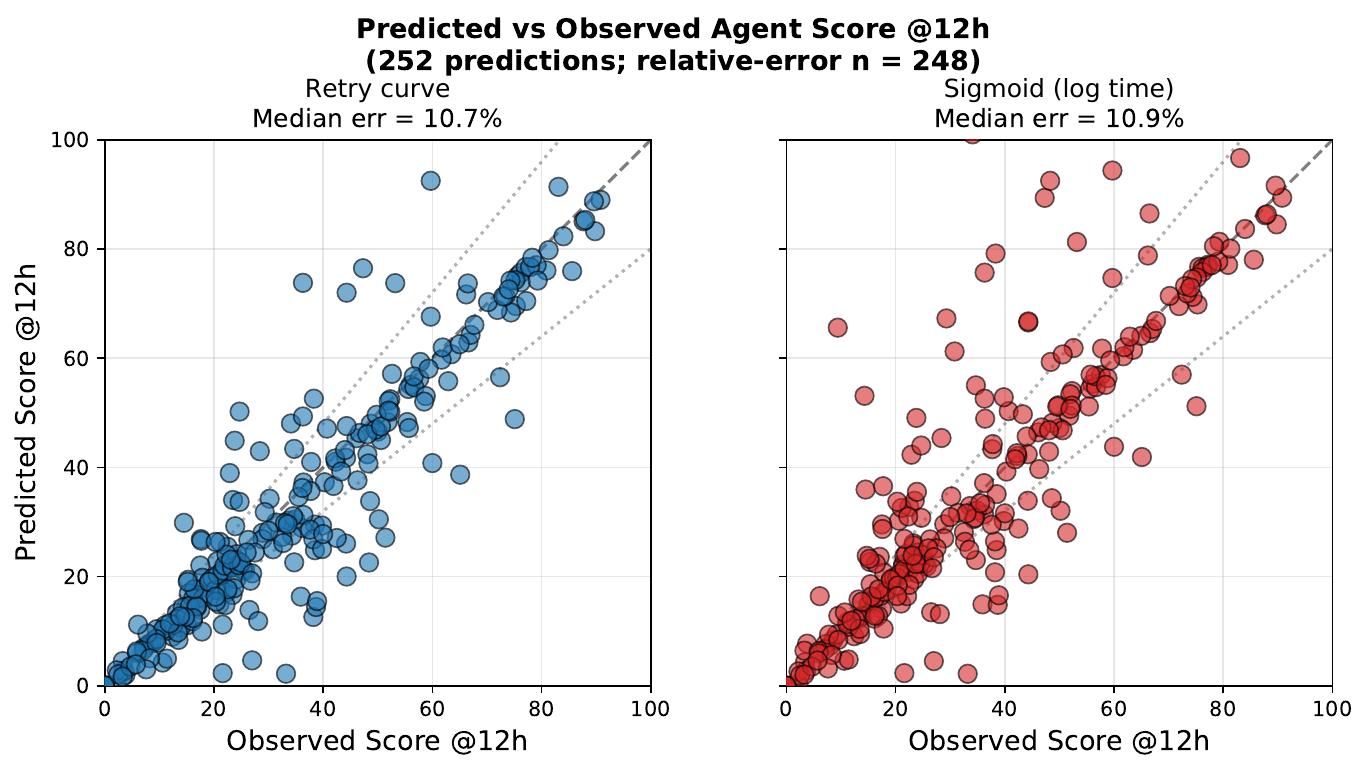}
    
    \caption{\textbf{Agent Score Prediction at 12 Hours.} Each point is a task--model pair. Predictions use its scores at 2 and 4 hours and parameter estimates from other tasks in the same category, excluding the target task. The retry curve models diminishing gains from repeated attempts; the comparison curve is sigmoid in log time (Appendix~\ref{app:agent_derivation}). All 252 pairs are plotted; median relative errors are 10.7\% and 10.9\%, respectively, over 248 pairs with positive observed scores. Dashed and dotted lines mark perfect prediction and $\pm20\%$ bands.}
    \label{fig:agent_scatter}
    
\end{figure}

Median relative errors for the retry and sigmoid curves are 3.8\% and 3.9\% at 6 hours, increasing to 10.7\% and 10.9\% at 12 hours.
Relative-error summaries use 248 pairs with positive observed scores at each prediction time, while all 252 pairs receive predictions.
\tabref{tab:agent_prediction} reports all four horizons and the fractions below 10\% and 20\% relative error.
Neither curve has the lower median error at every horizon; the retry curve has at least as many predictions below 10\% error and more below 20\% at each horizon.
The two curves have similar prediction errors, although their priors constrain different parameters.
The forecast does not identify internal checks or state reuse, or test savings from early stopping.

\FloatBarrier

\section{Related Work}
\label{sec:related_work}

\paragraph{Reasoning compositionality and long-horizon reliability.}
Prior work documents difficulties composing intermediate results~\citep{dziri2023faith}, falling success with task duration~\citep{kwa2025}, and failures at high problem complexity~\citep{shojaee2025illusion}.
\citet{zhang2025reasoning} formulate compute and accuracy laws, measure compositionality gaps, and introduce SFT-Compo as a training intervention.
We use their LoRe-Compo setting to study test-time sampling and state reuse.

\paragraph{Decomposition, search, and local error correction.}
Hierarchical reinforcement learning composes subgoals through options, MAXQ, and skill chaining~\citep{sutton1999between,dietterich2000hierarchical,konidaris2009skill}.
Least-to-Most prompting uses earlier answers to solve later subproblems~\citep{zhou2022least}; Tree of Thoughts searches and backtracks over partial states~\citep{yao2023tree}, and Graph of Thoughts combines and refines intermediate results~\citep{besta2024graph}.
Chain-in-Tree chains routine steps and reuses generated children when branching becomes necessary~\citep{li2026chain}.
Recursive and agentic inference systems similarly pass subcall outputs to later computations~\citep{zhang2025recursive,alzu2026roma,yu2025recode,nielsen2025learning}.
MAKER~\citep{meyerson2025solving} combines extreme decomposition, error detection, and first-to-$k$-ahead voting, with asymptotic foundations in~\citet{meyersonposition}.
Under its voting model, MAKER derives logarithmic voting budgets and $\Theta(N\log N)$ total expected cost for fixed task reliability, assuming local success probability above one half.
Our finite-cap analysis bounds work across correct and incorrect states under first acceptance and global stopping; its cost guarantee requires precision and acceptance bounds only at correct prefixes.

\paragraph{Stopped testing and restart analysis.}
Start-up demonstration tests study total-success and consecutive-failure stopping criteria, including dependent trials~\citep{gera2013startup,balakrishnan2014startup}.
Our result bounds expected work uniformly over the specified retry class and relates it to correctness at retained prefixes.
Classical restart scheduling optimizes repeated complete runs~\citep{luby1993optimal}; the present cap and stopping rule are fixed.
\appref{app:stopping_context} gives the comparison and exact work recursion.

\paragraph{Verification and test-time sampling.}
Independent sampling improves candidate coverage~\citep{brown2024large}, while adaptive test-time computation allocates effort according to task difficulty and verification quality~\citep{snell2025scaling,e3-2025,misaki2025wider}.
Self-consistency uses answer agreement as a proxy for correctness~\citep{wang2022self}; process supervision and reward models assess intermediate reasoning~\citep{uesato2022solving,lightman2023let,zheng2025survey}.
ProcessBench and Socratic-PRMBench document limits of identifying and evaluating reasoning errors~\citep{zheng2025processbench,li2025socratic}.
Formal verification and autoformalization provide stronger, domain-specific checks~\citep{leino2010dafny,szegedy2020promising}.
Our empirical Oracle is reference-answer matching, and SC is thresholded modal-answer selection; neither verifies a full reasoning trace.
\citet{labiad2026beyond} train a search policy to guide a frozen answer model and show that gains over repeated sampling depend on decoding and exploration choices.
Their comparisons match the answer-generation allocation; concept generation adds a small separate inference cost.
Our arithmetic experiments measure repeated-answer coverage under fixed decoding.

\paragraph{Retained context and execution cost.}
Atom of Thoughts decomposes and contracts problems into self-contained states designed to preserve answer equivalence~\citep{teng2025atom}.
MEM1 learns a compact recurrent state and discards redundant interaction history~\citep{zhou2025mem1}.
These methods address what to retain. Our analysis specifies the correctness and continuation conditions that the retained state must satisfy.
Using counterfactual error histories, \citet{sinha2025illusion} show that earlier mistakes can impair later reasoning, with thinking mitigating this effect in their experiments.
Context length and information placement can also affect model behavior~\citep{liu2024lost}.
Adaptive sequential Monte Carlo avoids recomputing shared prefixes by reordering KV caches and particle tensors during resampling~\citep{wang2026cache}.
This reduces execution cost without establishing the correctness of retained reasoning.
\appref{app:cascade_solution} gives a separate qualitative recurrence. The reliability proof instead controls errors at correct prefixes and bounds the probability of retaining an incorrect intermediate state.

\section{Discussion}
\label{sec:discussion}
\label{sec:what_achieved}

\paragraph{Acceptance and correctness.}
At a fixed retained state, more attempts increase the chance that a candidate passes the check.
With the same candidate distribution and check, they leave the error rate among accepted candidates unchanged.
Reliable continuation therefore requires both a sufficiently accurate check and a state that contains the information needed by later steps.
Checking an answer and retaining its full reasoning trace impose different correctness requirements on what is passed forward.

\paragraph{Executed work and total cost.}
The local cap reserves enough attempts for difficult steps, while acceptance often allows the procedure to advance earlier.
With constant acceptance probability and sufficient precision, the stopped loop uses linearly many attempts in expectation.
The cost of generating candidates, checking them, and processing retained states must then be counted at every reached state.
A stricter check can improve precision while increasing the cost of each attempt or reducing the chance of acceptance.
Both effects enter the comparison between checking cost and avoided generation.

\paragraph{Choosing a sampling configuration.}
The arithmetic results show that the benefit of more attempts varies across tasks, even when first-attempt accuracy is low.
On the paired mathematics problems, the answer-only prompt scores closer to full context than to no context, although the comparison changes instructions as well as retained information.
Larger candidate and continuation budgets also give uneven gains across models.
These observations help describe the configurations evaluated here.
Establishing computational savings from reuse still requires a comparison at the same whole-task success rate, with generation, checking, and retained-state processing all included.

\Needspace{13\baselineskip}
\section{Conclusion}
\label{sec:conclusion}

Reliable reuse depends on which candidates a check accepts and how later steps use the retained state.
Under the stated retry and stopping rules, we give conditions for correct completion and bounds on executed attempts.
Additional retries reduce exhaustion, while the candidate distribution and check determine accepted-state precision.
Total cost depends on per-attempt cost and acceptance probability.

The experiments document second-subproblem coverage across continuation prompts and variable gains from repeated arithmetic answers.
They describe practical sampling choices but do not establish the long-chain guarantees.
Assessing reuse on a new task requires measuring retained-state error, continuation reliability, and total cost at matched whole-task success.

\label{page:main_end}

\section*{Limitations}
\addcontentsline{toc}{section}{Limitations}
\label{sec:limitations}
\label{sec:what_not_achieved}
\paragraph{Execution and checking assumptions.}
The work theorem assumes conditionally independent, identically distributed retries at each retained state and termination when the local budget is exhausted.
Feedback that changes the next attempt and fallback after exhaustion require separate analysis.
The binary-check examples further assume that check outcomes are independent conditional on candidate correctness.
The experiments do not directly test these execution assumptions.

\paragraph{Checks and retained information.}
\label{sec:disc_verification}
Agreement can select common correct answers, but shared errors can also receive many votes~\citep{wang2022self}.
Similar SC and Oracle scores on the second subproblem do not tell us how often either selector keeps a wrong first answer.
The correctness condition applies to everything passed forward: checking an extracted answer does not verify a reasoning trace retained alongside it.
Process-verifier benchmarks document errors in learned checks; formal verification uses task-specific correctness criteria~\citep{zheng2025processbench,li2025socratic,leino2010dafny,szegedy2020promising}.\looseness=-1 

\paragraph{Decomposition and empirical scope.}
\label{sec:disc_states}
The theorem takes the decomposition and correctness criterion as inputs.
A smaller step can be easier to check but requires more checks; a larger step may retain information whose correctness is harder to establish.
The LoRe-Compo subproblems are mathematically separate, and arithmetic retries whole answers, so neither experiment estimates uniform reliability at the prefixes of a long dependent chain.
Open-ended tasks may need different representations and task-specific checks before the sufficient conditions become applicable.
The available composite aggregates do not identify accepted-state error or fallback contributions; prompt differences also prevent isolating an information-reuse effect.\looseness=-1 

\paragraph{Cost and auxiliary models.}
\label{sec:disc_capability}
The total-cost bound includes generation, checking, and state processing at every reached state.
Growing contexts or more accurate checks can increase that cost and must remain explicit in the budget expression.
When complete solutions are already reliable, the extra checking may outweigh the work saved by reuse.
The appendix's qualitative error-propagation and mixture models describe possible behaviors; finite retry pools cannot identify a permanent capability boundary.
Agent-score prediction also does not identify internal verification or establish a compute-allocation benefit.

%% file: resources/appendix.tex
\section{Modeling Assumptions and Their Scope}
\label{app:limitations}

The analysis uses explicit assumptions about sampling, checks, and task decomposition. This appendix states their scope and relates them to the experimental protocols.
The \hyperref[sec:limitations]{Limitations} section explains their implications for checking, continuation, and empirical interpretation.

\paragraph{Independence assumption (see~\hyperref[sec:disc_states]{Limitations}).}
\propref{prop:breakdown_product} requires matching subproblem success rates between composite and standalone prompts, within-run independence of subproblem correctness, and independent identically distributed composite runs.
The reuse result in \propref{prop:restoration} requires the same reliability bounds for every correct accepted state passed to later stages, with the total bound on accepted-state errors below the target failure probability.
Our experiments sample subproblem pairs from distinct MATH subjects as a practical way to reduce concept overlap between subproblems.
Subject separation does not establish these assumptions: shared formatting, latent correlations, and position-dependent generation can affect both dependence and marginal success rates.\looseness=-1 

\paragraph{Error propagation assumptions (see~\hyperref[sec:disc_capability]{Limitations}).}
The error model in~\appref{app:cascade_solution} assumes that earlier failures add linearly to later failure probabilities, remain in context, and have a common baseline probability $q_t^{\mathrm{clean}}=\bar q$ when preceding steps are correct.
Real autoregressive models exhibit nonlinear error interactions, heterogeneous step difficulty, and partial attention decay across long contexts~\citep{liu2024lost}, none of which the closed-form solution $\bar q(1{+}c_{\mathrm{ctx}})^{t-1}$ captures.
\appref{app:cascade_solution} gives the closed form; its parameters are not fitted to the experimental data.

\paragraph{Stepwise decomposition requirement (see~\hyperref[sec:disc_states]{Limitations}).}
\propref{prop:restoration} requires each step to have a positive probability of producing a correct accepted state from every correct input state. A suitable decomposition must keep the individual steps within the model's capabilities.
Open-ended generation, planning without intermediate checks, and tasks without a step-level correctness criterion may not admit such a decomposition, so the product lower bound need not apply.

\paragraph{Task length must reflect reasoning steps.}
The scaling variable $N$ counts steps in the chosen decomposition, rather than prompt length or operand digits.
Memorization or familiarity can let a model answer without following that decomposition, changing the relation between input size and executed work.
Comparing Base 7 and Base 10 varies the numeral representation; without observations of internal steps, it cannot determine the process used to obtain an answer.

\paragraph{Reliability of checks (see~\hyperref[sec:disc_verification]{Limitations}).}
On LoRe-Compo, self-consistency (SC) and Oracle selection yield similar second-subproblem success rates. This comparison does not estimate the accuracy of either check, and the arithmetic experiments do not compare these selection rules.
The error rate among accepted SC answers can be measured by comparing them with reference answers, excluding fallback candidates that did not pass the check.
Summaries that include fallback continuations describe all retained paths and cannot directly estimate the error rate among accepted states.

\newpage

\begin{table}[t]
    \caption{\textbf{Core notation used in the paper.}
    Proof, fit, and metric symbols are defined where they appear.}
    \label{tab:notation}
    \centering
    \resizebox{\textwidth}{!}{
    \begin{tabular}{ll}
    \toprule
    \textbf{Symbol} & \textbf{Meaning} \\
    \midrule
    \multicolumn{2}{l}{\textit{Base reasoning quantities}} \\
    $\mtheta$, $p_{\mtheta}$ & model parameters and induced generation distribution \\
    $x$, $x_1\oplus x_2$ & task and composite task \\
    $z_{1:N}$, $z_t$ & generated text sequence and text generated at one step \\
    $s_t$, $y$, $a^\star(x)$ & intermediate state, extracted final answer, and target answer \\
    $N$, $n$ & number of steps in the chosen decomposition and digit count \\
    $A_{\mtheta,1}(x)$, $C_{\mtheta,1}(x)$ & single-run answer accuracy and expected generated token count \\
    $\lambda_{\mtheta}$ & accuracy decay rate per step in the conditional scaling model \\
    \midrule
    \multicolumn{2}{l}{\textit{Independent sampling quantities}} \\
    $K$ & number of sampled runs \\
    $A_{\mtheta,K}(x)$, $C_{\mtheta,K}(x)$ & Pass@$K$ coverage and expected generated token count ($K{=}1$: one run) \\
    $\sigma$ & relative compositionality ratio \\
    \midrule
    \multicolumn{2}{l}{\textit{Error propagation and repeated attempts}} \\
    $\varepsilon$, $\bar q$ & per-step failure probability and common failure probability from correct preceding states \\
    $q_t$, $q_t^{\mathrm{clean}}$, $q_t^{\mathrm{clean}}(R)$ & step-failure probability; from a correct prefix; no correct candidate in $R$ attempts \\
    $q_t^{\mathrm{floor}}$, $c_{\mathrm{ctx}}$ & asymptotic error floor and coefficient for the effect of earlier errors \\
    \midrule
    \multicolumn{2}{l}{\textit{State reuse and retry quantities}} \\
    $z_t^{(r)}$, $s_t^{\checkmark}$ & candidate text and accepted intermediate state \\
    $\cG$, $\cG_t$; $\cV$, $\cV_t$ & correctness indicators; accept-or-reject checks \\
    $\rho_t(s)$, $\alpha_{t}(R;s)$ & rejection probability and accepted-state error, given input state $s$ \\
    $\beta_t(R;s)$ & probability of budget exhaustion or an incorrect returned state \\
    $\varepsilon_{\mathrm{loc},t}$ & single-attempt failure probability on solvable instances in the mixture model \\
    $R$ & local attempt cap; complete-answer attempt count in arithmetic \\
    $H_t$ & first $t$ steps complete with every retained state correct \\
    $W$, $T_t$ & executed candidate attempts in the run and at step $t$ \\
    $h_N$, $g_N$ & acceptance-probability lower bound and per-attempt cost upper bound \\
    $\mathcal P_{N,R}$ & processes obeying the local retry cap and global stopping rule \\
    $c_{\text{gen}}$, $c_{\text{ver}}$ & per-attempt generation and checking cost bounds in a common unit \\
    $A_{\mathrm{vsr}}$, $C_{\mathrm{vsr}}$, $\overline C_{\mathrm{vsr}}$, $C_{\mathrm{indep}}$ & reuse accuracy, expected reuse cost, reuse cost bound, and independent sampling cost \\
    \midrule
    \multicolumn{2}{l}{\textit{Experimental metrics}} \\
    $M$, $B$ & first-subproblem candidate budget and cap on candidates retained for continuation \\
    $\hat{E}(R)$ & empirical fraction of problems unsolved after $R$ attempts \\
    $\hat E(32)$ & unsolved fraction after 32 complete-answer attempts \\
    \bottomrule
    \end{tabular}
    }
\end{table}

\subsection{Error Propagation Model}
\label{app:cascade_solution}

Errors in autoregressive generation may remain in the input context and affect later steps~\citep{zhang2023language, liu2024lost}.
Let $q_t$ be the failure probability at step $t$ under the generated prefix, and $q_t^{\mathrm{clean}}$ its value when all preceding steps are correct.
For a coefficient $c_{\mathrm{ctx}}\geq0$ describing the effect of earlier errors, consider the qualitative approximation
\begin{equation}
q_t\;\approx\;q_t^{\mathrm{clean}}+c_{\mathrm{ctx}}\sum_{j<t}q_j.
\label{eq:volterra}
\end{equation}
If $q_t^{\mathrm{clean}}=\bar q$ at every step, differencing the recurrence gives
\begin{equation}
q_t\;\approx\;\bar q(1+c_{\mathrm{ctx}})^{t-1}.
\label{eq:cascade_solution}
\end{equation}
When $\bar q>0$ and $c_{\mathrm{ctx}}>0$, this expression increases with the step index. If the baseline difficulty varies across steps, $c_{\mathrm{ctx}}>0$ alone does not imply increasing failure probability.
The parameters $c_{\mathrm{ctx}}$ and $\bar q$ are not fitted here. The recurrence is a qualitative approximation with a linear, persistent context effect, valid only while predicted probabilities remain below one.
The recurrence offers an intuition for context-carried errors, while the product comparison and reuse guarantee use their own explicit probability assumptions.
Arithmetic retries do not observe step failures within a run and cannot directly test this recurrence.

\subsection{Correct Candidates before Checking}
\label{app:coverage_diagnostic}

For a fixed step instance and correct input state $s$, let $p_t(s)$ be the probability that a generated candidate is correct before checking.
With $R$ independent identically distributed attempts from this fixed context, the probability of no correct candidate is $(1-p_t(s))^R$, which tends to zero whenever $p_t(s)>0$.
Across a distribution of step instances and correct input states, define the probability of generating no correct candidate and its asymptotic error floor by
\begin{equation}
q_t^{\mathrm{clean}}(R):=\mathbb{E}_{s}\!\left[(1-p_t(s))^R\right],\qquad
q_t^{\mathrm{floor}}:=\lim_{R\to\infty}q_t^{\mathrm{clean}}(R)=P_s(p_t(s)=0).
\label{eq:app_coverage_floor}
\end{equation}
The limit follows by bounded convergence; a nonzero floor can arise from positive mass on locally unsolvable instances, rather than from independent retries at a fixed positive success probability.
A two-component approximation assigns mass $q_t^{\mathrm{floor}}$ to such instances and uses a common failure probability $\varepsilon_{\mathrm{loc},t}\in[0,1)$ for the remaining instances:
\begin{equation}
q_t^{\mathrm{clean}}(R)\;\approx\;q_t^{\mathrm{floor}}+(1-q_t^{\mathrm{floor}})\varepsilon_{\mathrm{loc},t}^{R}.
\label{eq:effective_failure_rate}
\end{equation}
This expression is exact for the stated two-point mixture; independence alone does not imply this mixture or a common failure probability on solvable instances.
Its zero-floor special case is
\begin{equation}
q_t^{\mathrm{floor}}=0\quad\Longrightarrow\quad q_t^{\mathrm{clean}}(R)\approx\varepsilon_{\mathrm{loc},t}^{R}.
\label{eq:app_floor_zero_reduction}
\end{equation}
These probabilities concern generated candidates before checking. They do not include the effect of accepting an incorrect candidate or rejecting a correct one.
The arithmetic experiments in~\secref{sec:two_component} use complete-answer attempts. The unsolved fraction at $R=32$ neither identifies an asymptotic error floor nor estimates the theorem's per-step failure probability.

\section{Proofs and Supporting Analysis}
\label{app:law_derivations}

This appendix gives the formal results and their proofs. It covers complete-solution sampling, retry budgets, stage composition, and stopped state reuse. Each result specifies its assumptions.

\subsection{Setup and Notation}
\label{app:notation}

We use the same model notation throughout the paper: $\mtheta$ denotes the parameters of the reasoning model, and $p_{\mtheta}$ denotes its induced generation distribution.
For a problem $x$, fix a decomposition into $N\geq1$ steps and write a sampled reasoning chain as $z_{1:N}=(z_1,\ldots,z_N)$, where $z_t$ is the text generated at step $t$.
Let $s_t=[x,z_{1:t}]$ denote the full generated history after step $t$, with $s_0=x$. In reuse, $s_t$ instead denotes the task information and intermediate results actually passed forward, as in Section~\ref{sec:prelim}; the prompt presents this state to the model.
Each step is sampled autoregressively from $p_{\mtheta}(\cdot\mid s_{t-1})$.
The final answer $y$ is extracted from the completed run.
The reuse result assumes that correct completion of every stage yields a correct final answer: if every accepted state is correct and all stages terminate, extraction returns the composite target (for two stages, $[a^\star(x_1),a^\star(x_2)]$; likewise for multiple stages below). Correctness refers to the retained state, not to reasoning discarded during answer extraction.
For local step analysis, $\cG_t(s_{t-1},z_t)\in\{0,1\}$ indicates whether the resulting state is correct, $q_t:=P(\cG_t=0)$ is its failure probability, and $q_t^{\mathrm{clean}}$ is the corresponding probability when all preceding states are correct.
The check $\cV_t(s_{t-1},z_t)\in\{0,1\}$ separately records rejection or acceptance and may disagree with $\cG_t$.
We write $N(x)$ for the length of the chosen decomposition of $x$, abbreviated as $N$ when the problem is clear.
The single-run quantities are $C_{\mtheta,1}(x)=\mathbb{E}[L]$ and $A_{\mtheta,1}(x)=P(y=a^\star(x))$, where $L$ is the total generated token count.
For two subproblems $x_1, x_2$, their composition $x_1 \oplus x_2$ is formed by concatenating them with a connector prompt. 
We define its composite target as $a^\star(x_1\oplus x_2):=[a^\star(x_1),a^\star(x_2)]$.
For these chosen decompositions, \emph{step-count additivity} means that the step counts satisfy
\begin{equation}
N(x_1 \oplus x_2) = N(x_1) + N(x_2).
\label{eq:app_independence_def}
\end{equation}
Operationally, our experiments use distinct MATH subjects as a practical proxy for reduced concept overlap across subproblems.
Step-count additivity does not imply the marginal-matching or probabilistic independence assumptions of~\propref{prop:breakdown_product}.
For reuse, a step must admit independent repeated attempts from a fixed input state and an accept-or-reject check $\cV$.
The ideal predicate $\cG$ defines correctness, whereas $\cV$ is the check applied to candidate states; acceptance alone does not imply $\cG=1$.
In the reuse proofs, $0<\rho<1$ uniformly bounds $\rho_t(s)$, and $\alpha$ uniformly bounds the error rate $\alpha_{t}(R;s)$ among accepted states over all relevant steps, correct accepted input states $s$, and attempt budgets $R$.
With $R$ total attempts, equivalently $R{-}1$ retries, all candidates are rejected with probability at most $\rho^R$. Returning an incorrect accepted state is a separate source of failure.
For independent sampling, the accuracy and generation cost of $K$ complete runs satisfy
\begin{align}
A_{\mtheta,K}(x) &= 1 - (1-A_{\mtheta,1}(x))^K \label{eq:passatk_def} \\
C_{\mtheta,K}(x) &= K \cdot C_{\mtheta,1}(x) \label{eq:passatk_cost}
\end{align}
Here $A_{\mtheta,K}$ is Pass@$K$ coverage and $C_{\mtheta,K}$ is expected generation cost; $K=1$ denotes a single complete run.
Across task lengths $N$, the conditional scaling assumptions are
\begin{align}
C_{\mtheta,1}(x) &= \Theta(N) \quad \text{(Compute Law)} \label{eq:app_ideal_compute} \\
A_{\mtheta,1}(x) &= \exp(-\lambda_{\mtheta}N),\quad \lambda_{\mtheta}>0 \quad \text{(Accuracy Law)} \label{eq:app_ideal_accuracy}
\end{align}

\paragraph{Notation conventions.}
All probabilities denoted by $P(\cdot)$ are taken under the relevant sampling process induced by $p_{\mtheta}$; $p_{\mtheta}$ itself denotes the model distribution.
For an integer $K$, $[K]$ denotes $\{1,\ldots,K\}$.
The symbol $\lll$ denotes an informal ``much smaller than'' condition in limiting approximations.
Failure budgets satisfy $\delta,\delta_i,\delta_m\in(0,1)$, with $\delta=\delta_1+\delta_2$ in the two-stage result and $\sum_m\delta_m=\delta$ in the multi-stage result, where the sum ranges over all stages of the task.

\subsection{Ideal Laws Under Per-Step Errors}
\label{app:ideal_failure}

We begin with the success probability and cost implied by independent per-step errors.
This calculation concerns a step decomposition; the arithmetic experiments measure complete-answer coverage.

\begin{proposition}[Exponential Decay of Success Probability]
\label{prop:exponential_decay}
Consider a chosen decomposition into $N$ steps, each with independent failure probability $\varepsilon\in(0,1)$, and assume the final answer is correct if and only if all $N$ steps are correct. Then:
\begin{equation}
A_{\mtheta,1}(x) = (1-\varepsilon)^N \leq e^{-\varepsilon N},
\label{eq:app_exponential_decay_form}
\end{equation}
where the equality $(1-\varepsilon)^N$ is exact and the upper bound $e^{-\varepsilon N}$ follows from $\log(1-\varepsilon) \leq -\varepsilon$.
This gives Eq.~\eqref{eq:single_chain_success} with the per-step accuracy decay rate $\lambda_{\mtheta}=-\log(1-\varepsilon)$.
\end{proposition}

\begin{proof}
\textbf{Step 1: Success requires all steps.}
Let $F_i$ indicate failure at step $i$; each step succeeds with probability $1-\varepsilon$, so independence yields the following product:
\begin{equation}
A_{\mtheta,1}(x) = \prod_{i=1}^N P(F_i=0) = (1-\varepsilon)^N.
\end{equation}

\textbf{Step 2: Logarithmic inequality.}
Since $\log$ is concave and $\log(1) = 0$, the tangent line at $x=1$ gives $\log(x) \leq x-1$ for all $x > 0$.
Substituting $x = 1-\varepsilon$: $\log(1-\varepsilon) \leq (1-\varepsilon)-1 = -\varepsilon$ for all $\varepsilon \in (0,1)$.
Taking the natural logarithm of Step~1:
\begin{equation}
\log A_{\mtheta,1}(x) = N\log(1-\varepsilon) \leq -\varepsilon N.
\end{equation}

\textbf{Step 3: Exponential bound and approximation.}
Exponentiating both sides:
\begin{equation}
A_{\mtheta,1}(x) \leq e^{-\varepsilon N}.
\end{equation}
For fixed $N$ as $\varepsilon\to0$, the Taylor expansion $\log(1-\varepsilon)=-\varepsilon+O(\varepsilon^2)$ yields $A_{\mtheta,1}(x)/e^{-\varepsilon N}\to1$.
\end{proof}

\subsection{The Compositionality Breakdown under Independent Sampling}
\label{app:coupling_proof}

We first give a general event-inclusion bound for a common pool of composite runs, then prove the conditional product comparison used in the main text.

\begin{proposition}[Event-Inclusion Bound]
\label{prop:breakdown}
For any $K\geq1$, sample a common pool of $K$ composite runs for $x_1\oplus x_2$. Let $E_i^K$ denote coverage of a correct answer to $x_i$, and $E_{12}^K$ coverage of both correct answers in one run. These events satisfy
\begin{equation}
P(E_{12}^K)\;\leq\;P(E_1^K\cap E_2^K)\;\leq\;\min\{P(E_1^K),P(E_2^K)\},
\label{eq:app_breakdown_inclusion}
\end{equation}
without independence. All probabilities refer to the common pool of composite runs.
\end{proposition}

\begin{proposition}[Product Form]
\label{prop:breakdown_product}
Let $a:=A_{\mtheta,1}(x_1)\in(0,1)$ and $b:=A_{\mtheta,1}(x_2)\in(0,1)$. Assume each subproblem has the same correctness probability in a composite run as in its standalone run, the two correctness indicators are independent within each composite run, and the composite runs are i.i.d.\ across $K$ samples. Then $A_{\mtheta,1}(x_1\oplus x_2)=ab$ and:
\begin{equation}
A_{\mtheta,K}(x_1 \oplus x_2) \;<\; A_{\mtheta,K}(x_1)\cdot A_{\mtheta,K}(x_2), \qquad \forall\, K \geq 2.
\label{eq:app_breakdown_product}
\end{equation}
\end{proposition}

\begin{remark}[Experimental proxy]
\label{rem:independence_proxy}
The main-text product comparison relies on the assumptions of~\propref{prop:breakdown_product}; the same-pool event-inclusion bound of~\propref{prop:breakdown} holds without independence.
Our experiments sample subproblem pairs from distinct MATH subjects to encourage reduced concept overlap across subproblems.
The \textit{no\_context} ablation in~\tabref{table:prompt_ablation} gives a descriptive comparison: second-subproblem success without $x_1$ information is 36.2\%. This differs from joint success on both problems, and the available standalone $x_2$ reference covers fewer models than the second-subproblem evaluation.
As discussed in \appref{app:limitations}, subject separation establishes neither independence within a run nor matching success rates across prompt formats.
\end{remark}

\begin{proof}[Proof of~\propref{prop:breakdown}]
For a pool of $K$ composite runs with sample space $\Omega$, define
\begin{equation}
\begin{aligned}
E_{12}^K &:= \{\exists k: \text{run } k \text{ returns a correct answer to both } x_1 \text{ and } x_2\}, \\
E_1^K &:= \{\exists k: \text{run } k \text{ returns a correct answer to } x_1\},
\end{aligned}
\end{equation}
and analogously $E_2^K$ for $x_2$.
The inclusion $E_{12}^K \subseteq E_1^K \cap E_2^K$ holds: if some run $k$ is correct on both subproblems, then that same run witnesses $E_1^K$ and witnesses $E_2^K$. Monotonicity of probability gives $P(E_{12}^K) \leq P(E_1^K \cap E_2^K) \leq \min\{P(E_1^K), P(E_2^K)\}$, which is Eq.~\eqref{eq:app_breakdown_inclusion}. This bound holds for every $K$ without assuming independence of subproblem-answer correctness.
\end{proof}

\begin{proof}[Proof of~\propref{prop:breakdown_product}]
Assume the matching composite/standalone marginals, within-run independence, and across-run i.i.d.\ conditions, and let $a:=A_{\mtheta,1}(x_1)$ and $b:=A_{\mtheta,1}(x_2)$, so that $A_{\mtheta,1}(x_1\oplus x_2)=ab$.
For $K$ independently sampled runs, define indicator random variables: $X_k \in \{0,1\}$ records that sampled run $k$ succeeds on $x_1$ with $P(X_k=1)=a$; $Y_k \in \{0,1\}$ records that sampled run $k$ succeeds on $x_2$ with $P(Y_k=1)=b$; and $Z_k := X_k Y_k$ records that sampled run $k$ succeeds on both with $P(Z_k=1)=ab$.
The two sequences are independent, with i.i.d.\ entries within each sequence.

\medskip\noindent\textbf{Step 1: Event formulation.}
Define success events:
\begin{equation}
E_1^K := \{\exists k: X_k=1\}, \quad E_2^K := \{\exists k: Y_k=1\}, \quad E_{12}^K:=\{\exists k: Z_k=1\}.
\label{eq:app_success_events}
\end{equation}
Independence across runs gives the following probabilities for the three success events:
\begin{equation}
P(E_1^K)=1-(1-a)^K, \quad P(E_2^K)=1-(1-b)^K, \quad P(E_{12}^K)=1-(1-ab)^K.
\label{eq:app_success_probs}
\end{equation}

\medskip\noindent\textbf{Step 2: Set inclusion.}
Key observation: $E_{12}^K \subseteq E_1^K \cap E_2^K$.
To see this, suppose $\zeta \in E_{12}^K$, so there exists $k$ with $Z_k(\zeta)=X_k(\zeta)Y_k(\zeta)=1$.
Then $X_k(\zeta)=1$, hence $\zeta \in E_1^K$; and $Y_k(\zeta)=1$, hence $\zeta \in E_2^K$.
Therefore $\zeta \in E_1^K \cap E_2^K$, completing the inclusion. Monotonicity of probability gives:
\begin{equation}
P(E_{12}^K) \leq P(E_1^K \cap E_2^K).
\label{eq:app_set_inclusion}
\end{equation}

\medskip\noindent\textbf{Step 3: Independence gives exact product.}
Since $\{X_k\}_{k=1}^K$ and $\{Y_k\}_{k=1}^K$ are mutually independent by assumption, any event measurable w.r.t.\ $\{X_k\}$ alone is independent of any event measurable w.r.t.\ $\{Y_k\}$ alone.
$E_1^K = \{\exists k: X_k=1\}$ depends only on $\{X_k\}$, and $E_2^K = \{\exists k: Y_k=1\}$ depends only on $\{Y_k\}$.
Therefore $E_1^K$ and $E_2^K$ are independent:
\begin{equation}
P(E_1^K \cap E_2^K) = P(E_1^K)P(E_2^K) = [1-(1-a)^K][1-(1-b)^K].
\label{eq:app_independent_product}
\end{equation}
Combining this product identity with the event inclusion from Step 2 gives the comparison
\begin{equation}
1-(1-ab)^K \leq [1-(1-a)^K][1-(1-b)^K].
\label{eq:app_breakdown_inequality}
\end{equation}
This establishes the non-strict product comparison in \propref{prop:breakdown_product}; strictness is proved next.

\medskip\noindent\textbf{Step 4: Strictness.}
For $K \geq 2$ and $a,b \in (0,1)$, the inequality is strict. Consider the event:
\begin{equation}
E_{1,2} := \{X_1=1,\; Y_1=0,\; X_2=0,\; Y_2=1,\; \forall\, k \geq 3:\; X_kY_k=0\}.
\label{eq:app_strict_event}
\end{equation}
The following three properties show that this event has positive mass outside joint success.

\emph{(i) $E_{1,2} \subseteq E_1^K \cap E_2^K$.}
$X_1=1$ implies $E_1^K$; $Y_2=1$ implies $E_2^K$.

\emph{(ii) $E_{1,2} \subseteq (E_{12}^K)^c$.}
Under $E_{1,2}$: sampled run 1 has $Z_1=X_1Y_1=1\cdot 0=0$; sampled run 2 has $Z_2=X_2Y_2=0\cdot 1=0$; and $Z_k=0$ for all $k\geq 3$ by construction.
Hence no sampled run succeeds on both subproblems, so $E_{12}^K$ does not occur under this construction.

\emph{(iii) $P(E_{1,2}) > 0$.}
By mutual independence of all $\{X_k\}$ and $\{Y_k\}$:
\begin{equation}
P(E_{1,2}) = a(1-b)(1-a)b\prod_{k=3}^{K}(1-ab) = ab(1-a)(1-b)(1-ab)^{K-2} > 0,
\label{eq:app_strict_prob}
\end{equation}
The strict positivity follows because every factor is positive for $a,b\in(0,1)$.

Properties (i) and (ii) give $E_{1,2} \subseteq (E_1^K\cap E_2^K) \setminus E_{12}^K$, so $P((E_1^K\cap E_2^K)\setminus E_{12}^K) \geq P(E_{1,2}) > 0$, establishing the required strict comparison $P(E_{12}^K)<P(E_1^K\cap E_2^K)$ for every $K\geq2$.
\end{proof}

\paragraph{Low accuracy behavior in the difficult problem range.}
Under the assumptions of~\propref{prop:breakdown_product}, consider the ratio between composite accuracy and the product of standalone accuracies when $Ka,Kb\lll1$. The approximation $1-(1-x)^K\approx Kx$ for $Kx\lll1$ follows from $(1-x)^K=1-Kx+O(K^2x^2)$.
Applying this to each term:
\begin{align*}
A_{\mtheta,K}(x_1) &= 1-(1-a)^K \approx Ka, \quad A_{\mtheta,K}(x_2) = 1-(1-b)^K \approx Kb, \\
A_{\mtheta,K}(x_1\oplus x_2) &= 1-(1-ab)^K \approx K(ab),
\end{align*}
where the last line uses $K(ab) \leq Ka \cdot b \leq Ka \lll 1$, ensuring the approximation is valid.
Define $\sigma := A_{\mtheta,K}(x_1 \oplus x_2) / [A_{\mtheta,K}(x_1) A_{\mtheta,K}(x_2)]$. Then $\sigma=1$ means accuracy compositionality, and $\sigma<1$ means composite accuracy falls below the product of standalone accuracies.
The resulting first-order ratio quantifies the relative gap before either accuracy saturates:
\begin{equation}
\sigma = \frac{A_{\mtheta,K}(x_1 \oplus x_2)}{A_{\mtheta,K}(x_1) \cdot A_{\mtheta,K}(x_2)} \approx \frac{K(ab)}{(Ka)(Kb)} = \frac{Kab}{K^2ab} = \frac{1}{K}.
\label{eq:app_synergy_coeff_app}
\end{equation}
This $1/K$ degradation scales inversely with the number of samples in the low-accuracy regime, showing that the relative compositionality breakdown of independent sampling worsens before saturation.
This first-order approximation requires $Ka,Kb\lll1$. Under the assumptions of~\propref{prop:breakdown_product}, the product inequality~\eqref{eq:app_breakdown_product} holds for all $K\geq2$ and $a,b\in(0,1)$; the common-pool inclusion in Eq.~\eqref{eq:app_breakdown_inclusion} remains valid for dependent subproblem-answer correctness.\looseness=-1 

\subsection{State Reuse: Mathematical Foundations}
\label{app:connected_foundations}

We now derive a product lower bound on success and a sufficient computation budget for state reuse.

\begin{definition}[Retry Loop]
\label{def:retry_loop}
Start from $s_0^{\checkmark}=x$.
At step $t$, let $s=s_{t-1}^{\checkmark}$ be the fixed accepted input state and $\rho_t(s):=P(\cV_t=0\mid s)$ its single-attempt rejection probability.
The loop uses an integer cap $R\geq1$ and repeats the following generation and checking steps:
\begin{enumerate}[nosep,leftmargin=12pt]
    \item Sample a candidate span $z_t^{(r)}\sim p_{\mtheta}(\cdot\mid s)$ from the fixed input state.
    \item Form the candidate intermediate state $\tilde s_t^{(r)}$, using either the full state $[s,z_t^{(r)}]$ or the prescribed extracted intermediate answer.
    \item Check $\cV_t(\tilde s_t^{(r)})$; if accepted, return $s_t^{\checkmark}=\tilde s_t^{(r)}$ and proceed.
    \item If rejected, retry from the same $s$ until $R$ total attempts have been used; if all are rejected, stop with failure.
\end{enumerate}
Let $A_t^{(R)}$ be the event that the loop returns an accepted state, and let $\cG_t$ indicate whether that state is correct. Define \emph{accepted-state error}, the error probability conditional on acceptance, as
\begin{equation}
\alpha_{t}(R;s):=P\!\left(\cG_t=0\mid A_t^{(R)},s\right).
\label{eq:app_contamination_def}
\end{equation}
Assume generation and check outcomes are independent and identically distributed across attempts conditional on $s$. Failure comprises budget exhaustion or an incorrect acceptance:
\begin{equation}
\begin{aligned}
\beta_t(R;s)&:=\rho_t(s)^R+\bigl(1-\rho_t(s)^R\bigr)\alpha_{t}(R;s)\\
&\leq\rho_t(s)^R+\alpha_{t}(R;s).
\end{aligned}
\label{eq:app_atomic_failure}
\end{equation}
The two terms are budget exhaustion and an incorrect state returned as accepted; the formal loop does not return a fallback after all candidates are rejected.
\end{definition}

\begin{remark}[Relation to Correct Candidates before Checking]
\label{rem:zero_floor}
The proofs use the failure bound in Eq.~\eqref{eq:app_atomic_failure}; they do not assume that an observed small unsolved fraction is an exactly zero asymptotic error floor.
Their uniform bounds $\rho<1$ and $\alpha<\delta/N<1$ imply a positive probability of a correct accepted candidate at every relevant correct prefix: at one attempt this probability is at least $(1-\rho)(1-\alpha)>0$.
Thus the probability of generating no correct candidate from a fixed correct input state tends to zero under independent attempts. The assumptions already require each step to be solvable.
The distribution considered in~\appref{app:coverage_diagnostic} may also include instances that do not satisfy these assumptions.
\end{remark}\looseness=-1 

\begin{proposition}[Logarithmic Attempt Requirement]
\label{prop:logarithmic_retry}
Under the assumptions of~\defref{def:retry_loop}, suppose $\rho_t(s)\leq\rho$ for a fixed $0<\rho<1$ and $\alpha_{t}(r;s)\leq\alpha$ for every relevant step, correct accepted input state $s$, and budget $r$. To keep the probability of any step exhausting its budget or accepting an incorrect state at most $\delta$, it suffices that $N\alpha<\delta$ and
\begin{equation}
R \;\geq\;
\left\lceil
\frac{\log\!\bigl(N/(\delta-N\alpha)\bigr)}{\log(1/\rho)}
\right\rceil.
\label{eq:retry_count}
\end{equation}
For fixed $0<\rho<1$ and $\alpha\leq c\delta/N$ with fixed $0\leq c<1$, the sufficient budget above is $\Theta(\log(N/\delta))$; in the separate case $\rho=0$, one attempt suffices for the rejection term.
\end{proposition}

\begin{proof}
\textbf{Step 1: Union bound.}
Let $H_{i-1}$ be the event that the first $i{-}1$ steps complete with correct accepted states, and let $E_i:=H_{i-1}\cap\{\text{step $i$ fails after its $R$-attempt loop}\}$ be the first-failure event at step $i$. These disjoint first-failure events exhaust trajectory failure, so the union bound gives
\begin{equation}
P(\text{global fail}) = P\!\left(\bigcup_{i=1}^{N} E_i\right) \;\le\; \sum_{i=1}^{N} P(E_i).
\end{equation}
Conditioned on $H_{i-1}$, step $i$ starts from a correct accepted prefix. By Eq.~\eqref{eq:app_atomic_failure}, $P(E_i)\leq P(\text{step $i$ fails}\mid H_{i-1})\leq\rho^R+\alpha$, giving:
\begin{equation}
P(\text{global fail}) \;\leq\; N\bigl(\rho^R+\alpha\bigr).
\label{eq:app_global_failure_contamination}
\end{equation}

\textbf{Step 2: Per-Step Target.}
For $N\alpha<\delta$, the per-step requirement is
\begin{equation}
N(\rho^R+\alpha)\leq\delta
\quad\Longleftrightarrow\quad
\rho^R\leq\frac{\delta-N\alpha}{N}.
\end{equation}

\textbf{Step 3: Solve for $R$.}
Taking logarithms (note $\log\rho<0$, so the inequality flips):
\begin{align}
R\log\rho &\leq \log\!\left(\frac{\delta-N\alpha}{N}\right) \\[4pt]
R &\geq \frac{\log\!\bigl(N/(\delta-N\alpha)\bigr)}{\log(1/\rho)}.
\end{align}
Because the attempt budget is an integer, taking the ceiling gives a sufficient choice:
\begin{equation}
R \;\geq\;
\left\lceil
\frac{\log\!\bigl(N/(\delta-N\alpha)\bigr)}{\log(1/\rho)}
\right\rceil.
\end{equation}
If, for example, $N\alpha\leq c\delta$ for a fixed $c<1$, then $R=\Theta(\log(N/\delta))$. More attempts alone cannot establish this guarantee when the bound on accepted-state errors already reaches $\delta$.
\end{proof}

\subsection{Accepted-State Error under a Fixed Check}
\label{app:fixed_check}

Fix a correct input state $s$ and a step $t$.
Suppose the same candidate-generation law and the same single-candidate checking rule are used independently on every attempt, irrespective of the attempt cap $R$.
Let $\rho=P(\cV_t=0\mid s)<1$ and $w=P(\cG_t=0,\cV_t=1\mid s)$.
Returning an incorrect state on attempt $r$ requires rejection on the first $r-1$ attempts and an incorrect acceptance on attempt $r$.
Summing over the possible first-acceptance attempts gives the joint and acceptance probabilities
\begin{align}
P(\text{incorrect returned state},A_t^{(R)}\mid s)
&=\sum_{r=1}^{R}\rho^{r-1}w
=\frac{w(1-\rho^R)}{1-\rho},\\
P(A_t^{(R)}\mid s)&=1-\rho^R.
\end{align}
Dividing by the acceptance probability yields the conditional error and the total local failure:
\begin{equation}
\alpha_t(R;s)=\frac{w}{1-\rho},\qquad
\beta_t(R;s)=\rho^R+(1-\rho^R)\frac{w}{1-\rho}.
\end{equation}
The conditional error among accepted states is independent of $R$ in this fixed-check setting, and the local failure probability approaches that error as $R$ grows.
This calculation includes false rejections through $\rho$; $w/(1-\rho)$ is the error conditional on acceptance, not the probability of acceptance conditional on an incorrect candidate.
Changing the voting pool or acceptance threshold changes the checking rule and can therefore change the error among accepted states.\looseness=-1 

A batch selector could be treated as one local attempt only after its returned-state law, rejection event, and computation cost are specified.
Independent batches would then need the corresponding conditional bounds at every correct prefix.
The empirical protocol's multiple retained states and fallback continuations are not obtained merely by substituting $M$ for $R$ in the single-state theorem.

\subsection{Product Lower Bound Under State Reuse}
\label{app:compositionality_restoration}

The next proposition composes stage reliability guarantees and adds their computation budgets.

\begin{proposition}[Product Lower Bound for State Reuse]
\label{prop:restoration}
For subproblems $x_1,x_2$ with $N_1,N_2$ steps, assume the loop of~\defref{def:retry_loop}. At stage $i$, let $\rho_{i,t}(s)\leq\rho_i$ with $0<\rho_i<1$ and $\alpha_{i,t}(r;s)\leq\alpha_i$ for every step, correct accepted input state $s$, and budget $r$. For stage 2, these bounds must hold for every correct state passed from stage 1.
Let $\delta_1, \delta_2 > 0$ be \emph{target failure probabilities} with $\delta_1 + \delta_2 = \delta$.
Assume each stage reaches a well-defined final state after its prescribed $N_i$ accepted steps. Let $S_i$ denote completion of stage $i$ with every accepted state correct. Completing both stages correctly must yield the target under answer extraction:
\begin{equation*}
S_1\cap S_2\subseteq\{y=a^\star(x_1\oplus x_2)\}.
\end{equation*}
When $N_i\alpha_i<\delta_i$, allocate the following sufficient attempt budget to each step in stage $i$:
\begin{equation}
R_i\;\geq\;
\left\lceil
\frac{\log\!\bigl(N_i/(\delta_i-N_i\alpha_{i})\bigr)}{\log(1/\rho_i)}
\right\rceil,
\label{eq:app_stage_retry_count}
\end{equation}
By \propref{prop:logarithmic_retry}, stage $i$ then fails to complete with all accepted states correct with probability at most $\delta_i$.
Under these assumptions, state reuse satisfies:
\begin{equation}
A_{\mathrm{vsr}}(x_1 \oplus x_2)
\;\geq\;(1-\delta_1)(1-\delta_2)
\;=\;1-\delta+\delta_1\delta_2
\;\geq\;1-\delta.
\label{eq:app_product_bound_statement}
\end{equation}
This exact lower bound uses target stage guarantees rather than standalone accuracies.
\end{proposition}

\paragraph{Computation budget.}
Set $\delta_i=\delta N_i/(N_1+N_2)$ and use the common error bound $\alpha_i=\alpha$. Both subproblems then have the same per-step target failure probability:
\begin{equation}
\frac{\delta_1}{N_1} = \frac{\delta_2}{N_2} = \frac{\delta}{N_1+N_2}.
\label{eq:app_uniform_perstep}
\end{equation}
If $N\alpha<\delta$ for $N=N_1+N_2$ and $\rho_i\leq\rho<1$, choose the common attempt count
\begin{equation}
R^*=\left\lceil\frac{\log\!\bigl(N/(\delta-N\alpha)\bigr)}{\log(1/\rho)}\right\rceil
\end{equation}
This choice satisfies both stage constraints. Let $c_{\text{gen}}$ and $c_{\text{ver}}$ uniformly bound generation and checking costs per attempt in a common unit, for all reached states. The per-step computation budget is $\bar c_{\text{step}}(R^*)=R^*(c_{\text{gen}}+c_{\text{ver}})$. It is an upper bound on actual cost and sums across stages:
\begin{align}
\overline C_{\mathrm{vsr}}(x_1) + \overline C_{\mathrm{vsr}}(x_2) &= N_1\,\bar c_{\text{step}}(R^*) + N_2\,\bar c_{\text{step}}(R^*) \nonumber\\
&= (N_1+N_2)\,\bar c_{\text{step}}(R^*) \;=\; \overline C_{\mathrm{vsr}}(x_1\oplus x_2).
\label{eq:app_cost_additivity}
\end{align}

\begin{proof}[Proof of~\propref{prop:restoration}]

\medskip\noindent\textbf{Step 1: Reliability of $x_1$.}
For $R_1$ from Eq.~\eqref{eq:app_stage_retry_count}, every correct accepted prefix $s$ has local failure bounded by the allocated per-step target for the first stage:
\begin{equation}
\beta_{1,t}(R_1;s)\;\leq\;\rho_1^{R_1}+\alpha_{1}
\;\leq\;\frac{\delta_1}{N_1}.
\label{eq:app_x1_perstep}
\end{equation}
A union bound over the $N_1$ possible first-failure locations gives the first-stage guarantee:
\begin{equation}
P(S_1^c)
\;\leq\;N_1(\rho_1^{R_1}+\alpha_{1})
\;\leq\;\delta_1,
\qquad\text{hence}\quad P(S_1)\;\geq\;1-\delta_1.
\label{eq:app_x1_success}
\end{equation}

\medskip\noindent\textbf{Step 2: Reliability of $x_2$ from correct accepted $x_1$ states.}
With $R_2$ satisfying Eq.~\eqref{eq:app_stage_retry_count}, each attempt is drawn from the model given its fixed accepted input state. The bounds $\rho_2$ and $\alpha_2$ must hold for every correct accepted state passed from $x_1$; acceptance alone does not establish this requirement. Conditioning on an arbitrary such state $s_{N_1}^{\checkmark}$ yields the same second-stage bound:
\begin{equation}
P\!\left(S_2^c\;\middle|\;s_{N_1}^{\checkmark},\,S_1\right)
\;\leq\;N_2(\rho_2^{R_2}+\alpha_{2})
\;\leq\;\delta_2.
\label{eq:app_x2_conditional}
\end{equation}
Marginalizing over correct accepted states (i.e., conditioning on $S_1$):
\begin{equation}
P(S_2^c \mid S_1) \;=\;
\mathbb{E}\!\left[P\!\left(S_2^c\mid s_{N_1}^{\checkmark},S_1\right)\;\middle|\;S_1\right]
\;\le\; \delta_2,
\label{eq:app_x2_success}
\end{equation}
Thus the conditional second-stage reliability satisfies the allocated target, $P(S_2\mid S_1)\geq1-\delta_2$.

\medskip\noindent\textbf{Step 3: Composite accuracy via chain rule.}
By assumption, $S_1\cap S_2$ implies that the prescribed extraction returns a correct composite answer. The probability chain rule gives:
\begin{equation}
A_{\mathrm{vsr}}(x_1\oplus x_2) \;\geq\; P(S_1\cap S_2)
= P(S_1)\cdot P(S_2\mid S_1)
\;\ge\; (1-\delta_1)(1-\delta_2).
\label{eq:app_composite_accuracy}
\end{equation}

\medskip\noindent\textbf{Step 4: Target level product bound.}
Expanding gives $(1-\delta_1)(1-\delta_2)=1-\delta+\delta_1\delta_2\geq1-\delta$.
No small-$\delta_i$ approximation is needed for this lower bound because the cross term is nonnegative.
Equation~\eqref{eq:app_composite_accuracy} is an exact lower bound; equality with actual standalone accuracies is not implied.
\end{proof}

\paragraph{Comparison with independent sampling.}
The two propositions compare different quantities. Under \propref{prop:breakdown_product}, independent sampling yields composite accuracy below the product of standalone accuracies, with a ratio of approximately $1/K$ when $Ka,Kb\lll1$. Under \propref{prop:restoration}, reuse from accepted states yields a product lower bound in stage reliability guarantees; it concerns target guarantees rather than an equality of actual standalone accuracies.\looseness=-1 

\subsection{Multi-Stage Generalization}
\label{app:multistage_complete}

We extend \propref{prop:restoration} by induction to arbitrary multi-stage workflows with $J$ subproblems $x = x_1 \oplus \cdots \oplus x_J$, each requiring $N_m$ steps, for a total of $N = \sum_{m=1}^J N_m$.

\paragraph{Step 1: Proportional budget allocation.}
Allocate failure budgets proportionally to step counts:
\begin{equation}
\delta_m \;:=\; \delta\,\frac{N_m}{N}, \qquad \sum_{m=1}^J \delta_m = \delta.
\label{eq:app_unified_budget_allocation}
\end{equation}
This yields a uniform per-step target across all subproblems:
\begin{equation}
\frac{\delta_m}{N_m} = \frac{\delta}{N} \quad \text{for all } m \in [J].
\label{eq:app_perstep_uniformity}
\end{equation}

\paragraph{Step 2: Uniform attempt count.}
Assume $\rho_t(s)\leq\rho$ with $0<\rho<1$ and $\alpha_t(r;s)\leq\alpha$ for all stages, relevant steps, correct accepted input states $s$, and attempt budgets $r$, with $N\alpha<\delta$. \propref{prop:logarithmic_retry} gives a common sufficient retry budget by imposing the constraint $\rho^R+\alpha\leq\delta/N$ at each step:
\begin{equation}
R\;\geq\;
\left\lceil
\frac{\log\!\bigl(N/(\delta-N\alpha)\bigr)}{\log(1/\rho)}
\right\rceil
\;=\;\Theta\!\left(\log\tfrac{N}{\delta}\right)
\label{eq:app_unified_retry}
\end{equation}
when $N\alpha\leq c\delta$ for a fixed $0\leq c<1$.
Since all stages receive the same error allowance per step, this common attempt budget meets the reliability requirement for every stage length $N_m$.

\paragraph{Step 3: Summing the computation budget.}
Let $R^*$ be the smallest integer budget in Eq.~\eqref{eq:app_unified_retry}. Each attempt costs at most $c_{\text{gen}}+c_{\text{ver}}$, so the per-step computation budget is
\begin{equation}
\bar c_{\text{step}}(R^*) \;=\; R^*\,(c_{\text{gen}}+c_{\text{ver}}).
\label{eq:app_unified_step_cost}
\end{equation}
Since $R^*$ is the same for all subproblems, the computation budget sums across stages. Actual cost is no larger because each loop stops at first acceptance:
\begin{align}
\overline C_{\mathrm{vsr}}(x) &= \sum_{m=1}^J \overline C_{\mathrm{vsr}}(x_m)
= \sum_{m=1}^J N_m\cdot \bar c_{\text{step}}(R^*) \nonumber\\
&= \Bigl(\sum_{m=1}^J N_m\Bigr)\,\bar c_{\text{step}}(R^*)
= N\cdot \bar c_{\text{step}}(R^*),
\label{eq:app_compute_compo_multistage}
\end{align}
The additive budget upper-bounds the expected total cost of reuse.

\paragraph{Step 4: Success probability across stages.}
Assume each stage reaches a well-defined final state after its prescribed $N_m$ accepted steps. Let $S_m$ denote stage-$m$ completion with every accepted state correct. Assume that completing all stages correctly yields a correct final answer under the prescribed extraction. The probability chain rule gives:
\begin{equation}
P\!\left(\bigcap_{m=1}^J S_m\right)
= P(S_1)\prod_{m=2}^J P\!\left(S_m \;\middle|\; \bigcap_{j=1}^{m-1} S_j\right).
\label{eq:app_accuracy_chain}
\end{equation}
As in Step 2 of the proof of \propref{prop:restoration}, completion of all earlier stages with correct states gives stage $m$ a correct accepted input state. The uniform bounds give per-step failure probability at most $\rho^R+\alpha$. Applying \propref{prop:logarithmic_retry} with target failure probability $\delta_m$ and $N_m$ steps yields:
\begin{equation}
P\!\left(S_m \;\middle|\; \bigcap_{j=1}^{m-1} S_j\right)
\;\geq\; 1 - \delta_m, \quad \text{for all } m.
\label{eq:app_conditional_reliability}
\end{equation}
Substituting these conditional stage guarantees into Eq.~\eqref{eq:app_accuracy_chain} gives
\begin{equation}
A_{\mathrm{vsr}}(x) \;\geq\; P\!\left(\bigcap_{m=1}^J S_m\right)
\;\geq\; \prod_{m=1}^J (1-\delta_m),
\label{eq:app_accuracy_product}
\end{equation}

\paragraph{Step 5: Product bounds.}

The product of stage guarantees obeys the following exact inequality:
\[
\prod_{m=1}^J(1-\delta_m)\geq1-\sum_m\delta_m.
\]
For $J=1$ there is equality. Suppose the inequality holds for $J-1$ factors; the induction step gives
\begin{align*}
\prod_{m=1}^J(1-\delta_m) &= \prod_{m=1}^{J-1}(1-\delta_m)\cdot(1-\delta_J) \geq \Bigl(1-\sum_{m=1}^{J-1}\delta_m\Bigr)(1-\delta_J) \\
&= 1 - \sum_{m=1}^J \delta_m + \delta_J\sum_{m=1}^{J-1}\delta_m \;\geq\; 1 - \sum_{m=1}^J \delta_m,
\end{align*}
where the last step drops the nonnegative cross term. Therefore:
\begin{equation}
\prod_{m=1}^J (1-\delta_m)
\;\geq\; 1 - \sum_{m=1}^J \delta_m = 1 - \delta.
\label{eq:app_union_bound}
\end{equation}
The exact product inequality therefore gives the desired guarantee $A_{\mathrm{vsr}}(x)\geq1-\delta$.

\emph{Exponential lower bound:} We use $1-x\geq e^{-x/(1-x)}$ for $x\in[0,1)$, which follows from $\log(1-x)\geq-x/(1-x)$. To verify it, let $f(x)=\log(1-x)+x/(1-x)$; then $f(0)=0$ and $f'(x)=x/(1-x)^2\geq0$.
For $\delta_m = \delta N_m/N \lll 1$, we have $1/(1-\delta_m) = 1+\delta_m+O(\delta_m^2) \approx 1$, so $\delta_m/(1-\delta_m) \approx \delta_m + \delta_m^2$. Substitution gives the following exponential comparison for the product of stage guarantees:
\begin{equation}
\prod_{m=1}^J (1-\delta_m)
\;\geq\; \exp\!\left(-\sum_{m=1}^J \frac{\delta_m}{1-\delta_m}\right)
\;\approx\; \exp\!\left(-\delta - \sum_{m=1}^J \delta_m^2\right).
\label{eq:app_exponential_bound}
\end{equation}
For $\delta_m=O(\delta/J)$, the second-order term is $\sum_m\delta_m^2=O(\delta^2/J)$. As the largest stage budget vanishes, we obtain the exponential bound below; the linear approximation also requires $\delta\to0$:
\begin{equation}
A_{\mathrm{vsr}}(x) \;\geq\; \exp\!\bigl(-\delta - o(\delta)\bigr) \;\approx\; 1-\delta.
\label{eq:app_final_accuracy_bound}
\end{equation}

Together with Step 3, this establishes $A_{\mathrm{vsr}}(x)\geq\prod_m(1-\delta_m)\geq1-\delta$ and an additive computation budget under proportional failure allocation, provided $N\alpha<\delta$.

\subsection{Complete-Sampling Background}
\label{app:sampling_background}
\subsubsection{Accuracy Compositionality under Independent Sampling}
\label{sec:passatk_fails}

At budget $K$, \emph{accuracy compositionality} means that composite accuracy equals $A_{\mtheta,K}(x_1)A_{\mtheta,K}(x_2)$, the probability of solving both subproblems in separate, independent $K$-attempt experiments.
In this established term, accuracy refers to Pass@$K$ coverage.
For the following comparison, write $a_i=A_{\mtheta,1}(x_i)\in(0,1)$ for $i\in\{1,2\}$.
Assume that each subproblem has the same probability of being answered correctly inside a composite run as in its standalone run, and that the two subproblem correctness indicators are independent within each composite run.
Together with independence across the $K$ runs, these assumptions give:
\begin{equation}
\begin{aligned}
A_{\mtheta,1}(x_1\oplus x_2)&=a_1a_2,\\
A_{\mtheta,K}(x_1\oplus x_2)&=1-(1-a_1a_2)^K<A_{\mtheta,K}(x_1)A_{\mtheta,K}(x_2),\quad K\geq2.
\end{aligned}
\label{eq:passk_compositionality_breakdown}
\end{equation}
For $K=1$, the corresponding product relation holds with equality.
The strict gap for $K\geq2$ arises because each composite success requires one run to solve both subproblems, whereas separately covered subproblems may be solved by different runs.
For example, when both subproblems have single-run accuracy $0.5$ and $K=2$, complete-solution success is $0.4375$, compared with $0.5625$ for separate sampling.
This requirement that one run solve both subproblems is the \emph{single-run requirement} under independent sampling.
The event inclusion behind it holds without the product assumptions: a pool containing a jointly correct run also covers each subproblem separately, while the converse need not hold (\propref{prop:breakdown}).
The comparison fixes attempts per subproblem; it does not equate measured token costs across prompt formats.
This bottleneck is a property of sampling and aggregation, and would also occur for modular stochastic solvers with independent errors.

\subsubsection{Sampling Cost as Task Length Grows}
\label{sec:brittle_chain}

To compare computational cost as task length grows at a fixed target success rate, we consider decompositions indexed by $N$ and adopt the idealized scaling assumptions of~\citet{zhang2025reasoning}.
For a fixed model and decoding configuration, assume:
\begin{equation}
C_{\mtheta,1}(x)=\Theta(N),\qquad A_{\mtheta,1}(x)=e^{-\lambda_{\mtheta}N},\qquad \lambda_{\mtheta}>0.
\label{eq:ideal_law}
\end{equation}
$\lambda_{\mtheta}$ is the per-step accuracy decay rate, and linear cost means expected generated tokens scale linearly with $N$ up to constant factors.
For example, the exponential accuracy relation holds when steps are correct independently with probability $1-\varepsilon$, where $0<\varepsilon<1$, and the final answer is correct if and only if all $N$ steps are correct~\citep{barlow1996mathematical}.
In this case:
\begin{equation}
A_{\mtheta,1}(x)=(1-\varepsilon)^N=e^{-\lambda_{\mtheta}N},\qquad\lambda_{\mtheta}=-\log(1-\varepsilon).
\label{eq:single_chain_success}
\end{equation}
For single-run accuracy $p=A_{\mtheta,1}(x)\in(0,1)$, target reliability $1-\delta$ requires at least
\begin{equation}
K_\delta(N)=\left\lceil\frac{\log\delta}{\log(1-p)}\right\rceil.
\label{eq:prelim_target_budget}
\end{equation}
At fixed $\delta$, the exponential accuracy law gives the asymptotic total generation cost
\begin{equation}
C_{\mathrm{indep}}(N,\delta)=K_\delta(N)\,C_{\mtheta,1}(x)=\Theta\!\left(N e^{\lambda_{\mtheta}N}\log(1/\delta)\right).
\label{eq:passk_exponential_cost}
\end{equation}
This restart cost motivates the following analysis of reliability and cost under state reuse.

\subsection{Sufficient Budgets and Stage Composition}
\label{app:prior_budgets}
The following bounds give sufficient retry budgets and stage-composition guarantees.
They follow from the local failure decomposition and remain valid under the execution assumptions of Section~\ref{sec:connected_solution}.
The expected-work theorem sharpens the bound based on the reserved budget while retaining these guarantees.
Summing the first-failure probabilities gives the trajectory bound:
\begin{equation}
P(\text{trajectory failure})\leq N\bigl(\rho^R+\alpha\bigr).
\label{eq:vsr_global_failure}
\end{equation}
For $0<\rho<1$ and $N\alpha<\delta$, a sufficient common cap is:
\begin{equation}
R^*=\left\lceil\frac{\log\!\bigl(N/(\delta-N\alpha)\bigr)}{\log(1/\rho)}\right\rceil.
\label{eq:vsr_attempt_budget}
\end{equation}
For stages of $N_1,N_2$ steps, allocate $\delta_i=\delta N_i/(N_1+N_2)$. Let $S_i$ denote correct completion of stage $i$. If every correct stage-1 output supports the stage-2 guarantee and $S_1\cap S_2$ yields the target answer, then:
\begin{equation}
A_{\mathrm{vsr}}(x_1\oplus x_2)\geq P(S_1\cap S_2)=P(S_1)P(S_2\mid S_1)\geq(1-\delta_1)(1-\delta_2)\geq1-\delta.
\label{eq:accuracy_compositionality}
\end{equation}
Let $c_{\mathrm{gen}},c_{\mathrm{ver}}$ bound per-attempt generation and checking cost at all reached states, with other processing included in these units. The expected total cost is bounded by the deterministic budget:
\begin{equation}
C_{\mathrm{vsr}}(N,\delta)\leq\overline C_{\mathrm{vsr}}(N,\delta):=NR^*(c_{\mathrm{gen}}+c_{\mathrm{ver}}).
\label{eq:vsr_cost}
\end{equation}
For fixed $\rho\in(0,1)$, $N\alpha\leq c\delta$ with fixed $c<1$, and bounded per-attempt costs, the sufficient scaling remains:
\begin{equation}
R^*=\Theta\!\bigl(\log(N/\delta)\bigr),\qquad C_{\mathrm{vsr}}=O\!\bigl(N\log(N/\delta)\bigr).
\label{eq:vsr_scaling}
\end{equation}
At a fixed correct input state $s$, let $T_t$ count local attempts. The truncated geometric identity gives:
\begin{equation}
\mathbb E[T_t\mid s]=\sum_{r=1}^{R}\rho_t(s)^{r-1}
=\frac{1-\rho_t(s)^R}{1-\rho_t(s)}\leq\frac{1}{1-\rho}.
\label{eq:expected_attempts}
\end{equation}
The product factors are conditional stage guarantees, rather than standalone accuracies under different prompts. If per-attempt cost grows, its dependence on task length must remain in the cost bound.

\subsection{Proofs for Reliability and Expected Work}
\label{app:new_theory}

\subsubsection{Exact Reliability and Minimal Common Cap}
\label{app:new_reliability}

Let $\mathcal F_t$ contain the complete history at entry to step $t$, including the event that the step is reached.
At a fixed correct entry history, put $q=\rho_t(s)$ and $w=\Pr(\cG_t=0,\cV_t=1\mid\mathcal F_t)$.
The probability of first accepting an incorrect state on attempt $r$ is $q^{r-1}w$.
Summing over $r\leq R$ and dividing by $1-q^R$ gives $\alpha_t(R;s)=w/(1-q)$ whenever $q<1$.
This establishes Eq.~\eqref{eq:fixed_check_contamination} for every finite cap under the fixed retry law.
The probability of correct advancement is exactly $(1-q^R)(1-\alpha_t(s))$.

On $H_{t-1}$, all entry histories are correct, so their conditional average has correct-advancement probability at least $(1-\rho^R)(1-\alpha)$.
Multiplying these conditional lower bounds proves \thmref{thm:reliability}.
No independence across entry states is needed.
The homogeneous construction in that theorem realizes equality and, by making errors irreversible, equates final-answer and trajectory success.
Consequently the lower bound is also the infimum of final-answer success over the specified model class.

Write $d=(1-\delta)^{1/N}$.
For $0<\rho<1$, a common cap certifies the target over the entire class exactly when
\begin{equation}
(1-\rho^R)(1-\alpha)\geq d.
\end{equation}
For finite $R$, $1-\rho^R<1$, so this is possible if and only if $\alpha<1-d$.
In that case its minimal integer solution is
\begin{equation}
R_{\min}=\max\left\{1,\left\lceil
\frac{\log\bigl(1-d/(1-\alpha)\bigr)}{\log\rho}
\right\rceil\right\}.
\label{eq:exact_retry_budget}
\end{equation}
Both logarithms are negative; division by $\log\rho$ reverses the inequality used to solve for $R$.
If $\alpha=1-d$, the target is attained only in the limit of unbounded retries.
If $\alpha>1-d$, even that limit cannot give a uniform guarantee.
When $\rho=0$, the exact condition is $(1-\alpha)^N\geq1-\delta$, and $R=1$ suffices.
At $\alpha=1$, no positive reliability guarantee follows.\looseness=-1 

A convenient sufficient choice comes from
\begin{equation}
1-\big[(1-\rho^R)(1-\alpha)\big]^N
\leq N\big[\rho^R+(1-\rho^R)\alpha\big]
\leq N(\rho^R+\alpha).
\label{eq:new_union_recovery}
\end{equation}
The first inequality is $1-(1-u)^N\leq Nu$ for $u\in[0,1]$.
Solving $\rho^R\leq(\delta-N\alpha)/N$ gives Eq.~\eqref{eq:vsr_attempt_budget} whenever $N\alpha<\delta$.
As $N\to\infty$ at fixed $\delta$,
$1-(1-\delta)^{1/N}=-\log(1-\delta)/N+O(N^{-2})$.
Replacing the numerator by $\delta$ additionally requires a small-$\delta$ approximation.

\subsubsection{Heterogeneous Steps and Multiple Stages}

Suppose step $t$ has predetermined cap $R_t$ and bounds $\rho_t<1$ and $\alpha_t<1$ valid at every correct entry history.
Conditioning on $H_{t-1}$ as above gives
\begin{equation}
\Pr(H_N)\geq\prod_{t=1}^N(1-\rho_t^{R_t})(1-\alpha_t).
\label{eq:heterogeneous_reliability}
\end{equation}
A bound averaged over standalone inputs need not hold for states produced by earlier stages.

Partition the steps into stages of sizes $N_1,\ldots,N_m$.
If stage $j$ completes correctly with conditional probability at least $1-\delta_j$ from every correct entry state, then the probability of correct completion of all stages is at least
\begin{equation}
\prod_{j=1}^m(1-\delta_j)\geq1-\sum_{j=1}^m\delta_j.
\label{eq:new_stage_product}
\end{equation}
For the common bounds of the main text, allocating $\delta_j=\delta N_j/N$ and using Eq.~\eqref{eq:vsr_attempt_budget} supplies such guarantees.
These factors need not equal standalone success probabilities.
The second stage can meet its target with or without a correct preceding state, despite different success rates.

\subsubsection{Expected Attempts: Full Conditional Proof}
\label{app:work_proof}
Let $I_t$ indicate that step $t$ is reached, and define its conditional single-attempt rejection probability $\rho_t$ given $\mathcal F_t$ when $I_t=1$.
This abbreviation records rejection conditional on the entry history, at both correct and incorrect states.
Assign $\rho_t=0$ on unreachable histories, which do not contribute to the sums below.
Let $T_t=0$ when $I_t=0$ and otherwise count the attempts at this step.
Conditional i.i.d.\ retries imply
\begin{equation}
\mathbb E[T_t\mid\mathcal F_t]=I_t\sum_{r=0}^{R-1}\rho_t^r.
\end{equation}
This identity includes $\rho_t=1$, for which the sum is $R$.
For any $\theta\in(0,1)$, the low-rejection contribution satisfies
\begin{equation}
\sum_{t=1}^N\mathbb E[T_t\mathbf1\{\rho_t\leq\theta\}]
\leq\frac{1}{1-\theta}\sum_{t=1}^N\Pr(I_t=1,\rho_t\leq\theta)
\leq\frac{N}{1-\theta}.
\end{equation}
Let $F_t$ denote exhaustion at step $t$ and $Z=\sum_t I_t\mathbf1\{\rho_t>\theta\}$.
Exhaustion has conditional probability $\rho_t^R$ at a reached step, hence
\begin{equation}
\theta^R\mathbb E[Z]
\leq\sum_{t=1}^N\mathbb E[I_t\mathbf1\{\rho_t>\theta\}\rho_t^R]
=\sum_{t=1}^N\Pr(F_t,\rho_t>\theta)\leq1.
\end{equation}
The last inequality uses mutual exclusivity of exhaustion events under global stopping.
Since a high-rejection visit costs at most $R$, its total expected contribution is at most $R\theta^{-R}$.
Adding the two contributions gives
\begin{equation}
\mathbb E[W]\leq\frac{N}{1-\theta}+\frac{R}{\theta^R}.
\label{eq:threshold_work}
\end{equation}
This estimate requires no bound on $\rho_t$, including at states reached after an incorrect acceptance.

For the uniform upper bound, put $\ell_N=\log(N+1)$ and $\theta=e^{-\ell_N/(2R)}$.
The inequality $e^u\geq1+u$ yields $(1-e^{-u})^{-1}\leq1+1/u$, so
\begin{equation}
\mathbb E[W]\leq N+\frac{2NR}{\ell_N}+R\sqrt{N+1}
\leq N+\frac{3NR}{\ell_N}.
\end{equation}
To verify the last step for every $N\geq1$, set $x=\sqrt{N+1}\geq1$.
The function $x-x^{-1}-2\log x$ is zero at one and has derivative $(x-1)^2/x^2\geq0$.
Thus $\log(N+1)\sqrt{N+1}\leq N$.
Independently of this expectation bound, the deterministic cap gives $\mathbb E[W]\leq NR$.

For the lower bound, fix $q=(N+1)^{-1/R}$ at every reached state and always advance on acceptance.
Let $v=1-q^R=N/(N+1)$.
The probability of reaching step $t$ is $v^{t-1}$, and its conditional expected attempts equal $v/(1-q)$.
Therefore
\begin{equation}
\mathbb E[W]=\frac{v}{1-q}\sum_{t=1}^Nv^{t-1}
=\frac{N(1-v^N)}{1-e^{-\ell_N/R}}.
\end{equation}
The binomial inequality $(1+1/N)^N\geq2$ gives $1-v^N\geq1/2$.
Also $1-e^{-u}\leq\min\{1,u\}$, so
\begin{equation}
\mathbb E[W]\geq\frac{N}{2}\max\left\{1,\frac{R}{\ell_N}\right\}
\geq\frac14\left(N+\frac{NR}{\ell_N}\right).
\end{equation}
Combining the upper and lower estimates yields the finite bounds
\begin{equation}
\frac14\left(N+\frac{NR}{\log(N+1)}\right)
\leq\sup_{P\in\mathcal P_{N,R}}\mathbb E_P[W]
\leq N+\frac{3NR}{\log(N+1)}.
\end{equation}
This proves the matching order for the entire class $\mathcal P_{N,R}$.
It is a worst-case expectation over the class, not a lower bound on every individual process.\looseness=-1 

\subsubsection{Combining Reliability and Work}
\label{app:reliable_cost}
For fixed $\rho\in(0,1)$ and $\alpha\leq c\delta/N$ with $c<1$, the cap in Eq.~\eqref{eq:vsr_attempt_budget} is at most
$1+\log(N/((1-c)\delta))/\log(1/\rho)$.
At fixed $\delta$, this is $O(\log N)$ for $N\geq2$.
Theorems~\ref{thm:reliability} and~\ref{thm:work} therefore imply $\Pr(H_N)\geq1-\delta$ and $\mathbb E[W]=O(N)$.
Since $W\geq N\mathbf1_{H_N}$, its expectation is also at least $(1-\delta)N$.
A uniform per-attempt cost bound $g_N$ on all reached states gives expected total cost $O(Ng_N)$ and a deterministic total-cost cap for the entire run of $O(N\log N\,g_N)$.\looseness=-1 

If $\delta_N\geq N^{-d}$ for a fixed $d>0$, the cap remains $O(\log N)$ and the same expectation upper bound follows.
For smaller targets, substitute the actual $R$ into \thmref{thm:work}.
For example, $\delta_N=e^{-\kappa N}$ can require $R=\Theta(N)$ under the sufficient allocation, yielding an $O(N^2/\log N)$ work bound from the unrestricted process class.
This is an upper bound for reliable processes; the matching witness for the unrestricted class is not asserted to meet that exponentially small failure target.\looseness=-1 

\paragraph{Proof of the acceptance-sensitive bound.}
For the cap in Eq.~\eqref{eq:coverage_cap}, every correct prefix exhausts with probability at most $(1-c)\delta/N$.
Together with $\alpha\leq c\delta/N$, the first-failure bound gives $\Pr(H_N)\geq1-\delta$.
For fixed $c,\delta$, $\log(N/((1-c)\delta))/\log(N+1)$ is bounded uniformly for $N\geq1$.
The ceiling and $h_N\leq1$ therefore give $R_N/\log(N+1)=O(1/h_N)$.
Applying \thmref{thm:work} proves $\mathbb E[W]=O(N/h_N)$.
For the lower construction, set $q=1-h_N$ and accept only correct states.
Writing $v=1-q^{R_N}$, each step's expected attempts equal $v/h_N$ and the success probability is $v^N\geq1-\delta$.
Thus $v\geq1-\delta$ and the probability of reaching any step is at least $1-\delta$, giving $\mathbb E[W]\geq N(1-\delta)^2/h_N$.
This also covers $h_N=1$, where exactly $N$ attempts are executed.
If every attempt has cost $g_N$, the cost upper bound has a matching construction in those units as well.\looseness=-1 

\subsubsection{Proofs for the Binary-Checking Examples}
\label{app:checking_proofs}
\paragraph{Binary-checking model.}
At every correct prefix, a new candidate is correct with fixed probability $p\in(0,1)$.
Given its correctness, repeated checks are independent: a correct candidate passes with probability $a$, an incorrect one with probability $b$, where $0<b<a<1$.
The parameters are fixed across correct prefixes and task lengths.
Each candidate and its checks form one conditionally independent attempt at a fixed entry history.
Different random seeds do not by themselves establish this noise model; shared candidate difficulty or systematic errors can violate it.
For unanimity among $k$ checks, the acceptance probability and accepted-state error are
\begin{equation}
h_k=pa^k+(1-p)b^k,\qquad
\alpha_k=\frac{(1-p)b^k}{pa^k+(1-p)b^k}.
\label{eq:unanimity}
\end{equation}
All $k$ checks are executed when full-check costs are counted below.

\paragraph{Unanimity: precision and acceptance probability.}
The probability that all $k$ checks accept a correct candidate is $a^k$, and it is $b^k$ for an incorrect candidate.
Bayes' rule gives Eq.~\eqref{eq:unanimity}.
For a target $\epsilon\in(0,1)$, the precision condition is
\begin{equation}
\alpha_k\leq\epsilon
\quad\Longleftrightarrow\quad
\left(\frac ba\right)^k\leq\frac{\epsilon p}{(1-\epsilon)(1-p)}.
\end{equation}
For the smallest integer $k\geq1$ satisfying this with $\epsilon=c\delta/N$, where $p,a,b,c,\delta$ are fixed and $0<c<1$, we have
$k=\log N/\log(a/b)+O(1)$.
Hence $a^k=\Theta(N^{-\gamma})$ for $\gamma=-\log a/\log(a/b)>0$.
Moreover, $h_k=pa^k/(1-\alpha_k)=\Theta(N^{-\gamma})$.
To make $(1-h_k)^R$ of order $1/N$, one needs
$R=\Theta(\log N/[-\log(1-h_k)])=\Theta(N^\gamma\log N)$.
Together with the accepted-state error allocation, a sufficiently large constant in this cap guarantees $\Pr(H_N)\geq1-\delta$.

\paragraph{Unanimity: expected-cost lower bound.}
Use the homogeneous law at every correct prefix and make any incorrect acceptance irreversible.
Write $h=h_k$, $\alpha=\alpha_k$, and $u=(1-(1-h)^R)(1-\alpha)$.
Then $\Pr(H_N)=u^N$.
If $u^N\geq1-\delta$, each correct-prefix reach probability $\Pr(H_{t-1})=u^{t-1}$ is at least $1-\delta$.
Also $1-(1-h)^R\geq u\geq1-\delta$.
The expected attempts from each such prefix are
$[1-(1-h)^R]/h\geq(1-\delta)/h$.
Thus
\begin{equation}
\mathbb E[W]\geq\sum_{t=1}^N\Pr(H_{t-1})\frac{1-(1-h)^R}{h}
\geq\frac{N(1-\delta)^2}{h}=\Omega(N^{1+\gamma}).
\label{eq:unanimity_work_lower}
\end{equation}
Counting $k$ unit-cost checks on every attempt gives $\Omega(N^{1+\gamma}\log N)$ checking cost.
If checking stops at the first failed check, the same candidate lower bound applies, but multiplying by $k$ is no longer justified.
The lower bound counts work on correct prefixes only, so arbitrary behavior after an error cannot invalidate it.
For the matching upper bound, choose $R_N$ from Eq.~\eqref{eq:coverage_cap} with $h_N=h_k$.
All $k$ unit-cost checks are executed on each candidate and generation has uniformly bounded cost, so $g_N=\Theta(k)$ on all reached states.
\corref{cor:coverage_cost} gives $O(Nk/h_k)=O(N^{1+\gamma}\log N)$, matching the lower bound.
The upper bound permits arbitrary rejection at incorrect states, subject to the same conditional i.i.d.\ local retry and global stopping rules.\looseness=-1 

\paragraph{Majority: precision with nonvanishing acceptance probability.}
For odd $k$, let $u_k$ be the probability of majority acceptance given a correct candidate and $v_k$ the same probability given an incorrect candidate.
Under $a>1/2>b$, Hoeffding's inequality yields
$1-u_k\leq e^{-2k(a-1/2)^2}$ and $v_k\leq e^{-2k(1/2-b)^2}$.
For all sufficiently large $k$, $u_k\geq1/2$, so
\begin{equation}
h_k=pu_k+(1-p)v_k\geq p/2,\qquad
\alpha_k=\frac{(1-p)v_k}{h_k}
\leq\frac{2(1-p)}p e^{-2k(1/2-b)^2}.
\end{equation}
Choose an odd $k$ of order $\log(N/\delta)$ with a sufficiently large constant to get $\alpha_k\leq c\delta/N$.
The correct-prefix rejection bound $\rho=1-p/2<1$ is independent of $N$.
A cap $R=O(\log(N/\delta))$ now certifies reliability.
For fixed $\delta$, conditional i.i.d.\ candidate attempts at every reached state and uniform per-check cost bounds give $\mathbb E[W]=O(N)$ and $\mathbb E[\mathrm{cost}]=O(N\log N)$.
The pathwise cap is $NRk=O(N\log^2N)$ in the same units.
The precision claim uses the noise law only at correct prefixes; the expected-work bound also requires the fixed retry law at incorrect states.\looseness=-1 

\paragraph{Dependence and boundary cases.}
If repeated checks share a latent blind spot, their false-acceptance probability need not decrease exponentially with $k$.
Independent random seeds do not remove dependence caused by a shared candidate or systematic misunderstanding.
The conditional Bernoulli law above fixes $a$ and $b$ within each correctness class and excludes such additional latent variation.
The unanimity exponent assumes $a<1$ and $b>0$; perfect true acceptance or zero false acceptance needs a separate calculation.
Majority's fixed threshold at one half additionally requires $a>1/2>b$.

\subsubsection{Relation to Classical Stopped Testing}
\label{app:stopping_context}

\paragraph{The stopping pattern.}
Each attempt is a binary trial: acceptance completes a step; rejection leaves it unfinished.
The run ends at the $N$th acceptance or the first run of $R$ consecutive rejections.
With constant rejection probability, this is a special case of total-success/consecutive-failure stopping in start-up demonstration testing.
\citet{gera2013startup} studies such criteria under a previous-sum-dependent model; \citet{balakrishnan2014startup} review i.i.d., Markov-dependent, exchangeable, and multistate formulations.

We connect worst-case expected work over $\mathcal P_{N,R}$ to correctness at correct prefixes.

\paragraph{A prescribed homogeneous law.}
For rejection probability $q\in(0,1)$ at every step, let
$a_R(q)=\sum_{j=0}^{R-1}q^j=(1-q^R)/(1-q)$.
A step takes $a_R(q)$ attempts in expectation and advances with probability $1-q^R$, so
\begin{equation}
\mathbb E_q[W]
=a_R(q)\sum_{t=0}^{N-1}(1-q^R)^t
=\frac{1-q^R}{1-q}\frac{1-(1-q^R)^N}{q^R}.
\label{eq:homogeneous_work}
\end{equation}
At $q=0$ the work is $N$; at $q=1$ it is $R$.
This fixed-law calculation supplies the lower construction in \thmref{thm:work}; it does not bound every process in $\mathcal P_{N,R}$.

\paragraph{Exact recursion for the process class.}
Let $V_n$ be the supremum of expected remaining attempts with $n$ steps left and the common cap $R$.
Then $V_0=0$ and
\begin{equation}
V_n=\max_{q\in[0,1]}\left\{
a_R(q)+(1-q^R)V_{n-1}\right\}.
\label{eq:work_bellman}
\end{equation}
Condition on the complete entry history.
Current-step expected work is $a_R(q)$; acceptance leaves at most $V_{n-1}$ expected attempts, while exhaustion ends the run.
Continuity on $[0,1]$ guarantees a maximum.
Choose a maximizing $q$ and, after acceptance, repeat for $n-1$ steps.
Induction gives equality and an attaining process whose rejection probability depends only on the number of remaining steps.
History dependence alone is not the distinction: the unrestricted supremum reduces to this scalar recursion.
\thmref{thm:work} gives its order uniformly in $N,R$ without solving the successive maximizations.
This characterizes the process class; a model need not realize its maximizing law.

\paragraph{What is needed for reliable reuse.}
A uniform acceptance lower bound $h$ on all reached states immediately gives $\mathbb E[W]\leq N/h$ by summing geometric waiting times.
In \corref{cor:coverage_cost}, the acceptance lower bound holds only at correct prefixes.
An incorrect acceptance may lead to states with arbitrary rejection probabilities.
Global stopping bounds their work without extending the correct-prefix acceptance assumption to those states.
If rejection changes the retry law, rejecting the first $R-1$ candidates and accepting the last at every step gives $W=NR$.
If exhaustion permits fallback and continuation, always rejecting likewise gives $W=NR$.
These constructions show why the execution assumptions matter.
The lower witness for \thmref{thm:work} completes all steps with probability $(N/(N+1))^N$, approaching $e^{-1}$; \corref{cor:coverage_cost} supplies a separate matching construction within its reliable class.\looseness=-1 

\paragraph{Restart schedules and checking costs.}
\citet{luby1993optimal} optimize schedules for repeated complete runs of Las Vegas algorithms.
Our cap is fixed, and exhaustion ends the task; the result does not optimize restart schedules or decomposition.
Selective classification studies the risk--coverage tradeoff~\citep{elyaniv2010foundations}.
Here its coverage corresponds to acceptance probability, distinct from answer coverage in our experiments.
MAKER already accounts for voting cost and the probability that a generated response survives red-flagging~\citep{meyerson2025solving}.
Precision and acceptance probability are thus established concerns.
Here they enter a first-acceptance process with a finite cap through $g_N/h_N$ and the accepted-state error requirement in \corref{cor:coverage_cost}.
Standard Bayes and concentration calculations~\citep{hoeffding1963probability} show why low error need not imply low cost.\looseness=-1

\newpage

\section{Implementation Details}
\label{app:implementation_details}

This section records the implementation details used for the LoRe-Compo experiments in~\secref{sec:experiments}.

The protocol first generates candidate answers to $x_1$, selects intermediate answers, and uses them when generating answers to $x_2$. It also retains fallback candidates that did not pass the check. This manuscript package documents the protocol but does not include the implementation.

The finite-budget arithmetic experiments in~\secref{sec:two_component} use a separate complete-answer retry setup, described after the LoRe-Compo details; they do not perform digit-level state reuse.

\subsection{LoRe-Compo Implementation}
\label{app:lore_compo_impl}

\paragraph{Dataset format and connector prompt.}
The LoRe-Compo data file contains 500 unique MATH500 subproblems and 250 composite pairs.
Readers can obtain the benchmark online from the LoRe authors' repository at \url{https://github.com/ASTRAL-Group/LoRe}, under \path{LoRe-Bench/LoRe-Compo/compo_bench_math500.jsonl}; this paper does not redistribute the records.
Each subproblem record stores a MATH subject label, problem text, solution, and ground-truth answer.
A pair links two subproblems, their answers, and this prompt:
\begin{mdframed}[
    linewidth=0.4pt,
    linecolor=black!18,
    backgroundcolor=black!2,
    innerleftmargin=8pt,
    innerrightmargin=8pt,
    innertopmargin=5pt,
    innerbottommargin=5pt,
    roundcorner=3pt]
\noindent\textbf{Composite pair prompt.}
\begin{verbatim}
Answer the following questions in order:
Q1. {question 1}
Q2. {question 2}
\end{verbatim}
\end{mdframed}
The Pass@$K$ baseline uses this pair prompt directly for independent sampling over complete runs.
The subject labels are used only as a proxy for reduced concept overlap across subproblems; they do not imply probabilistic independence of model outputs.\looseness=-1 

\paragraph{Independent sampling prompts.}
For each standalone or composite problem, independent sampling uses the template \texttt{Question: \{problem\}} followed by \texttt{Answer:}.
For composite pairs, \texttt{\{problem\}} is the pair prompt above.
We sample $K\in\{1,2,4,8,16,32,64,128\}$ complete runs with a maximum generation length of 10,240 tokens.
The reported runs use vLLM with seed 42, repetition penalty 1.05, and family-specific temperatures 0.6/0.7/0.8 for DeepSeek/Qwen3/GPT-OSS.
For $K{=}1$, the executed parameters use top-$p=1.0$ and top-$k=-1$; the \texttt{do\_sample} field is not consumed by the vLLM call, so Pass@1 is not greedy decoding in this code path.
For $K{>}1$, the executed parameters use top-$p=0.7$ and top-$k=50$.
The LoRe-Compo scripts pass these prompt strings directly to vLLM; model-specific chat templates are not applied in this implementation.

\paragraph{Two-stage protocol prompts.}
Stage 1 samples $M$ answers to $x_1$ with its standalone prompt.

The main two-stage comparison generates $M=8$ candidates and retains up to $B=4$ intermediate answers for continuation. These answers pass the selection rule when it succeeds; otherwise they are fallback candidates.
The budget-allocation sweep also evaluates $(M,B)\in\{(4,2),(8,2),(8,4),(16,8)\}$.
Stages 1 and 2 use the same family-specific temperatures, top-$p=0.7$, top-$k=50$, seed 42, repetition penalty 1.05, and 10,240-token cap.
The main stage-2 prompt supplies the retained answer through the following \texttt{answer\_only} template:
\begin{mdframed}[
    linewidth=0.4pt,
    linecolor=black!18,
    backgroundcolor=black!2,
    innerleftmargin=8pt,
    innerrightmargin=8pt,
    innertopmargin=5pt,
    innerbottommargin=5pt,
    roundcorner=3pt]
\noindent\textbf{Two-stage protocol: stage-2 prompt (\texttt{answer\_only}).}
\begin{verbatim}
Question 1 has been solved. The answer was: {q1 answer}

========================================
Now solve Question 2:
========================================

{question 2}

Important: Question 2 is a different problem.
Think step by step and provide the answer for Question 2.
\end{verbatim}
\end{mdframed}

The prompt ablation additionally evaluates \texttt{no\_context}, which passes only $x_2$, and \texttt{full\_context}, which passes the generated $x_1$ reasoning text and retained $x_1$ answer; retained fallback answers are not accepted.
For \texttt{full\_context}, the implementation truncates the retained $x_1$ context to the most recent portion when it exceeds the configured context budget.
\texttt{no\_context} uses the same single-subproblem prompt for $x_2$, while \texttt{full\_context} adds the retained $x_1$ reasoning and answer before Question 2.

\paragraph{Answer extraction and matching.}
Before extraction, generated text is split at \texttt{</think>} when a reasoning tag is present, and answer extraction is applied first to the post-thinking answer text.
For math-style answers, extraction first uses the final \texttt{\textbackslash boxed\{...\}} expression if present, then explicit answer patterns such as \texttt{Answer:}, \texttt{The answer is}, and \texttt{Final answer:}, and finally falls back to the last numeric expression.
For composite answers, extraction first tries to split the generated text into \texttt{Q1} and \texttt{Q2} sections and extract one answer from each section.
If this split fails, it falls back to the first two numeric expressions when available.
Correctness uses numeric comparison with tolerance $10^{-6}$ when both prediction and target are numeric; otherwise strings are stripped, lowercased, spaces are removed, and exact string equality is used.
No computer algebra system (CAS) checks algebraic equivalence.

\paragraph{Oracle and self-consistency selection.}

Oracle selection compares each stage-1 candidate's extracted $x_1$ answer with the reference answer.
If reference-matching candidates exist, the first $B$ matches are retained and each receives a continuation.\looseness=-1 

If none matches, the fallback keeps the first $B$ candidates so that continuation tokens and second-subproblem success are still measured. All such paths use fallback routing, which the formal loop excludes by terminating without returning a state whenever selection fails.

Self-consistency selection counts the extracted $x_1$ answers among the $M$ candidates.
If the most common answer has vote share at least 0.5, it keeps up to $B$ candidates with that answer.

If no answer reaches this threshold, the empirical protocol keeps the first $B$ candidates as a fallback for measurement; the formal loop would instead stop without returning a state.
Ties are not broken by a separate randomized procedure: the implementation uses the first most common answer returned by the counter over generated candidates.\looseness=-1 

Comparing accepted SC answers against reference answers measures the error rate among accepted states in the evaluated sample. It does not estimate the uniform theoretical bound $\alpha$ across all steps, correct input states, and budgets. An error estimate that includes fallbacks instead describes all retained candidates, including those rejected by the selection rule.\looseness=-1 

\paragraph{Retained-path correctness fields.}
The implementation defines path-level fields for first-answer correctness, second-answer correctness, and their conjunction.

The implementation field \texttt{final\_correct} denotes second-subproblem success: whether any retained continuation solves $x_2$, regardless of the retained $x_1$ answer's correctness. Joint success requires both answers to be correct on a shared continuation path.\looseness=-1 

Joint success must be computed from the path-level field requiring both answers to be correct. Estimating the error rate among accepted states additionally requires excluding fallback paths.

The distinction matters when fallback candidates are retained or the check accepts an incorrect answer. With Oracle selection and at least one reference match, the reused intermediate answers are correct according to the answer matcher; this does not verify the full reasoning traces.

\begin{mdframed}[
    hidealllines=true,
    backgroundcolor=color4!18,
    innerleftmargin=10pt,
    innerrightmargin=10pt,
    innertopmargin=4pt,
    innerbottommargin=4pt,
    roundcorner=4pt]
\noindent\textcolor{color3}{\textbf{Case Example: One LoRe-Compo Pair.}}
Consider the pair with MATH500 indices $(477,21)$.
The first subproblem asks for the number of possible tens digits $B$ such that the two-digit number ``$B6$'' is a square, with answer $2$.
The second asks for the greatest integer less than $(\sqrt{7}+\sqrt{5})^6$, with answer $13535$.
The independent-sampling baseline presents both subproblems in one prompt and succeeds only when a single sampled run returns the correct pair $[2,13535]$.

The two-stage protocol samples $M$ solutions for the first subproblem and selects extracted answers using Oracle matching or SC voting. If an answer passes the selection rule, the model receives that accepted intermediate answer when solving the second subproblem; otherwise it continues from fallback candidates to measure second-subproblem coverage.

The example contrasts complete-solution sampling with separate continuation prompts. The second subproblem does not mathematically depend on the first answer.
\end{mdframed}

\begin{mdframed}[
    hidealllines=true,
    backgroundcolor=color4!18,
    innerleftmargin=10pt,
    innerrightmargin=10pt,
    innertopmargin=4pt,
    innerbottommargin=4pt,
    roundcorner=4pt]
\noindent\textcolor{color3}{\textbf{Case Example: Answer Matching and Selection.}}
The pair with MATH500 indices $(277,222)$ contains a first subproblem asking for $\sqrt{53}$ in simplest radical form, with target answer \texttt{\textbackslash sqrt\{53\}}.
For nonnumeric answers, the checker strips whitespace, lowercases strings, removes spaces, and then applies exact string matching.
Thus \texttt{\textbackslash sqrt\{53\}} is accepted, while a semantically equivalent expression may not be accepted unless extracted into the same \mbox{normalized string}.
For generated composite answers, if \texttt{Q1}/\texttt{Q2} section extraction fails, the extractor falls back to the first two numeric expressions in the post-\texttt{</think>} answer text.

For Oracle selection, if no stage-1 candidate matches the target, the empirical fallback keeps the first $B$ candidates so that continuation budget is still measured; these paths are neither accepted states nor part of the formal product theorem.
Consider these $M{=}8$ candidate answers:
\[
200,\;200,\;200,\;200,\;210,\;300,\;2000,\;\texttt{PARSE\_ERROR}.
\]
The answer $200$ has vote share $4/8=0.5$, meeting the self-consistency threshold, so up to $B=4$ candidates with that answer are kept.
If the top vote share were $3/8$, the protocol would retain the first $B$ candidates as a fallback.
Ties use the first most common answer encountered by the counter; no randomized tie-breaking rule is applied.
\end{mdframed}\looseness=-1 

\subsection{Metric Definitions and Limits of the Reported Aggregates}
\label{app:metric_alignment}

\paragraph{Joint versus second-subproblem coverage.}
Let $i=1,\ldots,D$ index the $D=250$ composite pairs, and let $\mathcal R_i$ be the retained paths that receive continuations.
Write $G^{(1)}_{ib}$ and $G^{(2)}_{ib}$ for the correctness indicators of the first and second answers on path $b\in\mathcal R_i$.
The distinction between second-subproblem coverage and joint coverage is expressed by the statistics
\begin{align}
\widehat A_2^{\mathrm{ret}}
&=\frac1D\sum_{i=1}^{D}\mathbf1\{\exists b\in\mathcal R_i:G^{(2)}_{ib}=1\},\\
\widehat A_{\mathrm{joint}}^{\mathrm{ret}}
&=\frac1D\sum_{i=1}^{D}\mathbf1\{\exists b\in\mathcal R_i:G^{(1)}_{ib}=G^{(2)}_{ib}=1\}.
\end{align}
Only the first statistic is reported for the two-stage experiments. The reported aggregates cannot replace missing path-level records.
The second cannot be reconstructed by multiplying marginal accuracies or from $\widehat A_2^{\mathrm{ret}}$ alone.
Restricting $\mathcal R_i$ to paths that passed selection would define accepted-path coverage; its denominator must still contain all $D$ pairs, with an empty path set contributing zero.
These definitions supply no new measurements and do not make the multiple-path protocol equivalent to the formal loop.

\paragraph{Acceptance and fallback.}
Error among accepted intermediate answers requires comparing those answers to the reference, excluding every fallback that did not pass selection.
A calculation must specify whether it weights individual retained answers or one selected state per pair; multiple retained copies of a modal answer need not define the same distribution as the theorem's returned state.
An empirical average under either choice would not establish a uniform bound over all correct prefixes.
The reported continuation aggregates do not determine this error rate, the frequency of fallback, or the joint contribution of accepted and fallback paths.
Oracle reference matching checks an extracted answer; in the full-context ablation, the retained reasoning text is not separately checked.\looseness=-1 

\paragraph{Partial successes in a common candidate pool.}
For $K$ complete composite runs, let $G_{i1r}$ and $G_{i2r}$ denote subproblem-answer correctness in run $r$.
Define
\begin{align}
J_i(K)&=\mathbf1\{\exists r\leq K:G_{i1r}=G_{i2r}=1\},\\
U_i(K)&=\mathbf1\{\exists r\leq K:G_{i1r}=1\}
          \mathbf1\{\exists r\leq K:G_{i2r}=1\}.
\end{align}
Then $J_i(K)\leq U_i(K)$ for every realized pool, without an independence assumption.
The average $D^{-1}\sum_i[U_i(K)-J_i(K)]$ would measure pools containing separately correct subproblem answers but no jointly correct run.
It is an oracle coverage diagnostic, not a deployable answer-selection rule.
The reported Pass@$K$ values average $J_i(K)$; they cannot identify the difference $U_i(K)-J_i(K)$.

\paragraph{Comparability and uncertainty.}
The decoding differences and unspecified cross-budget nesting in \appref{app:lore_compo_impl} prevent interpreting all composite rows as one controlled budget curve.
Single-subproblem and paired-answer extraction also follow different fallback rules: the paired extractor can return the first two numeric expressions when section parsing fails, while semantic equivalence is not checked by a CAS.
The available aggregate results do not quantify extraction errors or their differential effect between protocols.
Scores depend on the specified prompts, answer matching, and extraction.

There are 250 pairs per model, giving 0.4-percentage-point resolution for each model's pair-level accuracy.
SC--Oracle differences of at most 1.6 points correspond to a net difference of at most four successful pairs; the number of pairs on which the selectors disagree cannot be inferred from these marginal scores.
Paired uncertainty cannot be recovered from the rounded marginal scores alone; the shared test pairs also make models' aggregate outcomes dependent.
We report the observed differences without asserting statistical equivalence, significance, or robustness across decoding seeds.

\subsection{Controlled Arithmetic Implementation}
\label{app:arithmetic_impl}

The Base 7 and Base 10 experiments use the same complete-answer sampling structure. Digit count indexes problem size; no digit-position state is checked or reused.
For each base, task, and digit count, the scripts pre-generate 50 problems using seed $42\cdot 10000 + 100\cdot j+n$, where $j=0$ for addition, $j=1$ for multiplication, and $n$ is the digit count.
The evaluated digit counts are $n\in\{1,\ldots,10\}$ and the tasks are addition and multiplication.
Each problem receives $R_{\max}=32$ independent attempts in one vLLM call, with checkpoints at $R\in\{2,4,8,16,32\}$.\looseness=-1 

Thus each model is evaluated on $2$ bases $\times$ $2$ tasks $\times$ $10$ digit counts $\times$ $50$ problems, with 32 attempts per problem; each combination of model, base, task, and digit count contains 1,600 completions.

\paragraph{Arithmetic prompts.}
The Base 7 template specifies its carry rule and answer format:
\begin{mdframed}[
    linewidth=0.4pt,
    linecolor=black!18,
    backgroundcolor=black!2,
    innerleftmargin=8pt,
    innerrightmargin=8pt,
    innertopmargin=5pt,
    innerbottommargin=5pt,
    roundcorner=3pt]
\noindent\textbf{Base 7 addition prompt.}
\begin{verbatim}
Task: Perform addition in Base-7 (septenary system).
Calculate {a_b7} + {b_b7} in Base-7.
Rules: In Base-7, digits are 0-6. When a digit-sum reaches 7,
carry 1 to the next position.
Think step by step and give your final answer
in Base-7 inside \boxed{}.
\end{verbatim}
\end{mdframed}

Base 7 multiplication uses the analogous prompt with multiplication-specific carry instructions.
Base 10 uses the corresponding decimal arithmetic prompts, asking for the sum or product of two decimal operands and a final boxed numerical answer.\looseness=-1 

\paragraph{Arithmetic sampling and extraction.}
For DeepSeek-family models, the arithmetic scripts use temperature 0.6 and top-$p=0.95$.
For Qwen3 models, they use temperature 0.6, top-$p=0.95$, top-$k=20$, and min-$p=0.0$.
For GPT-OSS models, they use temperature 1.0 and top-$p=1.0$.
All arithmetic runs use seed 42, maximum generation length 10,240 tokens, and \texttt{n=32} completions per prompt.
Answer extraction first uses the final \texttt{\textbackslash boxed\{...\}} expression if present.
Base 7 falls back to the last digit string with commas and spaces removed; Base 10 removes commas and takes the last signed integer.

\paragraph{Unsolved fraction.}
For each problem, the retry matrix stores a binary correctness value for each of the 32 attempts.
The checkpoint statistic $\hat E(R)$ averages failure indicators over the first $R$ attempts:
\[
\hat E(R)=\frac{1}{|\mathcal{D}|}\sum_{x\in\mathcal{D}}\mathbf{1}\{\text{attempts }1,\ldots,R\text{ all fail on }x\}.
\]

The unsolved fraction after 32 attempts is $\hat E(32)$, and empirical Pass@$R$ is $1-\hat E(R)$. Each value of $R$ uses the first $R$ of the same 32 attempts.
In this arithmetic probe, Pass@1 is the first sampled attempt among the 32 returned completions, not a greedy decoding run.

\begin{mdframed}[
    hidealllines=true,
    backgroundcolor=color4!18,
    innerleftmargin=10pt,
    innerrightmargin=10pt,
    innertopmargin=4pt,
    innerbottommargin=4pt,
    roundcorner=4pt]
\noindent\textcolor{color3}{\textbf{Case Example: Base 7 Addition.}}
For a one-digit Base 7 addition problem, the script samples decimal operands in $[1,6]$ and converts them to Base 7.
If the sampled operands are decimal $5$ and $6$, the prompt asks for $5+6$ in Base 7.
The target is decimal $11$, written as \texttt{14} in Base 7.
A final boxed answer of \texttt{14} is correct; a decimal answer \texttt{11} is incorrect because the output must be in Base 7.
\end{mdframed}

\begin{mdframed}[
    hidealllines=true,
    backgroundcolor=color4!18,
    innerleftmargin=10pt,
    innerrightmargin=10pt,
    innertopmargin=4pt,
    innerbottommargin=4pt,
    roundcorner=4pt]

\noindent\textcolor{color3}{\textbf{Case Example: Counting Unsolved Problems.}}
Suppose a condition has 50 problems and, for a given problem, the 32-attempt correctness row is
\[
0,0,0,1,0,\ldots.
\]
This problem contributes to $\hat E(1)$, $\hat E(2)$, and $\hat E(3)$ because the first $R$ attempts all fail for those checkpoints.
It does not contribute to $\hat E(4)$ or later checkpoints because an attempt has succeeded by $R=4$.
Problems left unsolved within this budget are not necessarily impossible to solve.
\end{mdframed}

\section{Arithmetic Results by Base and Digit Count}
\label{app:base_comparison}

Figure~\ref{fig:base7_laws} replots the first-attempt accuracies and mean token counts shown in Figures~\ref{fig:app_pass1_heatmap} and~\ref{fig:app_token_heatmap}.
Figure~\ref{fig:base_comparison} gives the corresponding unsolved fractions for each model and operation.
Digit count indexes input size, not measured reasoning steps; the numerical ranges differ across bases.

\begin{figure}[t]
    \vspace{-4pt}
    \centering
    \captionsetup{skip=4pt}
    \includegraphics[width=\linewidth]{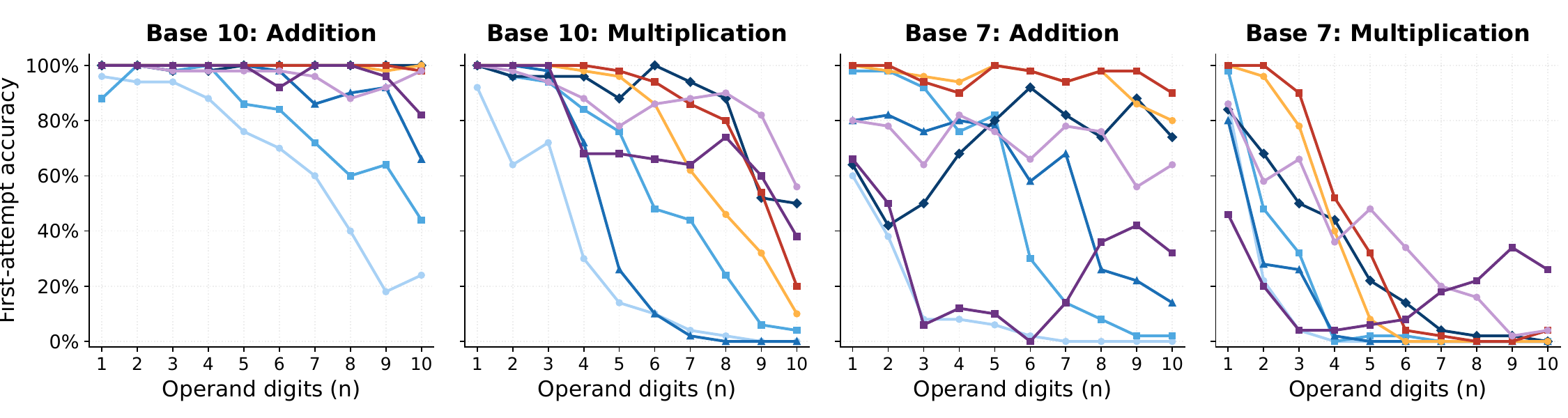}
    \includegraphics[width=0.99\linewidth]{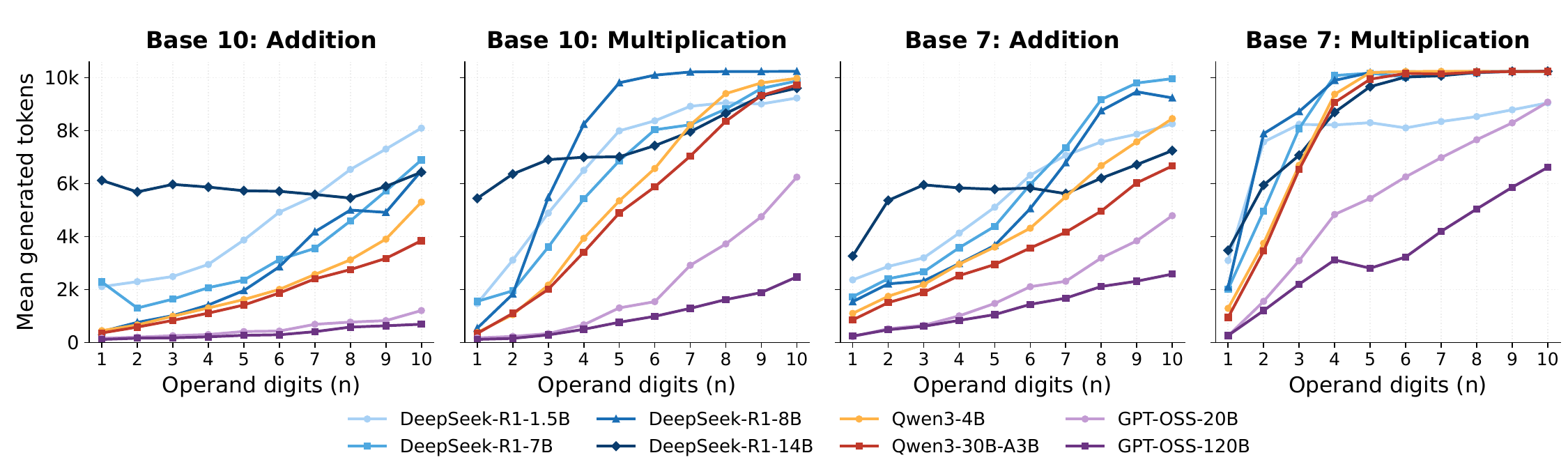}
    \caption{\textbf{Accuracy and Generated Tokens by Digit Count.} Top: first-attempt accuracy from Figure~\ref{fig:app_pass1_heatmap}. Bottom: mean generated tokens per attempt from Figure~\ref{fig:app_token_heatmap}. Both rows plot the reported heatmap cells under the protocol in Appendix~\ref{app:arithmetic_impl}; colors identify models.}
    \label{fig:base7_laws}
    \vspace{-6pt}
\end{figure}

\begin{figure}[h]
    \centering
    \includegraphics[width=\linewidth]{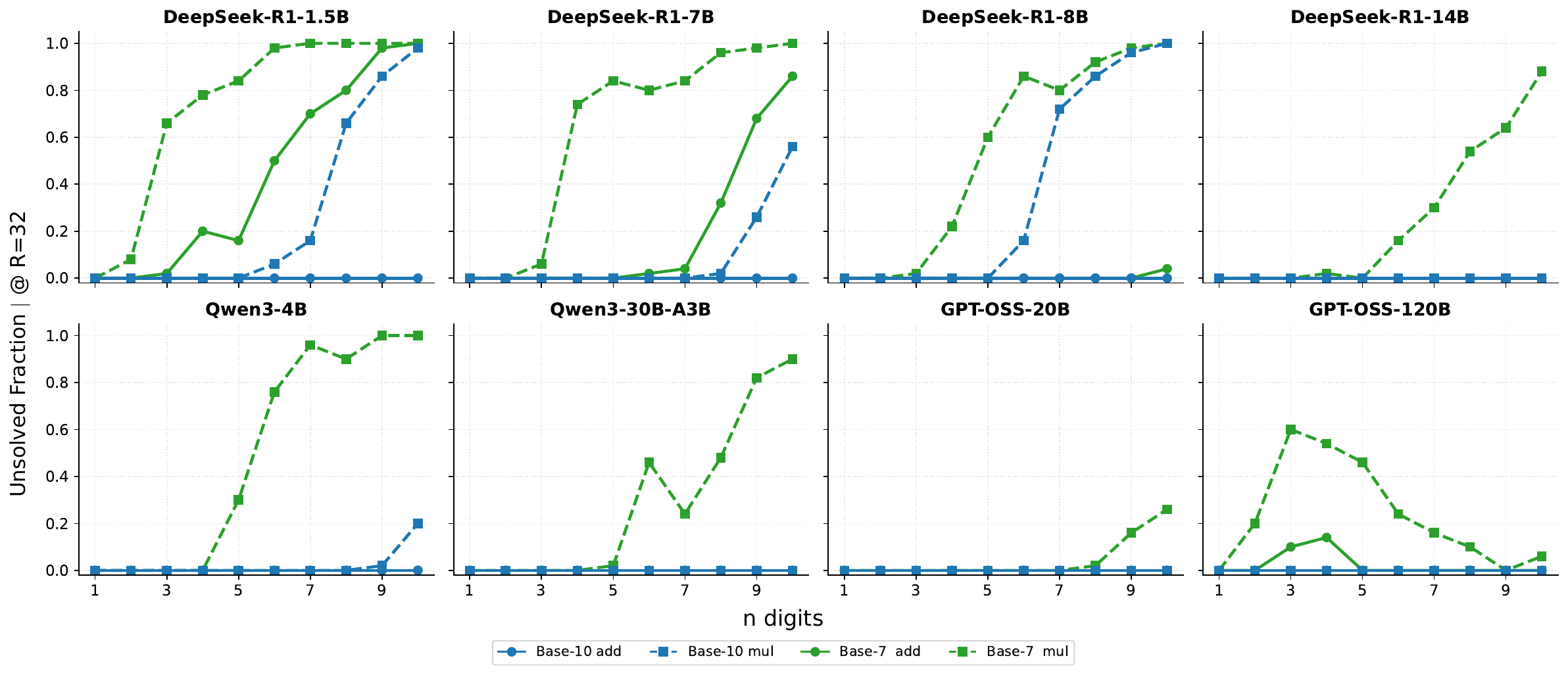}
	    
    \caption{\textbf{Unsolved Fraction in Base 7 and Base 10.} Each curve reports $\hat E(32)$ by digit count for one model and operation, comparing the two bases under the stated sampling protocol.}
    \label{fig:base_comparison}
\end{figure}

\section{Detailed Results of Repeated Arithmetic Attempts}
\label{app:floor_analysis}

The arithmetic protocol in Appendix~\ref{app:arithmetic_impl} yields 320 combinations of model, base, operation, and digit count, each with 50 problems and 32 complete-answer attempts per problem.
The following figures show the per-condition curves, single-attempt scores, and token measurements.

\subsection{Unsolved-Fraction Curves}
\label{app:er_curves}

Figure~\ref{fig:app_er_curves} shows $\hat E(R)$ for all 320 conditions.
At $R=32$, 235 conditions have no unsolved problems and 85 have at least one.
All 80 Base 10 addition conditions reach zero unsolved problems. Base 7 multiplication has the largest mean unsolved fraction, 0.40, with 33 of 80 conditions above 0.5.
The curves vary across models and digit counts, including GPT-OSS-120B's nonmonotonic low-digit Base 7 multiplication results.
Appendix~\ref{app:mixture_fit} defines the two-component fits shown as dashed curves.

\begin{figure}[h]
    \centering
    \includegraphics[width=0.965\linewidth]{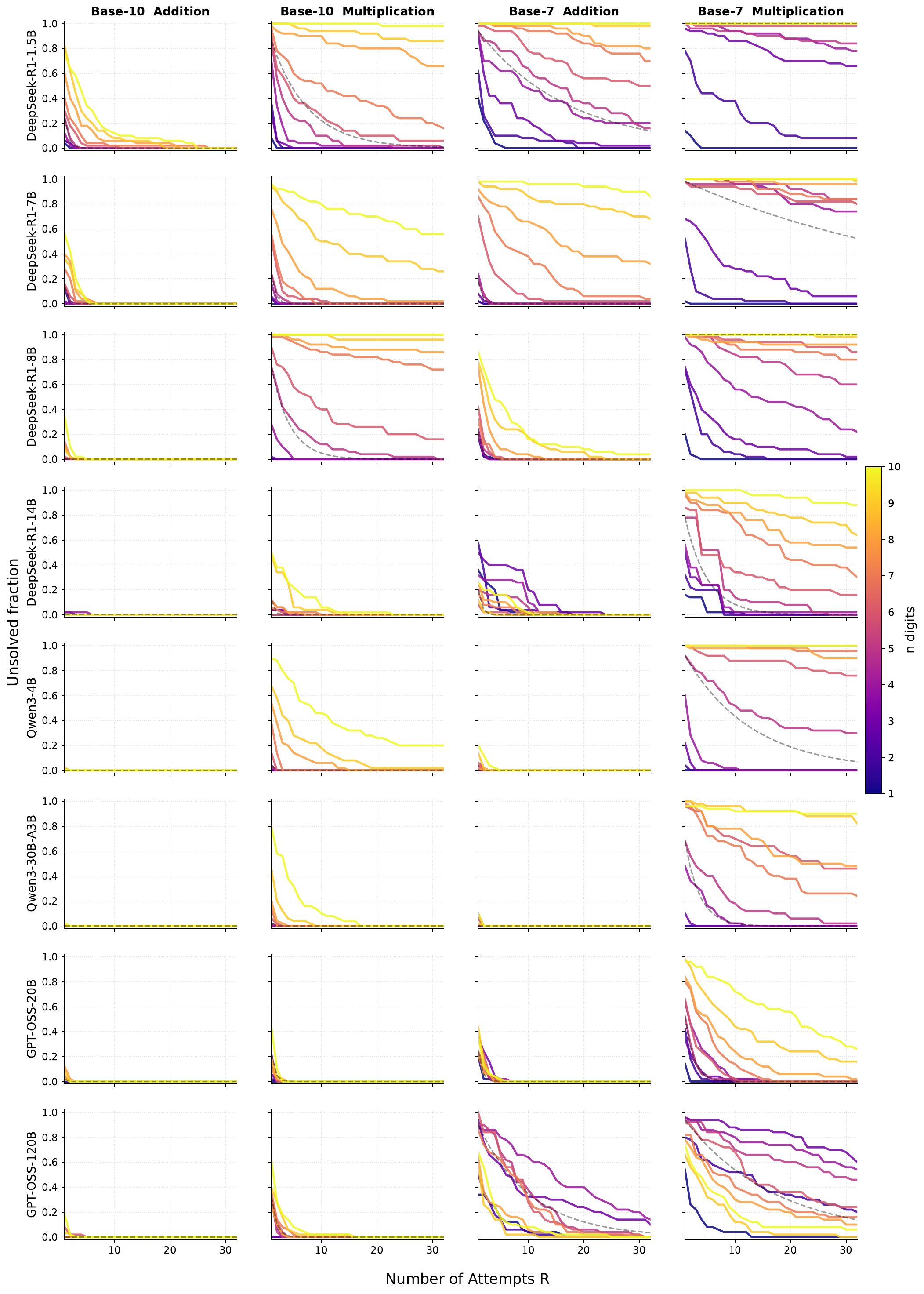}
    \caption{\textbf{Empirical Unsolved-Fraction Curves.} Rows are models; columns are Base 10 addition, Base 10 multiplication, Base 7 addition, and Base 7 multiplication. Colors encode $n=1,\ldots,10$. Solid curves show $\hat E(R)$ and dashed gray curves show fitted two-component models. The panels retain both zero and nonzero unsolved fractions within 32 attempts, including GPT-OSS-120B's nonmonotonic Base 7 multiplication results.}
    \label{fig:app_er_curves}
\end{figure}

\subsection{First-Attempt Accuracy}
\label{app:pass1_heatmap}

Figure~\ref{fig:app_pass1_heatmap} complements the unsolved-fraction heatmap by showing Pass@1.
Low Pass@1 can coexist with high Pass@32: R1-1.5B on Base 7 addition at $n=3$ has Pass@1$=0.08$ and Pass@32$=0.98$, with $\hat E(32)=0.02$.
Across the 320 conditions, 12 have a gain from additional attempts, Pass@32 minus Pass@1, greater than 0.80.
For R1-1.5B on Base 7 multiplication at $n\geq7$, Pass@1 remains zero and $\hat E(32)=1$.
These cases distinguish first-attempt accuracy from finite-pool answer coverage, without identifying how an intermediate state would affect later steps.

\begin{figure}[h]
    \centering
    \includegraphics[width=\linewidth]{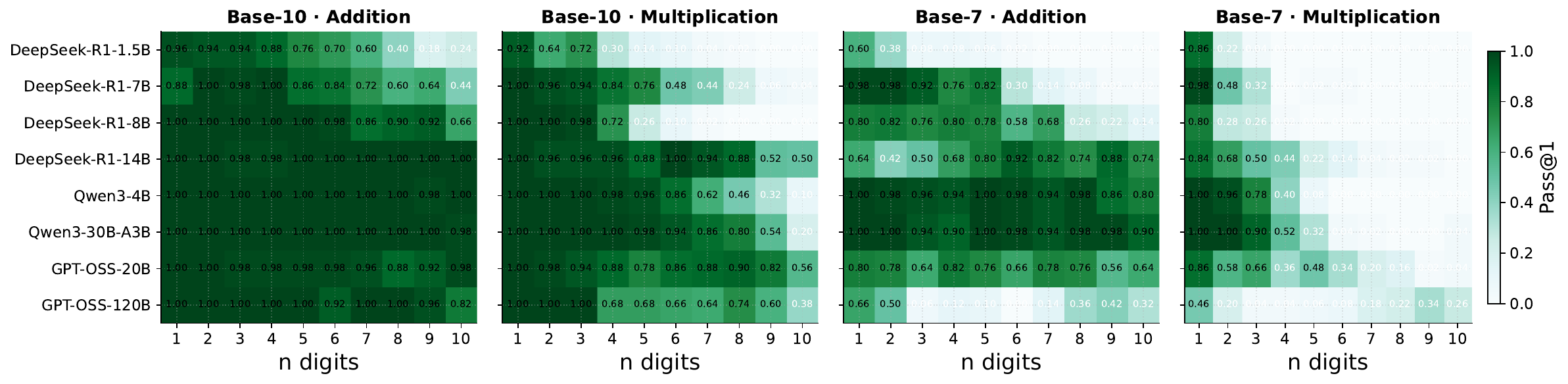}
    \caption{\textbf{First-Attempt Accuracy.} For Base 7 multiplication, R1-14B ($n=6$) has Pass@1=0.14 and $\hat E(32)=0.16$ (84\% solved within 32 attempts); R1-1.5B ($n=7$) has Pass@1=0 and $\hat E(32)=1$.}
    \label{fig:app_pass1_heatmap}
\end{figure}

\subsection{Unsolved Fraction at Different Attempt Budgets}
\label{app:checkpoint_convergence}

Figure~\ref{fig:app_checkpoint_convergence} reports $\hat E(R)$ at $R\in\{2,4,8,16,32\}$.
Across 320 conditions, the mean absolute gap between $\hat E(16)$ and $\hat E(32)$ is 0.031, with a maximum of 0.32.
These differences describe changes within the observed budget; they do not establish convergence beyond it.
Because every checkpoint uses a prefix of the same pool, solved problems remain covered at later checkpoints.

\begin{figure}[h]
    \centering
    \includegraphics[width=\linewidth]{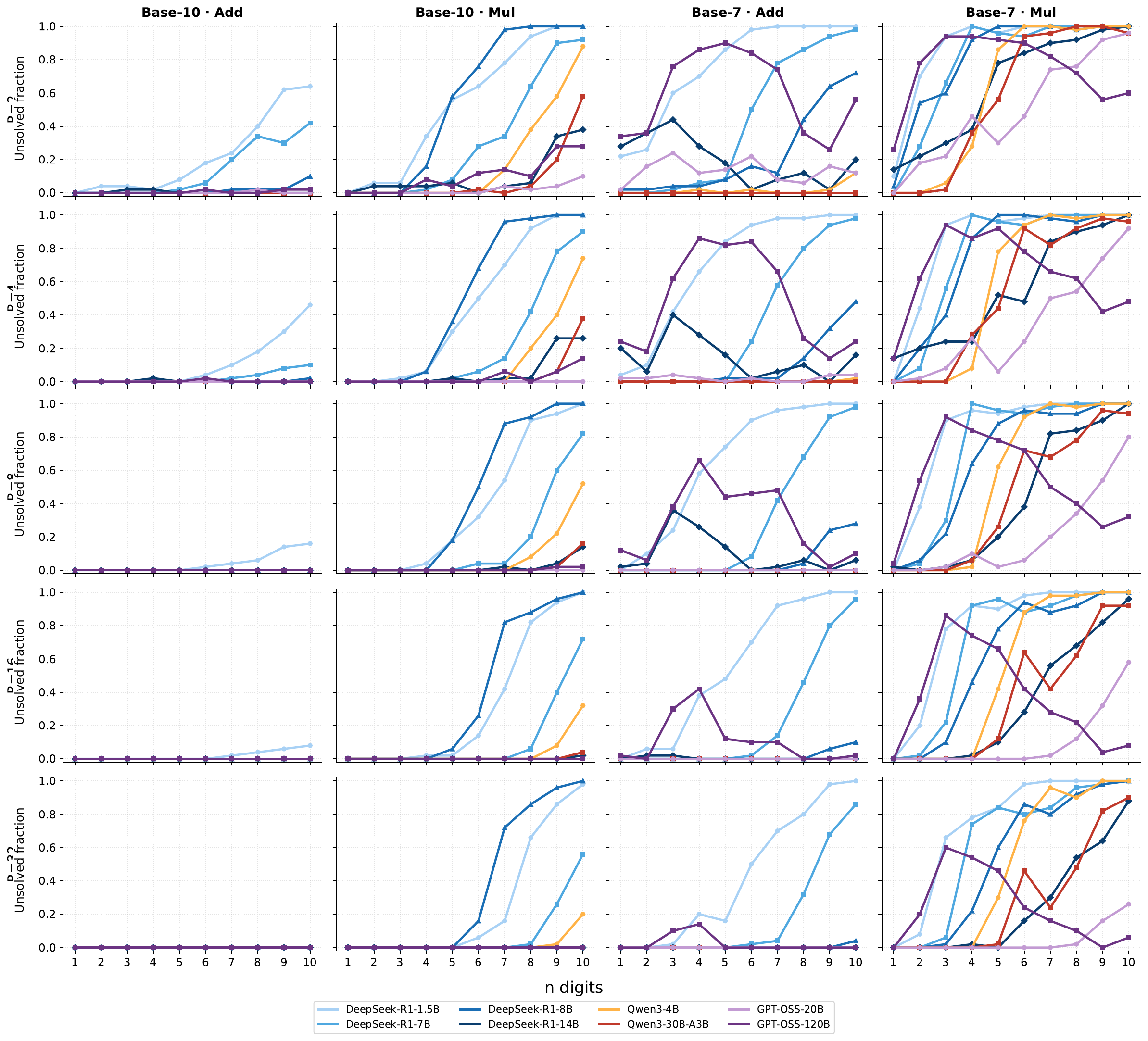}
    
    \caption{\textbf{Unsolved Fraction at Different Attempt Budgets.} Rows correspond to $R\in\{2,4,8,16,32\}$ and columns to the four base--operation combinations. Curves show $\hat E(R)$ against digit count; colors identify models. The reported mean absolute gap between $R=16$ and $R=32$ is 0.031 across 320 conditions, with a maximum of 0.32; every checkpoint uses the same nested attempt pool.}
    \label{fig:app_checkpoint_convergence}
\end{figure}

\subsection{Two-Component Mixture Fits}
\label{app:mixture_fit}

We fit the empirical curves by nonlinear least squares using
\begin{equation}
    \hat E(R)\approx \pi_s\varepsilon_{\mathrm{good}}^R
    +(1-\pi_s)\varepsilon_{\mathrm{bad}}^R,
    \label{eq:empirical_retry_mixture}
\end{equation}
Here $\pi_s$ is the weight of the faster-decaying component and $\varepsilon_{\mathrm{good}}$ its per-attempt failure rate. The remaining weight $1-\pi_s$ belongs to a slowly decaying component with $\varepsilon_{\mathrm{bad}}$ near one.
Figures~\ref{fig:app_mixture_fit_base7} and~\ref{fig:app_mixture_fit_base10} display all 128 conditions with $n\in\{3,5,7,9\}$.
The weight $1-\pi_s$ of the slowly decaying component tracks $\hat E(32)$ in the reported comparisons with $\hat E(32)>0.1$, but can exceed it when few problems remain unsolved.
The fit summary also reports $\varepsilon_{\mathrm{good}}<0.1$ for R1-14B in conditions where R1-1.5B has $\varepsilon_{\mathrm{good}}>0.5$.
The weight $1-\pi_s$ describes the fitted component over the observed window; if both fitted failure rates are below one, the fitted curve tends to zero.

\begin{figure}[htbp]
    \centering
    \includegraphics[width=\linewidth]{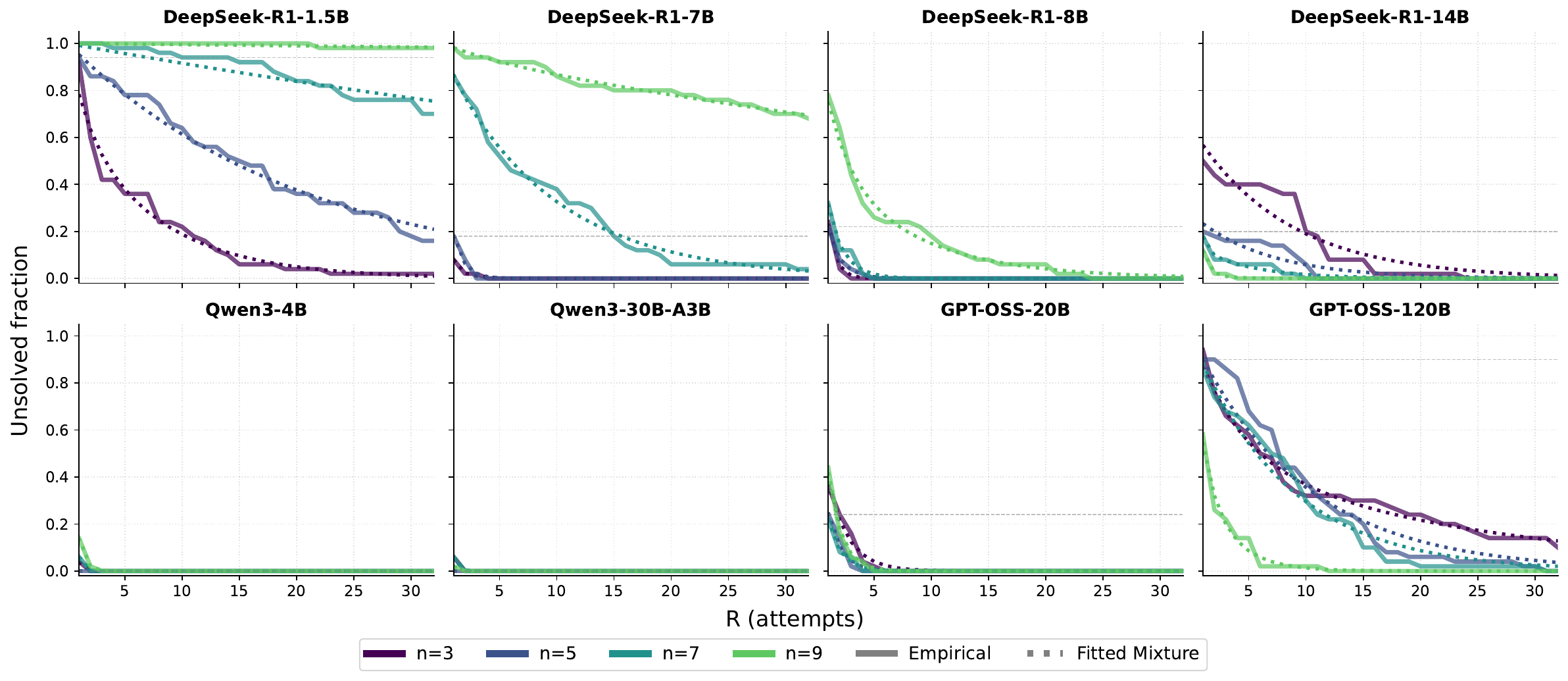}
    \includegraphics[width=\linewidth]{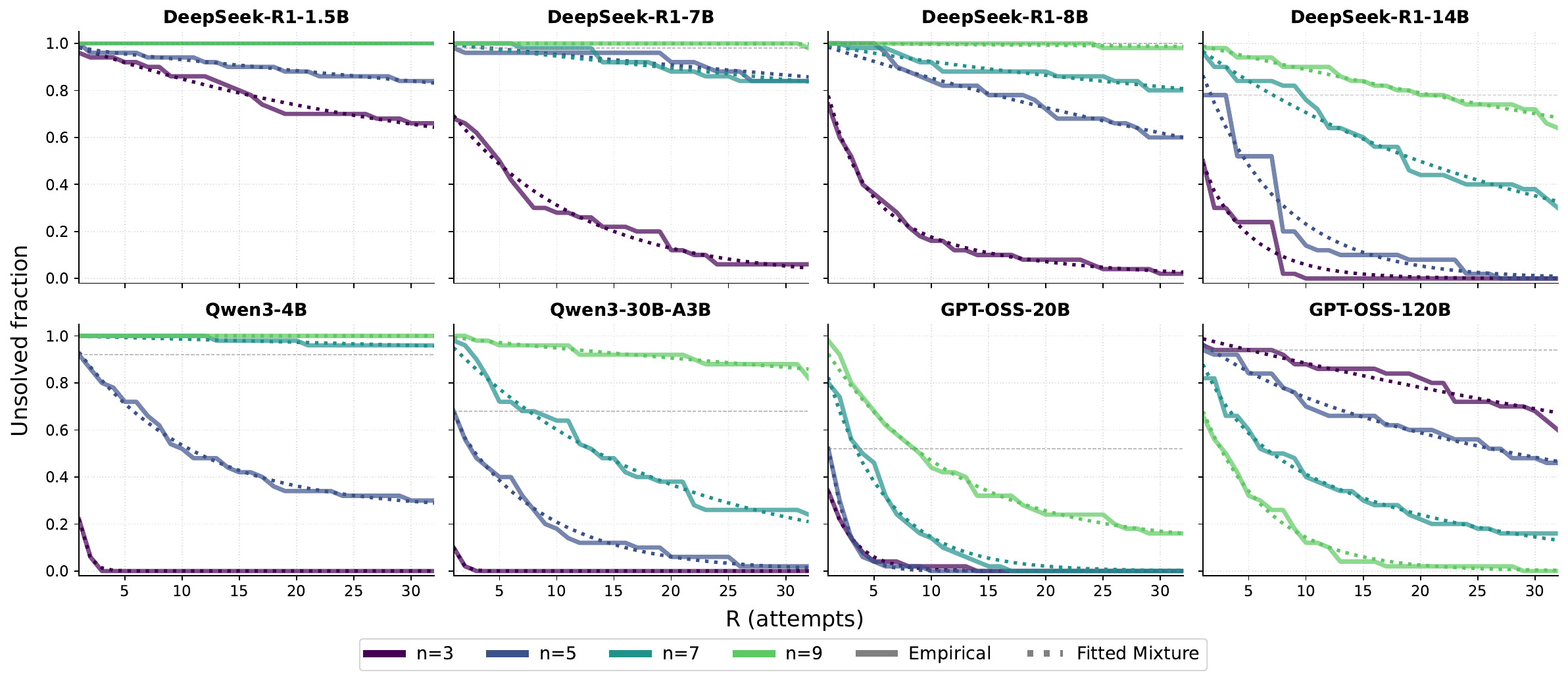}
    \caption{\textbf{Two-Component Fits for Base 7.} Solid curves show empirical $\hat E(R)$; dashed curves show nonlinear least-squares fits of $\pi_s\varepsilon_{\mathrm{good}}^R+(1-\pi_s)\varepsilon_{\mathrm{bad}}^R$ at $n\in\{3,5,7,9\}$. The slowly decaying component's weight $1-\pi_s$ is compared with the unsolved fraction in Appendix~\ref{app:mixture_fit}.}
    \label{fig:app_mixture_fit_base7}
\end{figure}

\begin{figure}[htbp]
    \centering
    \includegraphics[width=\linewidth]{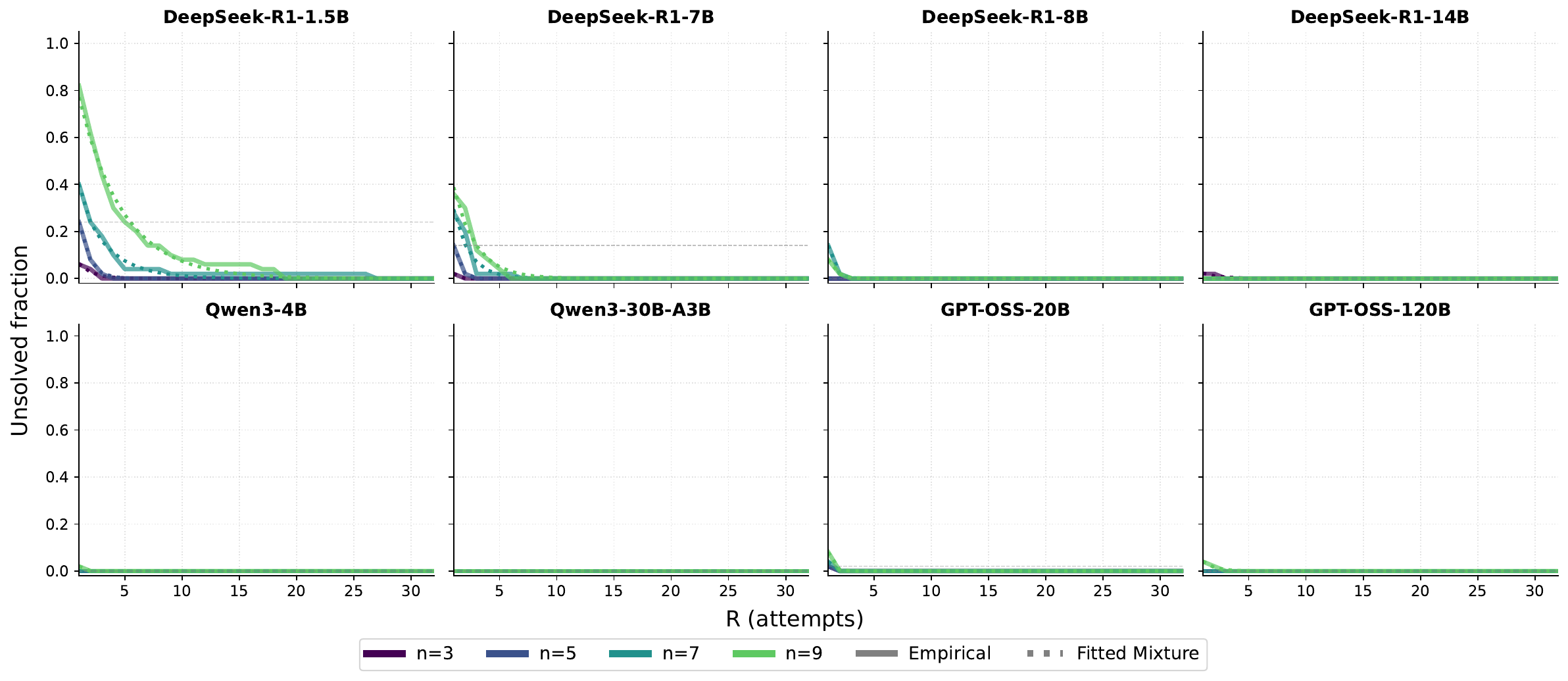}
    \includegraphics[width=\linewidth]{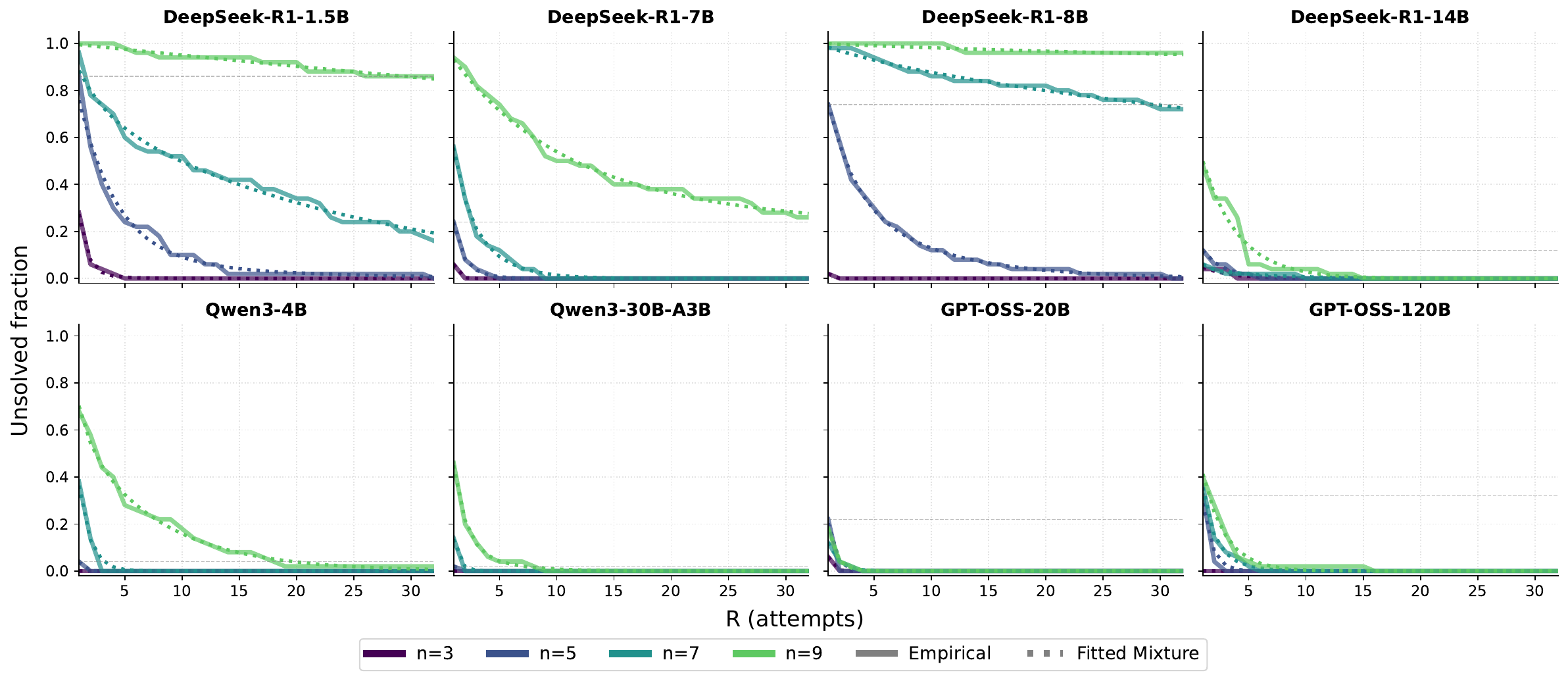}
    \caption{\textbf{Two-Component Fits for Base 10.} The layout, checkpoints, and curve styles follow Figure~\ref{fig:app_mixture_fit_base7}, with Base 10 addition and multiplication shown in the two panels.}
    \label{fig:app_mixture_fit_base10}
\end{figure}

\subsection{Gain from Additional Attempts by Digit Count}
\label{app:pass_gain}

Figure~\ref{fig:app_pass_gain} plots Pass@$R$ minus Pass@1 against digit count at $R\in\{4,8,16,32\}$, separately for addition and multiplication.
The gain is at most $1-\mathrm{Pass}@1$; its distance from this maximum equals $\hat E(R)$.
For R1-7B on Base 7 multiplication at $n=4$, Pass@1$=0$ and Pass@32$=0.26$, leaving $\hat E(32)=0.74$.
The vertical axis spans the full gain range from zero to one.

\begin{figure}[h]
    \centering
    \includegraphics[width=\linewidth]{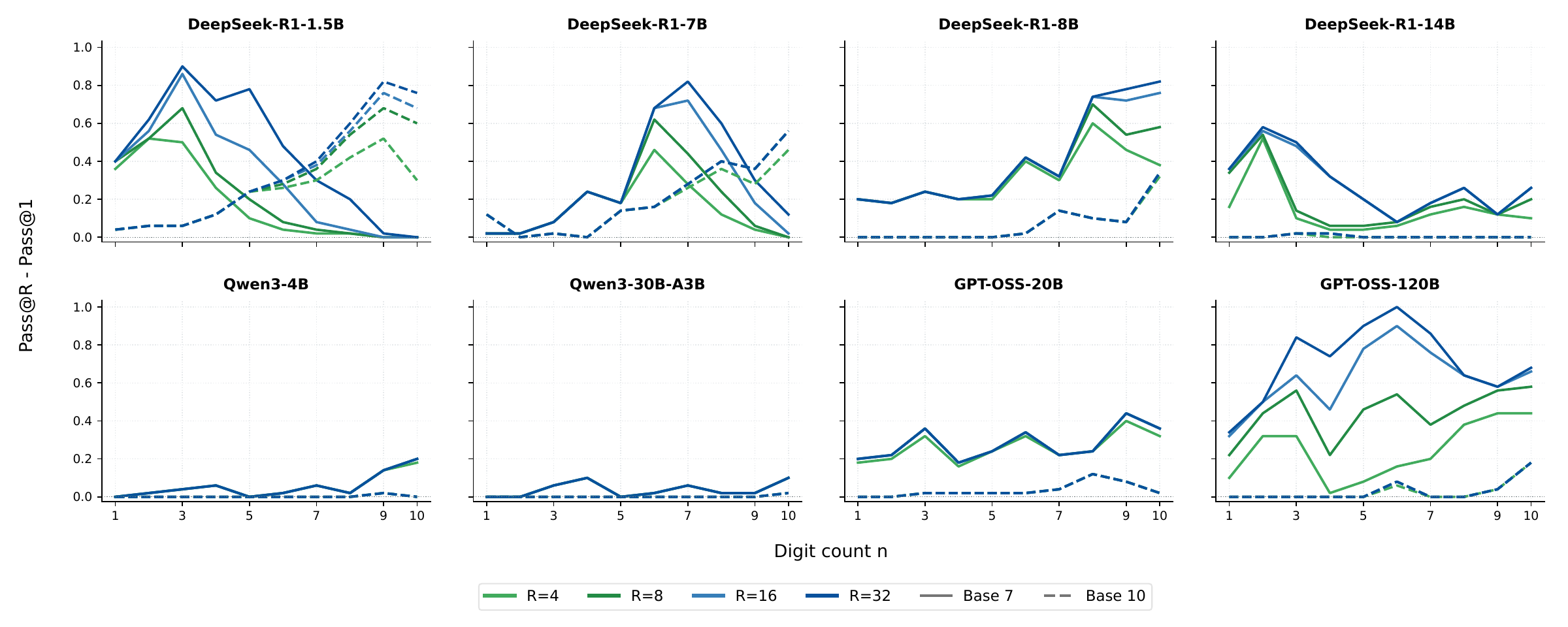}
    \vspace{0.5em}
    \includegraphics[width=\linewidth]{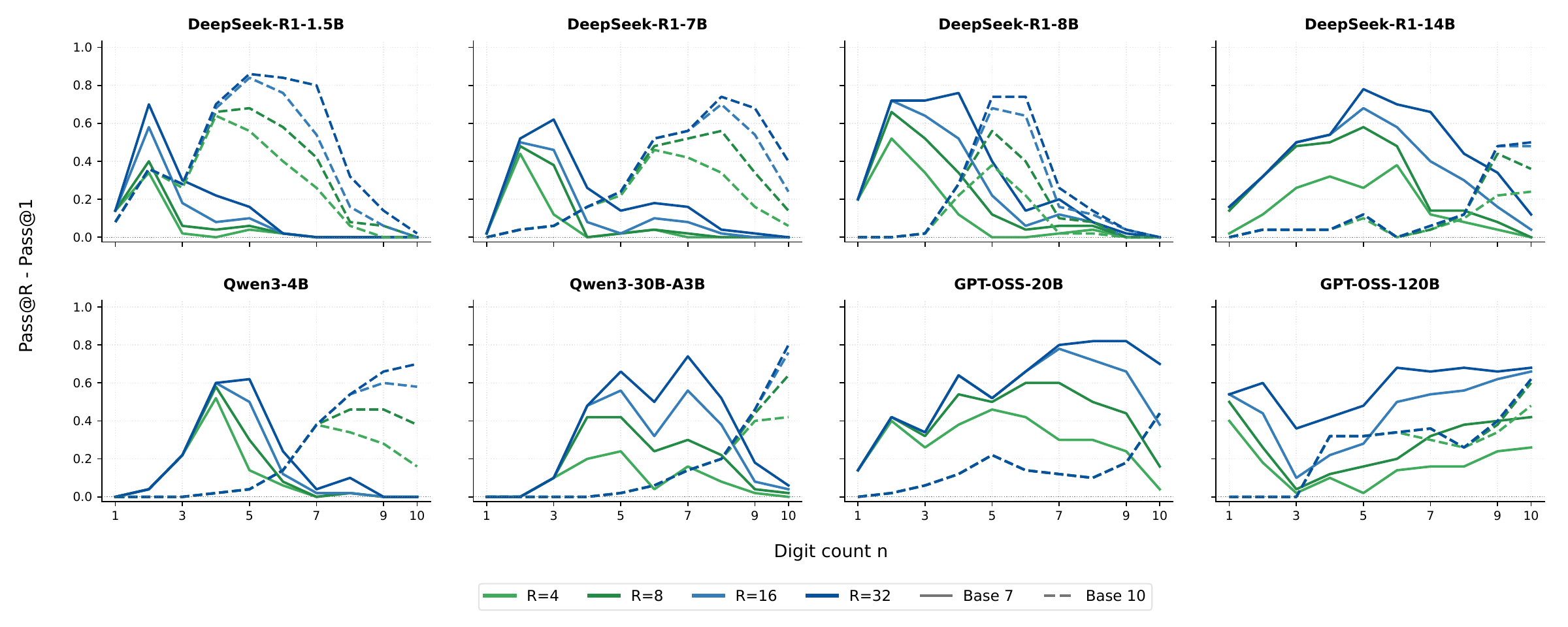}
    
    \caption{\textbf{Gain from Additional Attempts by Digit Count.} Top: addition; bottom: multiplication. Panels identify models, the horizontal axis is digit count, colors encode $R\in\{4,8,16,32\}$, and line styles distinguish bases. The plotted quantity is Pass@$R$ minus Pass@1; its gap to $1-\mathrm{Pass}@1$ equals $\hat E(R)$.}
    \label{fig:app_pass_gain}
\end{figure}

\begin{samepage}
\subsection{Eight-Digit Multiplication Trajectories}
\label{app:floor_convergence}

Figure~\ref{fig:app_floor_convergence} selects eight-digit multiplication in each base and plots the full $\hat E(R)$ trajectory for every model.
Several trajectories flatten around 8--16 attempts, while others continue to decrease through the observed window.
Figure~\ref{fig:app_checkpoint_convergence} reports the corresponding checkpoints across all conditions.

\end{samepage}
\begin{figure}[h]
    \centering
    \includegraphics[width=\linewidth]{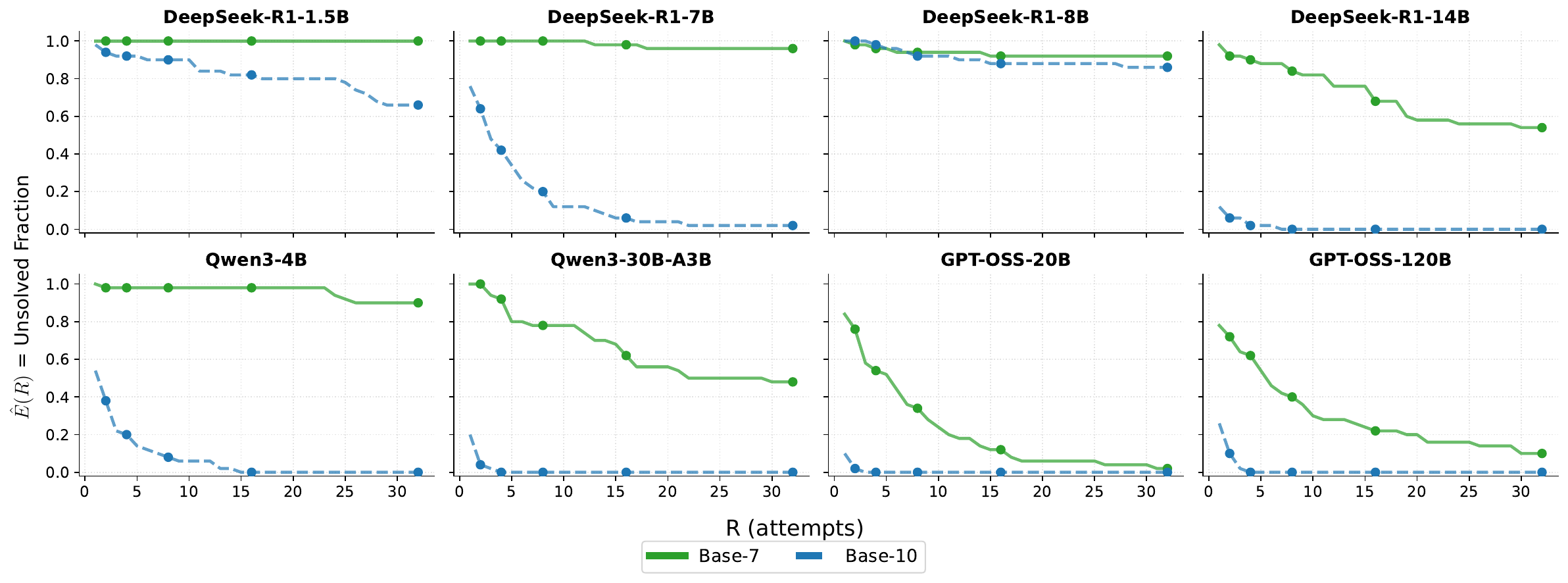}
    
    \caption{\textbf{Unsolved Fraction for Eight-Digit Multiplication.} Each model panel shows $\hat E(R)$ for Base 7 (solid green) and Base 10 (dashed blue) at $n=8$. Markers identify $R\in\{2,4,8,16,32\}$. The curves retain the full observed range from 1 to 32 attempts for both numeral bases.}
    \label{fig:app_floor_convergence}
\end{figure}

\subsection{Pass@1 and Pass@32 by Digit Count}
\label{app:pass1_vs_pass32}

Figure~\ref{fig:app_pass1_vs_pass32} pairs Pass@1 and Pass@32 at each digit count for all 320 conditions.
Panels identify models, colors distinguish bases, and the two marker types identify the attempt budgets.
Within-pair gaps measure retry gain; gaps above Pass@32 measure the unsolved fraction.

\begin{figure}[h]
    \centering
    \includegraphics[width=\linewidth]{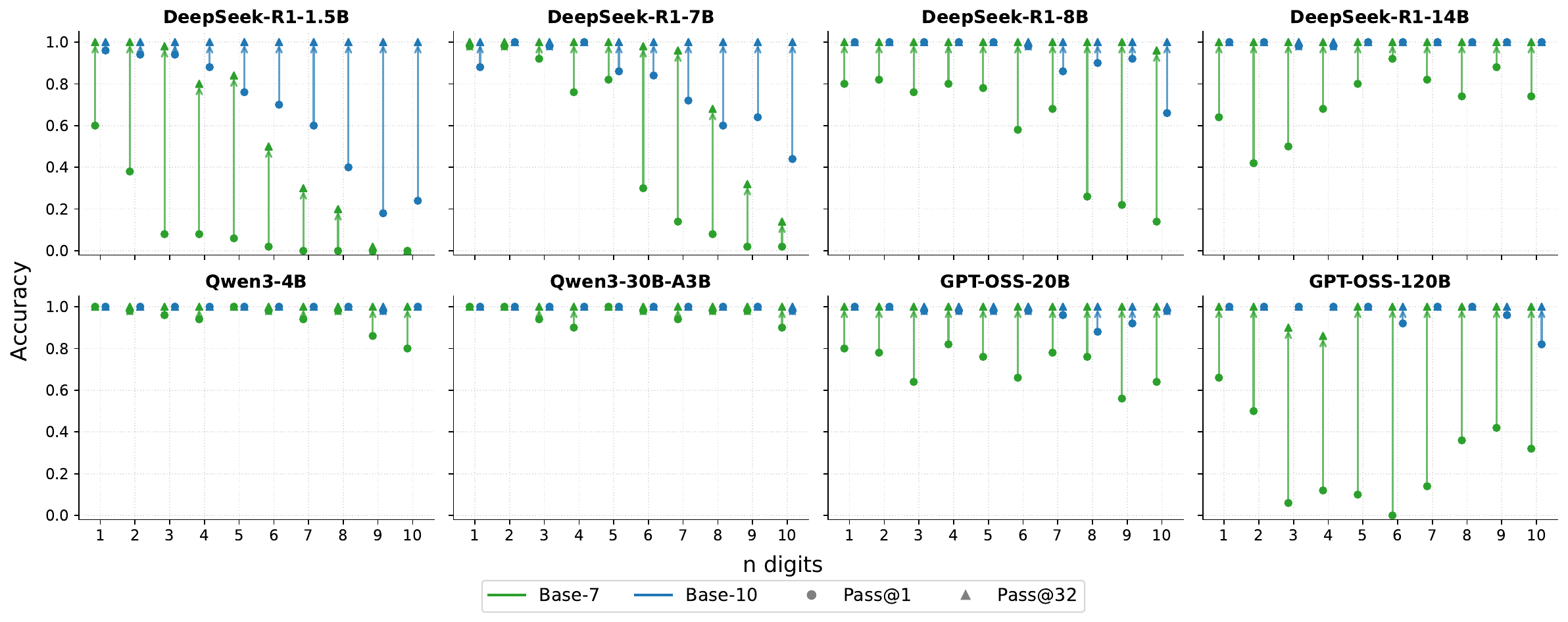}
    \vspace{0.5em}
    \includegraphics[width=\linewidth]{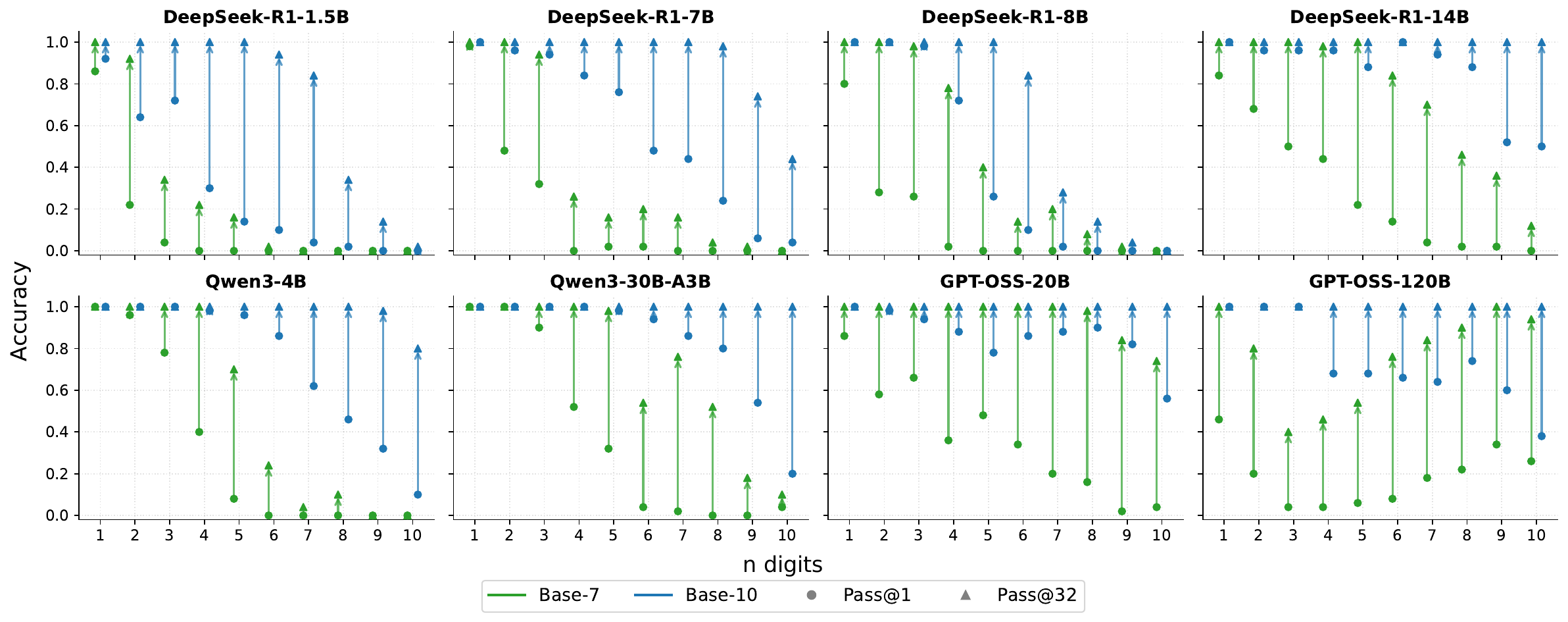}
    \caption{\textbf{Pass@1 and Pass@32 by Digit Count.} Top: addition; bottom: multiplication. Panels identify models, colors distinguish Base 7 and Base 10, and markers distinguish Pass@1 from Pass@32. The vertical separation at each digit count is the gain from additional attempts.}
    \label{fig:app_pass1_vs_pass32}
\end{figure}

\subsection{Generated Token Counts}
\label{app:token_heatmap}

Figure~\ref{fig:app_token_heatmap} reports average generated tokens per attempt.
The displayed range is 115--10,240 tokens, with the upper value set by the generation cap.
Base 7 multiplication averages 7,391 tokens, compared with 2,702 for Base 10 addition, a ratio of about $2.7\times$.
This contrast varies both the base and the operation.
GPT-OSS-120B uses 115 tokens for one-digit Base 10 addition and 116--153 for one- and two-digit Base 10 multiplication; these short decimal outputs do not explain the different coverage pattern observed for Base 7 arithmetic under the same attempt budget.\looseness=-1 

\begin{figure}[h]
    \centering
    \includegraphics[width=\linewidth]{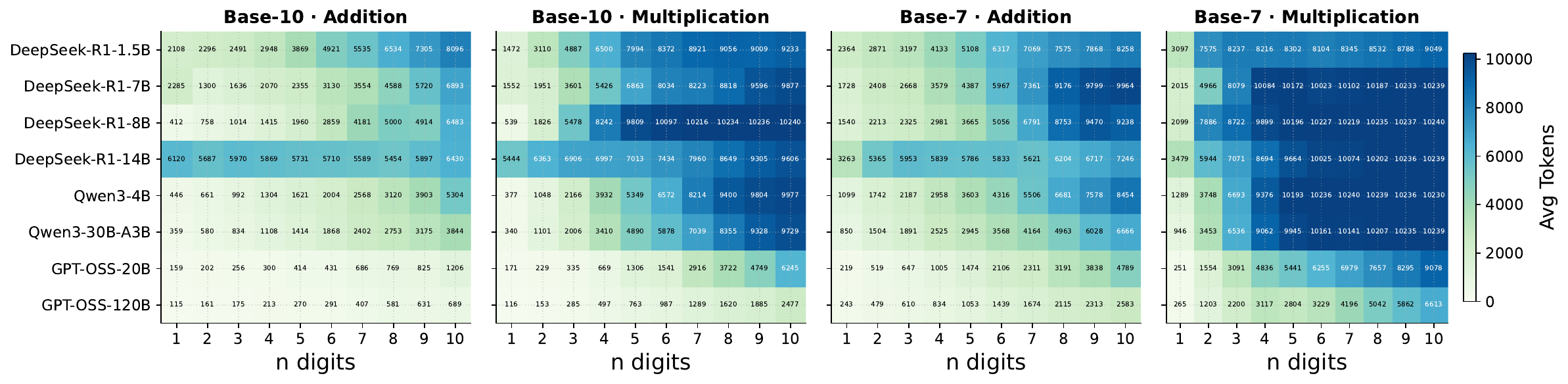}
    \caption{\textbf{Average Generated Tokens per Attempt.} Cells show mean output length for each combination of model, base, operation, and digit count under a 10,240-token generation cap. The displayed range is 115--10,240 tokens; averages are 7,391 for Base 7 multiplication and 2,702 for Base 10 addition.}
    \label{fig:app_token_heatmap}
\end{figure}

\FloatBarrier
\section{Agent Score Prediction: Derivation and Details}
\label{app:agent_prediction}

The forecasting comparison in \secref{sec:agent_prediction} uses a curve motivated by a special case of the repeated-attempt model: every instance is solvable and checks accept exactly the correct candidates.

We evaluate agent score prediction using each task--model pair's scores at 2 and 4 hours, together with priors fitted to other tasks in the same category, excluding the target task.
At 12 hours, the retry curve and the sigmoid curve in log time obtain median relative errors of 10.7\% and 10.9\%, respectively, over 248 pairs with positive observed scores.\looseness=-1 

\subsection{From Repeated Attempts to a Score Curve}
\label{app:agent_derivation}

We extend the model for generating correct candidates in~\appref{app:coverage_diagnostic} to continuous time.
Equation~\eqref{eq:effective_failure_rate} approximates the probability of no correct candidate after $R$ attempts:
\begin{equation*}
    q_t^{\mathrm{clean}}(R) \;\approx\; q_t^{\mathrm{floor}} + (1 - q_t^{\mathrm{floor}})\cdot\varepsilon_{\mathrm{loc},t}^{\,R}.
    \tag*{\ref*{eq:effective_failure_rate}, restated}
\end{equation*}

In the special case $q_t^{\mathrm{floor}}=0$ and $\varepsilon_{\mathrm{loc},t}=\varepsilon$, every instance has the same per-attempt failure probability. A perfect check gives $\alpha=0$ and $\rho=\varepsilon$, with success probability
\begin{equation}
    1 - q_t^{\mathrm{clean}}(R) \;=\; 1 - \varepsilon^R.
    \label{eq:app_step_success}
\end{equation}

The general loop also includes incorrect acceptance in its failure bound (Eq.~\eqref{eq:app_atomic_failure}).

\paragraph{Mapping to agent tasks.}

For the prediction model, $t$ is elapsed time in hours. We replace $R$ with $t/T$, where the positive time scale $T$ is fitted together with the remaining parameters.
The fitted score $S(t)\in[0,S_{\max}]$ is a continuous-time analogue of success probability, scaled by the task-specific ceiling $S_{\max}$ to represent cumulative task credit.\looseness=-1 

Scaling the resulting success curve by the task-specific ceiling $S_{\max}$ gives:
\begin{equation}
    S(t) = S_{\max} \cdot \bigl(1 - \varepsilon^{t/T}\bigr).
    \label{eq:app_agent_simplified}
\end{equation}
Equation~\eqref{eq:app_agent_simplified}, the \emph{retry curve}, substitutes $R=t/T$ and scales the result by the task ceiling.

\paragraph{Comparison curve.}
We also evaluate a sigmoid curve in log time:
\begin{equation}
    S(t) = \frac{S_{\max}}{1 + (t_{\mathrm{mid}} / t)^{\beta}},
    \label{eq:app_agent_logsig}
\end{equation}
Here $t_{\mathrm{mid}}>0$ is the time when $S=S_{\max}/2$, and $\beta>0$ controls curve steepness; this $\beta$ is unrelated to the local failure probability $\beta_t(R;s)$. The expression is a sigmoid function of $\log t$, labeled ``Sigmoid (log time)'' throughout the figures and tables for this forecasting comparison.

Both models are monotone and bounded and approach their fitted ceilings as $t$ increases.

\subsection{Experimental Setup}
\label{app:agent_setup}

\paragraph{Data.}
We use publicly reported scores from the EdgeBench long-horizon agent benchmark~\citep{edgebench2026}, which evaluates 5 frontier models across 51 tasks spanning 6 categories (Formal Mathematics, Games, Systems Engineering, Optimization, Scientific Computing, and Knowledge) at time budgets $t \in \{2, 4, 6, 8, 10, 12\}$ hours.
All six scores are available for 252 task--model pairs.

\paragraph{Prediction protocol.}

For each target (task, model) pair, we use $S(2\text{h})$ and $S(4\text{h})$ as its observed scores, fit model parameters using the priors below, and predict $S(t)$ for $t \in \{6, 8, 10, 12\}$h.
Both curves use category-level priors, meaning parameter estimates obtained from other tasks before fitting the target pair. The retry curve uses $(\varepsilon_{\mathrm{cat}},c_{\mathrm{cat}})$, the median fitted failure parameter and score-ceiling ratio; the sigmoid curve in log time uses $\beta_{\mathrm{cat}}$, the median slope.
For each target task, priors are medians of fits to all six observations from other tasks in the same category. Every model trajectory belonging to the target task is excluded when computing these priors.\looseness=-1 

\paragraph{Parameter estimation.}
For the retry curve (Eq.~\ref{eq:app_agent_simplified}), we jointly optimize $(\varepsilon,T,S_{\max})$ using maximum a posteriori (MAP) estimation: squared error on the two observed points plus a penalty for deviation from $(\varepsilon_{\mathrm{cat}},c_{\mathrm{cat}})$.
Because $\varepsilon^{t/T}$ depends on $\varepsilon$ and $T$ only through $-\log(\varepsilon)/T$, we do not interpret the two parameters separately.
For the sigmoid curve in log time (Eq.~\ref{eq:app_agent_logsig}), we optimize $(S_{\max}, t_{\mathrm{mid}}, \beta)$ with regularization toward $\beta_{\mathrm{cat}}$.
For both curves, fitting uses regularization coefficient $2.0$ and Nelder--Mead optimization with at most 2,000 iterations and $10^{-4}$ parameter and objective tolerances. The retry curve uses $0.01<\varepsilon<0.999$, $0.05<T\leq10$, and $0<S_{\max}\leq105$; the sigmoid curve in log time uses $0.01<t_{\mathrm{mid}}\leq500$, $0.05<\beta\leq6$, and $0<S_{\max}\leq105$.
MAP optimization produces numerical predictions for all 252 pairs, including the 53 cases where an analytic closed-form solution for the sigmoid curve in log time does not exist (i.e., when the observed growth ratio $S(4\text{h})/S(2\text{h})$ is inconsistent with the prior slope $\beta_{\mathrm{cat}}$).\looseness=-1

\paragraph{Fitting inputs and regularization.}
For fitting only, let $u=\max\{S(2),0\}$ and let $v=S(4)$ when $S(4)>u$, or $v=1.001u+0.01$ otherwise. Both procedures use these same adjusted early observations; evaluation uses the original later scores.
The ceiling ratio for each source-task fit is $S_{\max}/\max\{S(4),0.1\}$, using that source task's observed four-hour score. The prior $c_{\mathrm{cat}}$ is the median of these ratios after excluding the target task.
Let $\sigma_\varepsilon,\sigma_c,\sigma_\beta$ be the population standard deviations of the corresponding source-fit parameters, with
\[
s_\varepsilon=\max\{\sigma_\varepsilon,0.05\},\qquad
s_c=\max\{\sigma_c,0.1\},\qquad
s_\beta=\max\{\sigma_\beta,0.15\}.
\]
Writing $f_{\mathrm{retry}}$ and $f_{\mathrm{sig}}$ for the two score curves, the implemented objectives are
\begin{align*}
\mathcal L_{\mathrm{retry}}
&=[f_{\mathrm{retry}}(2)-u]^2+[f_{\mathrm{retry}}(4)-v]^2\\
&\quad+2\left[\left(\frac{\varepsilon-\varepsilon_{\mathrm{cat}}}{s_\varepsilon}\right)^2
+\left(\frac{S_{\max}/\max\{v,0.1\}-c_{\mathrm{cat}}}{s_c}\right)^2\right],\\
\mathcal L_{\mathrm{sig}}
&=[f_{\mathrm{sig}}(2)-u]^2+[f_{\mathrm{sig}}(4)-v]^2
+2\left(\frac{\beta-\beta_{\mathrm{cat}}}{s_\beta}\right)^2.
\end{align*}
The sigmoid penalizes slope deviations. Both optimizers start with ceiling $\min\{v c_{\mathrm{cat}},100\}$. A coarse parameter grid supplies the starting point; the optimizer's output replaces it only when the objective improves.
Prior construction uses unregularized six-observation fits with up to 3,000 iterations, $S_{\max}\leq110$, and, for the retry curve, $T\leq8$. The bounds above apply to the target-pair prediction fits.

\subsection{Results}
\label{app:agent_results}

Table~\ref{tab:agent_prediction} reports median relative errors and the fractions of predictions below the 10\% and 20\% error thresholds at all four prediction times.
Median errors are similar, with neither curve having a smaller error at every prediction time.
The retry curve has at least as many predictions below 10\% relative error, and more below 20\%, at every prediction time.\looseness=-1 

\begin{table}[h]
    
    \caption{\textbf{Agent Score Prediction from 2- and 4-Hour Observations.} Both curves use priors from other tasks as specified in Appendix~\ref{app:agent_setup}. Median relative errors and fractions below 10\% and 20\% use 248 pairs with positive observed scores at each prediction time. All 252 pairs receive predictions; four pairs with zero observed scores are excluded from relative errors at each prediction horizon.}
    \label{tab:agent_prediction}
    \centering
    \footnotesize
    \renewcommand{\arraystretch}{1.15}
    \setlength{\tabcolsep}{2.0mm}
    \begin{tabular}{lcccccc}
    \toprule[0.95pt]
    \multirow{2}{*}[-0.5ex]{\textbf{Time}} & \multicolumn{3}{c}{\textbf{Retry Curve} (Eq.~\ref{eq:app_agent_simplified})} & \multicolumn{3}{c}{\textbf{Sigmoid (log time)} (Eq.~\ref{eq:app_agent_logsig})} \\
    \cmidrule(lr){2-4}\cmidrule(lr){5-7}
    & Median & ${<}10\%$ & ${<}20\%$ & Median & ${<}10\%$ & ${<}20\%$ \\
    \midrule[0.6pt]

    6 hours  & 3.8\% & 71\% & 87\% & 3.9\% & 67\% & 81\% \\
    8 hours  & 6.5\% & 61\% & 78\% & 5.8\% & 58\% & 73\% \\
    10 hours & 8.9\% & 53\% & 69\% & 8.8\% & 52\% & 66\% \\
    12 hours & 10.7\% & 48\% & 65\% & 10.9\% & 48\% & 62\% \\
    \bottomrule[0.95pt]
    \end{tabular}
\end{table}

Figures~\ref{fig:agent_scatter} and~\ref{fig:agent_horizon} show the 12-hour predictions and increasing median error at later times. Both curves have median error below 11\% at all four prediction times.

\begin{figure}[h]
    \centering
    \includegraphics[width=0.65\linewidth]{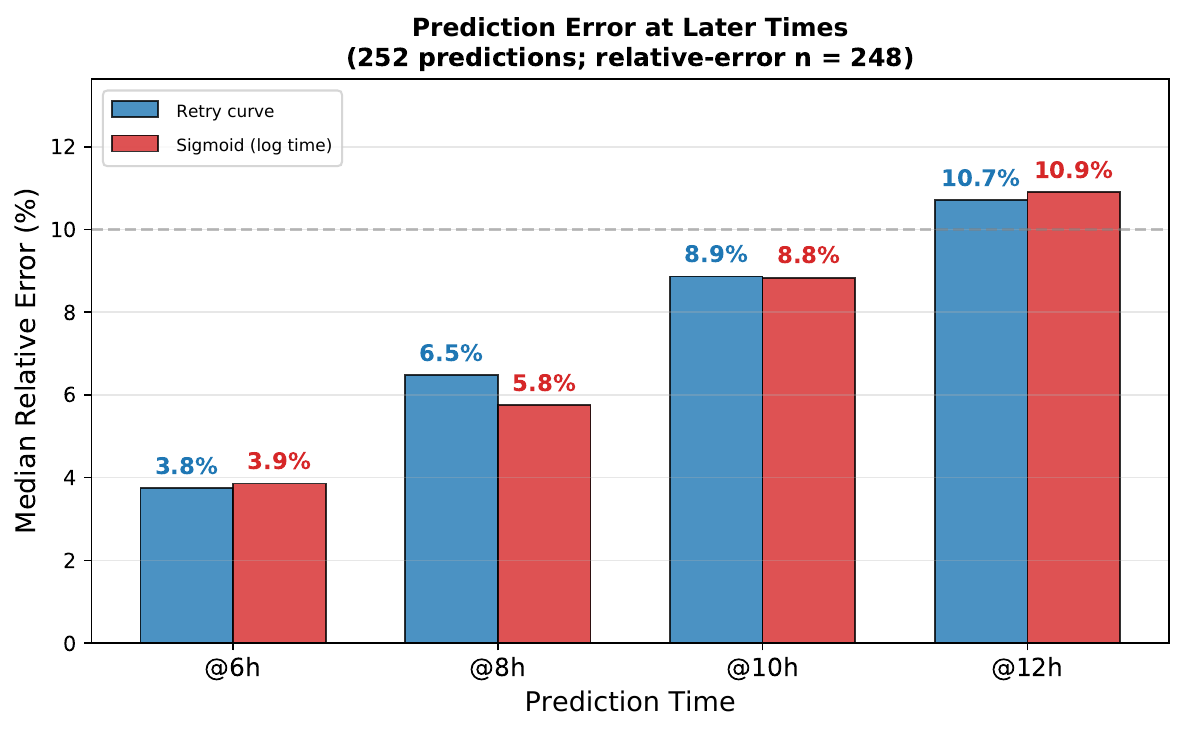}
    \caption{\textbf{Median Relative Error by Prediction Time.} Both curves use the protocol in Appendix~\ref{app:agent_setup}. Relative errors use 248 pairs with positive observed scores at each prediction time; the dashed line marks 10\%.}
    \label{fig:agent_horizon}
\end{figure}

\subsection{Interpretation}
\label{app:agent_interpretation}

The two curves have similar median errors when fitted to early observations with priors from other tasks in the same category.
These predictions use early scores and information from other tasks; they do not test the checks assumed by the reuse theorem.
Although the source tasks are shared, the penalties constrain different parameters: failure and ceiling ratio for the retry curve, and slope for the sigmoid. Both procedures also use the ceiling-ratio prior for initialization.
The differences compare the fitted procedures; they do not isolate curve shape under identical parameter information.
Predictive fit does not identify agents' internal checking or state-reuse behavior.
The experiment predicts later scores; it does not evaluate early stopping or compute allocation.\looseness=-1